\documentclass{article}

\usepackage{microtype}
\usepackage{graphicx}
\usepackage{booktabs} 
\usepackage{amsmath}
\usepackage{amssymb}
\usepackage{enumerate}
\usepackage{amsmath}
\usepackage{amsfonts}
\usepackage{amsthm}
\usepackage{dsfont}
\usepackage{subcaption}
\usepackage{hyperref}

\usepackage{xcolor}

\usepackage[accepted]{icml2020}

\def\bfa{{\boldsymbol a}}

\def\bfe{{\boldsymbol e}}

\def\bfu{{\boldsymbol u}}
\def\bfv{{\boldsymbol v}}
\def\bfw{{\boldsymbol w}}
\def\bfx{{\boldsymbol x}}
\def\bfy{{\boldsymbol y}}
\def\bfz{{\boldsymbol z}}

\def\bfA{{\boldsymbol A}}

\def\bfD{{\boldsymbol D}}

\def\bfI{{\boldsymbol I}}

\def\bfL{{\boldsymbol L}}
\def\bfM{{\boldsymbol M}}

\def\bfP{{\boldsymbol P}}

\def\bfS{{\boldsymbol S}}
\def\bfT{{\boldsymbol T}}

\def\bfV{{\boldsymbol V}}
\def\bfW{{\boldsymbol W}}
\def\bfX{{\boldsymbol X}}

\def\bfZ{{\boldsymbol Z}}

\newcommand{\W[1]}{{\bfW}^{(#1)}}

\newtheorem{theorem}{Theorem}
\newtheorem{lemma}{Lemma}
\newtheorem{definition}{Definition}

\begin{document}
\twocolumn[
\icmltitle{Fast Learning of Graph Neural Networks with Guaranteed Generalizability: One-hidden-layer Case}



\icmlsetsymbol{equal}{*}

\begin{icmlauthorlist}
	\icmlauthor{Shuai Zhang}{to}
	\icmlauthor{Meng Wang}{to}
	\icmlauthor{Sijia Liu}{goo}
	\icmlauthor{Pin-Yu Chen}{ed}
	\icmlauthor{Jinjun Xiong}{ed}
\end{icmlauthorlist}

\icmlaffiliation{to}{Dept. of Electrical, Computer, and Systems Engineering,   Rensselaer Polytechnic Institute,  NY, USA}
\icmlaffiliation{goo}{MIT-IBM Waston AI Lab, Cambridge, MA, USA}
\icmlaffiliation{ed}{IBM Thomas J. Watson Research Center, Yorktown Heights, NY, USA}

\icmlcorrespondingauthor{Shuai Zhang}{zhangs21@rpi.edu}
\icmlcorrespondingauthor{Meng Wang}{wangm7@rpi.edu}
\icmlcorrespondingauthor{Sijia Liu}{Sijia.Liu@ibm.com}
\icmlcorrespondingauthor{Pin-Yu Chen}{Pin-Yu.Chen@ibm.com}
\icmlcorrespondingauthor{Jinjun Xiong}{jinjun@us.ibm.com}

\icmlkeywords{Machine Learning, ICML}

\vskip 0.3in
]
\printAffiliationsAndNotice{}
\begin{abstract}
Although graph neural networks (GNNs) have made great progress recently on learning from graph-structured data in practice, their theoretical guarantee on generalizability  remains elusive in the literature. In this paper, we provide a theoretically-grounded generalizability analysis of GNNs with one hidden layer for both regression and binary classification problems. Under the assumption that there exists a ground-truth GNN model (with zero generalization error), the objective of GNN learning is to estimate the ground-truth GNN parameters from the training data. 
To achieve this objective, we propose a learning algorithm that is built on tensor initialization and accelerated gradient descent.
We then show that the proposed 
learning algorithm converges to the ground-truth GNN model for the regression problem, and to a model sufficiently close to the ground-truth for the binary classification problem. Moreover, for both cases, the convergence rate of the proposed learning algorithm is proven to be linear and faster than the vanilla gradient descent algorithm. We further explore the relationship between the sample complexity of GNNs and their underlying graph properties. Lastly, we provide numerical experiments to demonstrate the validity of our analysis and the effectiveness of the proposed learning algorithm for GNNs.

\end{abstract}	

\section{Introduction}

 Graph neural networks (GNNs)   \cite{GMS05,SGAHM08} have demonstrated great practical performance in learning with graph-structured data. 
Compared with traditional (feed-forward) neural networks, GNNs introduce an additional neighborhood aggregation layer, where the features of each node are aggregated with the features of  the neighboring nodes \cite{GSRVD17,XLTSTKKJ18}. GNNs have a better learning  performance in
  applications including physical reasoning \cite{BPLR16}, recommendation systems \cite{YHCEHL18},  biological analysis \cite{DMIBH15}, and compute vision \cite{MDSG06}. 
Many variations of GNNs, such as Gated Graph Neural Networks (GG-NNs) \cite{LTBZ16}, Graph Convolutional Networks (GCNs) \cite{KW17} and   others  \cite{HYL17, VCCRLB18}     have recently  been  developed to enhance the learning performance on graph-structured data.



Despite the numerical success, the theoretical understanding of the generalizability of the learned GNN models to the testing data is very limited. Some works \cite{XLTSTKKJ18,XHLJ19,WSZFYW19,MRFMW19} analyze the expressive power of GNNs but do not provide learning algorithms that are guaranteed to return the desired GNN model with proper parameters.  Only    few works  \cite{DHPSWX19,VZ19} explore the generalizabilty of GNNs, under the one-hidden-layer setting, as  even with one hidden layer  the models are  already    complex to analyze,  not to mention the multi-layer setting. Both works show that for regression problems, the generalization gap of the training error and the testing error decays with respect to the number of training samples at a sub-linear rate. The  analysis in   Ref.~\cite{DHPSWX19} analyzes GNNs through Graph Neural Tangent Kernels (GNTK)  
{which is an extension of Neural Tangent kernel (NTK) model \cite{JGH18,CB18, NS19,CG20}.   When over-parameterized, this line of works   shows sub-linear convergence to the global optima of the learning problem with assuming enough filters in the hidden layer \cite{JGH18,CB18}. }  
Ref.~\cite{VZ19} only applies to the case of one single filter in the hidden layer, and the activation function  needs to be smooth, excluding the popular ReLU activation function. 
Moreover, refs.~\cite{DHPSWX19,VZ19} do not consider classification and do not discuss {if a small training error and a small generalization error can be achieved simultaneously.}


 One recent line of research analyzes the generalizability of neural networks (NNs) from the perspective of model estimation \cite{BG17,DLTPS17,DLT17,FCL19,GLM17,SS17,ZSJB17,ZSD17}. These works assume   
the existence of a ground-truth NN model with some unknown parameters that maps the input features to the output labels for both training and testing samples. Then the learning objective is to estimate  the ground-truth model parameters from the training data, and this ground-truth model is guaranteed to have  a zero generalization                                             error on the testing data. The analyses are  focused on one-hidden-layer NNs, assuming the input features following the Gaussian   distribution  \cite{Sh18}.  
 If one-hidden-layer NNs only have one filter in the hidden layer, gradient descent (GD) methods can learn the ground-truth parameters with a high probability \cite{DLTPS17,DLT17,BG17}. When there are multiple filters in the hidden layer, the learning problem is much more challenging to solve because of the common spurious local minima \cite{SS17}. \cite{GLM17} revises the learning objective and shows the global convergence of GD   to the global optimum of the new learning problem. 
 The required number for training samples, referred to as the sample complexity in this paper, is a high-order polynomial function of the   model size. A few works  \cite{ZSJB17,ZSD17,FCL19} study a learning algorithm that initializes using the  tensor initialization method \cite{ZSJB17} and iterates using GD. This algorithm is proved to  converge  to the ground-truth model parameters with a zero generalization error for the one-hidden-layer NNs with multiple filters, and the sample complexity is shown to be linear in the model size. All these works only consider NNs rather than GNNs.

 \textbf{Contributions.} 
 This paper provides the first algorithmic design and theoretical analysis to learn a GNN model with a zero generalization error, assuming the existence of such a ground-truth model. 
  We study GNNs  in semi-supervised learning,  
 and the results apply to  both regression and binary classification problems. 
Different from NNs, each output label on the graph depends on multiple neighboring features in GNNs, and such dependence significantly complicates the analysis of the learning problem.  Our proposed algorithm uses the tensor initialization \cite{ZSJB17} and updates by accelerated gradient descent (AGD). We prove that with a sufficient number of training samples, our algorithm returns the ground-truth model with the zero generalization error for regression problems. For binary classification problems, our algorithm returns a model sufficiently close to the ground-truth model, and its distance to the ground-truth model decays to zero as the number of samples increases. Our algorithm converges linearly, with a rate that is proved to be faster than that of vanilla GD. We quantifies the dependence of the sample complexity on the model size and the underlying graph structural properties. The required number of samples is linear in the model size. It is also a polynomial function of the graph degree and the largest singular value of the normalized adjacency matrix. Such dependence of the sample complexity on graph parameters is exclusive to GNNs and does not exist in NNs. 

The rest of the paper is organized as follows.  Section 2 introduces the problem formulation. 
The algorithm is presented in Section 3,  and Section 4 summarizes the major theoretical results. Section 5 shows the numerical results, and Section 6 concludes the paper. All the proofs are   in the supplementary materials.

{\it Notation:} Vectors are bold lowercase, matrices and tensors are bold uppercase. Also, scalars are in normal font, and sets are in calligraphy and blackboard bold font. For instance,  $\bfZ$ is a matrix, and $\bfz$ is a vector. $z_i$ denotes the $i$-th entry of $\bfz$, and $Z_{ij}$ denotes the $(i,j)$-th entry of $\bfZ$. $\mathcal{Z}$ stands for a regular set. Special sets $\mathbb{N}$ (or $\mathbb{N}^+$), $\mathbb{Z}$ and $\mathbb{R}$ denote the sets of all natural numbers (or positive natural numbers), all integers and all real numbers, respectively.
Typically, [$Z$] stands for the set of $\{ 1, 2, \cdots, Z  \}$ for any number $\mathbb{N}^+$.
$\bfI$ and $\bfe_i$ denote the identity matrix and the $i$-th standard basis vector.
$\bfZ^T$ denotes the transpose of $\bfZ$, similarly for $\bfz^T$. 
$\|\bfz\|$ denotes the $\ell_2$-norm of a vector $\bfz$, and $\|\bfZ\|_2$ and $\|\bfZ\|_F$ denote the spectral norm and Frobenius norm of  matrix $\bfZ$, respectively.
We use $\sigma_{i}(\bfZ)$ to denote the $i$-th largest singular value of $\bfZ$.
Moreover, the outer product of a group of vectors $\bfz_i\in \mathbb{R}^{n_i}, i \in [l]$, is defined as $\bfT=\bfz_1\otimes \cdots\otimes \bfz_l \in \mathbb{R}^{n_1  \times \cdots\times n_l}$ with $T_{j_1,\cdots, j_l}= (\bfz_{1})_{j_1}\cdots(\bfz_{l})_{j_l}$. 
\section{Problem Formulation}\label{Section: pf_1}
Let $\mathcal{G}=\{ \mathcal{V}, \mathcal{E} \}$ denote an un-directed graph, where $\mathcal{V}$ is the set of nodes with size $|\mathcal{V}|=N$ and $\mathcal{E}$  is the set of edges. Let $\delta$ and $\delta_{\textrm{ave}}$ denote the maximum and average node degree  of $\mathcal{G}$, respectively.  Let $\boldsymbol{\tilde{A}}\in \{0,1\}^{N\times N}$ be the adjacency matrix of  $\mathcal{G}$ with added self-connections. Then,   $\tilde{A}_{i,j}=1$ if and only if  there exists an edge between node $v_i$ and node $v_j$, $i,j \in [N]$, and $\tilde{A}_{i,i}=1$ for all $i\in [N]$.  Let  $\bfD$ be the degree matrix with diagonal elements $D_{i,i} = \sum_{j}\tilde{A}_{i,j}$ and zero entries otherwise. $\bfA$     denotes  the normalized adjacency matrix with $\bfA= \bfD^{-1/2}\boldsymbol{\tilde{A}}\bfD^{-1/2}$, and $\sigma_1(\bfA)$ is the largest singular value of   $\bfA$.  

 Each node $v_n$ in $\mathcal{V}$  $(n=1,2,\cdots, N)$ corresponds to an input feature vector, denoted by $\bfx_n\in \mathbb{R}^{d}$, and a label $y_n\in\mathbb{R}$. $y_n$ depends on not only $\bfx_n$ but also all $\bfx_j$ where $v_j$ is a neighbor of $v_n$. Let $\bfX = \begin{bmatrix}
\bfx_1, \bfx_2, \cdots, \bfx_N
\end{bmatrix}^T\in \mathbb{R}^{N\times d}$ denote the feature matrix.  Following the analyses of NNs \cite{Sh18}, we assume $\bfx_n$'s are i.i.d. samples from the standard Gaussian distribution $\mathcal{N}(\boldsymbol{0}, \bfI_{d})$.    
 For GNNs,  we consider the \ typical  semi-supervised learning  problem setup.  Let $\Omega\subset [N]$ denote the set of node indices with known labels, and let  $\Omega^c$ be its complementary set.  
The objective of the GNN is to predict $y_i$ for every $i$ in $\Omega^c$. 

Suppose there exists a one-hidden-layer GNN that maps  node  features to labels, as shown in Figure \ref{Figure: GNN}. There are $K$ filters\footnote{We assume $K\leq d$   to simplify the representation of the analysis, while the result still holds for $K>d$ with   minor changes. }
 in    the hidden layer, and the weight matrix is denoted by 
	 $\bfW^* = \begin{bmatrix} 
	\bfw_1^*&\bfw_2^*&\cdots&\bfw_K^*
	\end{bmatrix}\in \mathbb{R}^{d\times K}$ . The hidden layer is followed by a pooling layer. Different from NNs, GNNs have an additional aggregation layer with $\bfA$ as the aggregation factor matrix \cite{KW17}.   For every node $v_n\in\mathcal{V}$, the input to the hidden layer is $\bfa_n^T\bfX$, where $\bfa_n^T$ denotes the $n$-th row of $\bfA$. When there is no edge in $\mathcal{V}$, $\bfA$  is reduced to the identify matrix, and a GNN model is reduced to an NN model.

	\begin{figure}[h]
		\centering
		\includegraphics[width=0.44\textwidth]{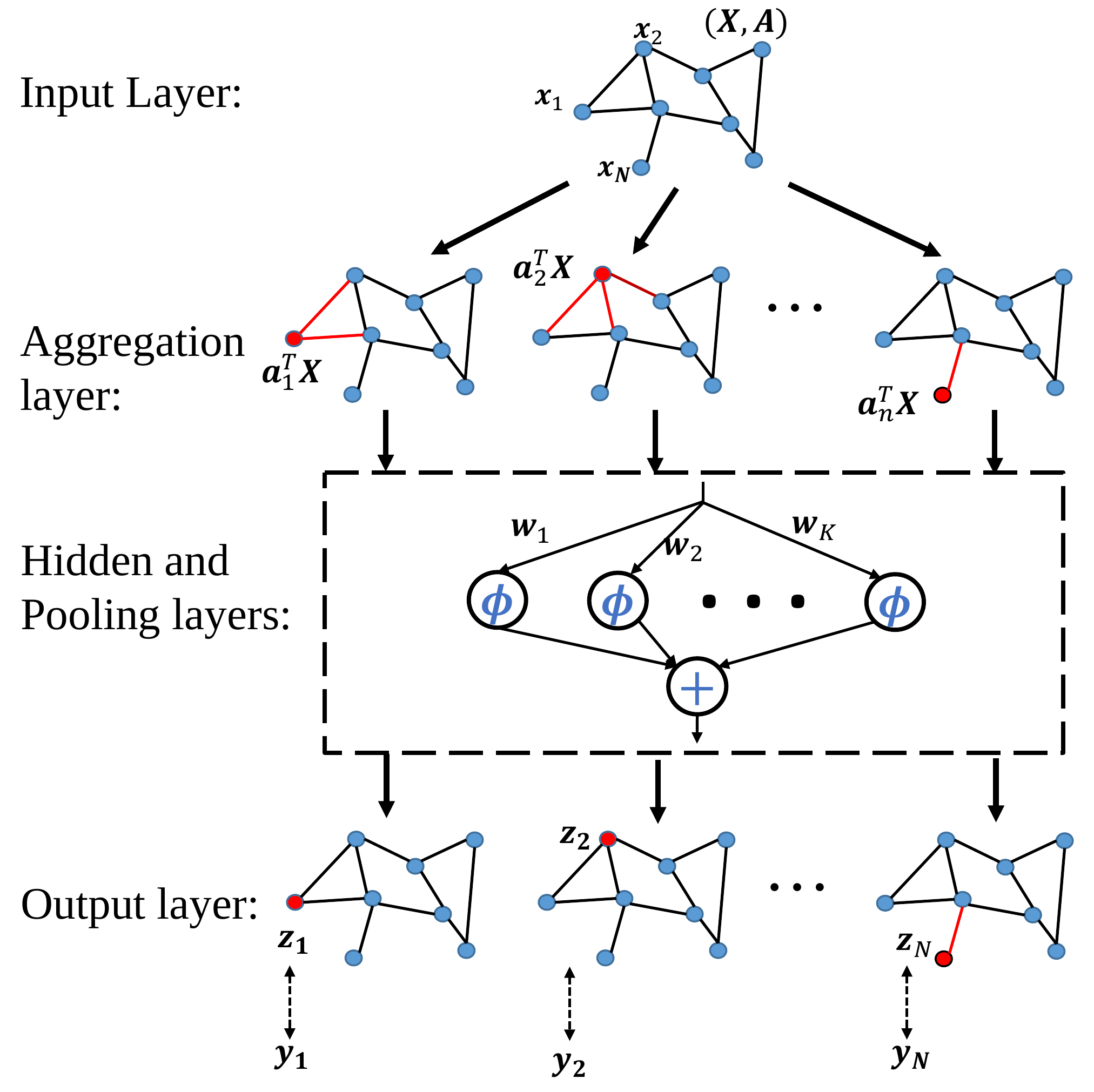}
		\caption{Structure of the graph neural network}
		\label{Figure: GNN}
	\end{figure}	
 
The output $z_n$ of the node $v_n$ of the GNN is 
	 	\begin{equation}\label{eqn: y_n}
	 \begin{gathered}
	z_n= g(\bfW^*;\bfa_n^T\bfX) = \frac{1}{K}\sum_{j=1}^{K}\phi(\bfa_n^T\bfX{\bfw_j^*}), \forall n \in[N],
	 \end{gathered}
	 \end{equation}
	 where $\phi(\cdot)$ is  the activation function. We consider both regression and binary classification in this paper. For regression,  $\phi(\cdot)$ is the ReLU function\footnote{Our result can be extended to the sigmoid activation function with minor changes.} 
	 $\phi(x)=\max\{x , 0  \}$, 
	 and  
	 $y_n = z_n$. 
	 For binary classification, we consider the sigmoid activation function where $\phi(x)=1/(1+e^{-x})$. Then $y_n$ is  a binary variable generated from $z_n$ by  $\text{Prob}\{y_n=1\}=z_n$, and $\text{Prob}\{y_n=0\}=1-z_n$.

Given $\bfX$, $\bfA$, and $y_i$ for all $i\in \Omega$, the learning objective is to  estimate $\bfW^*$, which  is assumed to have a zero generalization error. The  {training objective} is to minimize the empirical risk function,
	\begin{equation}\label{eqn: optimization}
	\min_{\bfW\in\mathbb{R}^{d\times K}}  \hat{f}_{\Omega}(\bfW):=\frac{1}{|\Omega|}\sum_{n\in \Omega} \ell(\bfW;\bfa_n^T\bfX),
	\end{equation}
	where $\ell$  is the loss function. 
For regression,	 we use the squared loss function
, and \eqref{eqn: optimization} is written as
	\begin{equation}\label{eqn:linear_regression}
	\min_{\bfW}: \quad \hat{f}_{\Omega}(\bfW)=\frac{1}{2|\Omega|}\sum_{n\in \Omega}\Big| y_n - g(\bfW;\bfa_n^T\bfX) \Big|^2.
	\end{equation}
For classification, we use the cross entropy loss function, and \eqref{eqn: optimization} is written as
	\begin{equation}\label{eqn: classification}
		\begin{split}
		\min_{\bfW}: \quad \hat{f}_{\Omega}(\bfW)
		=&\frac{1}{|\Omega|}\sum_{n\in \Omega} -y_n \log \big(g(\bfW;\bfa_n^T\bfX)\big)\\
		-(1&-y_n) \log \big(1-g(\bfW;\bfa_n^T\bfX)\big).
		\end{split}
	\end{equation}

	Both \eqref{eqn:linear_regression} and \eqref{eqn: classification} are nonconvex due to the nonlinear function $\phi$. 
 Moreover,   while 	 $\bfW^*$ is a global minimum of \eqref{eqn:linear_regression},   $\bfW^*$ is	 not necessarily  a global minimum of \eqref{eqn: classification}\footnote{$\bfW^*$ is a global minimum if replacing all $y_n$ with $z_n$ in \eqref{eqn: classification}, but $z_n$'s are unknown in practice. }. Furthermore, compared with NNs, the additional difficulty of analyzing the generalization performance of GNNs lies in the fact that each label $y_n$ is correlated with all the input features that are connected to node $v_n$, as shown in the risk functions in \eqref{eqn:linear_regression} and \eqref{eqn: classification}.

Note that our model with $K=1$ is equivalent to  the one-hidden-layer convolutional network (GCN)  \cite{KW17} for binary classification. 
To study the  multi-class classification problem,  the GCN model in  \cite{KW17} has $M$ nodes for $M$ classes in the second layer and employs the softmax activation function at the output. Here, our model has a pooling layer and uses the sigmoid function for binary classification.
Moreover,
 we consider both regression and binary classification problems   using the same model architecture with different activation functions. 
We consider one-hidden-layer networks following the state-of-art works in NNs \cite{DLTPS17,DLT17,BG17,ZSJB17,ZSD17,FCL19} and GNNs \cite{DHPSWX19,VZ19} because the theoretical analyses are extremely complex and still being developed for multiple hidden layers.

\section{Proposed Learning Algorithm}
 

  In what follows, we illustrate the algorithm used for solving problems \eqref{eqn:linear_regression} and \eqref{eqn: classification}, summarized in Algorithm \ref{Alg}. Algorithm \ref{Alg} has two   components: a) accelerated gradient descent and b) tensor initialization. We initialize  $\bfW$  using the tensor initialization method    \cite{ZSJB17} with minor modification for GNNs  and update iterates by the Heavy Ball method \cite{P87}.


	\textbf{Accelerated gradient descent}. Compared with the vanilla GD method,
  each iterate in the Heavy Ball method is updated along the combined directions of both the gradient and the moving direction of the previous iterates. 
 	 Specifically, one computes the difference of the  estimates in the previous two iterations, and the difference is scaled by a constant $\beta$. This additional  momentum term   is added to the gradient descent update. When $\beta$ is $0$, AGD reduces to GD.
	 
	During each iteration, a fresh subset 
	of data is applied to estimate the gradient.  The assumption of disjoint subsets   is  standard     to simplify the analysis \cite{ZSD17,ZSJB17} but  not necessary in numerical experiments. 
	\begin{algorithm}[h]
		\caption{Accelerated Gradient Descent Algorithm with Tensor Initialization}
		\label{Alg}
		\begin{algorithmic}[1]
			\STATE \textbf{Input:} $\bfX$, $\big\{y_n\big\}_{n\in\Omega}$, $\bfA$, 
			the  step size $\eta$, the momentum constant $\beta$, and the  error tolerance  $\varepsilon$;
			\STATE \textbf{Initialization:} Tensor   Initialization via Subroutine 1;
			\STATE {Partition} ${\Omega}$ into $T=\log(1/\varepsilon)$ disjoint subsets, denoted as $\{{\Omega}_t\}_{t=1}^T$;
			\FOR  {$t=1, 2, \cdots, T$}		
			\STATE $\bfW^{(t+1)}=\bfW^{(t)}-\eta\nabla \hat{f}_{{\Omega}_t}(\bfW^{(t)}) +\beta(\bfW^{(t)}-\bfW^{(t-1)})$
			\ENDFOR
		\end{algorithmic}
	\end{algorithm}
	
	
  	\textbf{Tensor initialization}.  The main idea of the tensor initialization method \cite{ZSJB17} is to utilize the homogeneous property of an activation function such as ReLU to estimate the magnitude and direction separately for each $\bfw_j^*$ with $j\in[K]$. 
   A non-homogeneous function can be approximated by piece-wise linear functions,  if the function is strictly monotone with lower-bounded derivatives \cite{FCL19}, like the sigmoid function. Our  initialization is similar to those in \cite{ZSJB17,FCL19} for NNs with some definitions {are changed} to handle the graph structure, and the initialization process is summarized in Subroutine 1.   
   
Specifically, following \cite{ZSJB17},  we define a special outer product, denoted by $\widetilde{\otimes}$, such that for any vector $\bfv\in \mathbb{R}^{d_1}$ and $\bfZ\in \mathbb{R}^{d_1\times d_2}$,   
 \begin{equation}
 	\bfv\widetilde{\otimes} \bfZ=\sum_{i=1}^{d_2}(\bfv\otimes \bfz_i\otimes \bfz_i +\bfz_i\otimes \bfv\otimes \bfz_i + \bfz_i\otimes \bfz_i\otimes \bfv),
 \end{equation} 
 where $\otimes$ is the outer product and $\bfz_i$ is the $i$-th column of $\bfZ$.
 Next, we define
 \footnote{$\mathbb{E}_{\bfX}$ stands for the expectation over the distribution of random variable $\bfX$.}
\begin{equation}\label{eqn: M_1}
\bfM_{1} = \mathbb{E}_{\bfX}\{y \bfx \} \in \mathbb{R}^{d},
\end{equation}
\begin{equation}\label{eqn: M_2}
\bfM_{2} = \mathbb{E}_{\bfX}\Big\{y\big[(\bfa_n^T\bfX)\otimes (\bfa_n^T\bfX)-\bfI\big]\Big\}\in \mathbb{R}^{d\times d},
\end{equation}
\begin{equation}\label{eqn: M_3}
\bfM_{3} = \mathbb{E}_{\bfX}\Big\{y\big[(\bfa_n^T\bfX)^{\otimes 3}- (\bfa_n^T\bfX)\widetilde{\otimes} \bfI  \big]\Big\}\in \mathbb{R}^{d\times d \times d},
\end{equation}
where $\bfz^{\otimes 3} := \bfz \otimes \bfz \otimes \bfz$. \textcolor{black}{The tensor $M_3$ is used to identify the directions of $\{w_j^*\}_{j=1}^K$. $M_1$ depends on both the magnitudes and directions of $\{w_j^*\}_{j=1}^K$. We will sequentially estimate the directions  and magnitudes of
$\{w_j^*\}_{j=1}^K$ from $M_3$ and $M_1$. The matrix $M_2$ is  used to identify
the subspace spanned by $w_j^*$. We will project to this subspace to reduce
the computational complexity of decomposing $M_3$.}
\floatname{algorithm}{Subroutine}
 \setcounter{algorithm}{0}
\begin{algorithm}[t]
	\caption{Tensor Initialization Method}\label{Alg: initia}
	\begin{algorithmic}[1]
		\STATE \textbf{Input:}  $\bfX$, $\big\{y_n\big\}_{n\in\Omega}$ and $\bfA$;\\
		\STATE Partition $\Omega$ into three disjoint subsets $\Omega_1$, $\Omega_2$, $\Omega_3$;\\
		\STATE Calculate $\widehat{\bfM}_{1}$, $\widehat{\bfM}_{2}$ {following \eqref{eqn: M_1}, \eqref{eqn: M_2}} using $\Omega_1$, $\Omega_2$, respectively;\\
		\STATE Estimate $\widehat{\bfV}$  by orthogonalizing the eigenvectors with respect to the $K$ largest eigenvalues of  $\widehat{\bfM}_{2}$;\\
		\STATE Calculate $\widehat{\bfM}_{3}(\widehat{\bfV},\widehat{\bfV},\widehat{\bfV})$ {using \eqref{eqn: M_3_v1}} through $\Omega_3$;\\
		\STATE Obtain $\{ \widehat{\bfu}_j \}_{j=1}^K$ via {tensor decomposition method \cite{KCL15}};\\
		\STATE Obtain $\widehat{\boldsymbol{\alpha}}$ by solving  optimization problem $\eqref{eqn: int_op}$;\\
		\STATE \textbf{Return:} $\bfw^{(0)}_j=\widehat{\alpha}_j\widehat{\bfV}\widehat{\bfu}_j$, {$j=1,...,K$.}
	\end{algorithmic}
\end{algorithm} 

\textcolor{black}{Specifically, the values of $\bfM_1$, $\bfM_2$ and $\bfM_3$ are all estimated through samples, and let $\widehat{\bfM}_{1}$, $\widehat{\bfM}_{2}$, $\widehat{\bfM}_{3}$ denote the corresponding estimates of these high-order momentum.}
Tensor decomposition method \cite{KCL15} provides the estimates of \textcolor{black}{the vectors $\bfw_j^*/\|\bfw_j^*\|_2$ 
from $\widehat{\bfM}_{3}$, and the estimates are denoted as $\widehat{\overline{\bfw}}_{j}^*$.}

However, the computation complexity of estimate through $\widehat{\bfM}_{3}$ depends on poly($d$). To reduce the computational complexity of tensor decomposition, $\widehat{\bfM}_{3}$ is in fact first projected to a lower-dimensional tensor \cite{ZSJB17} through a matrix $\widehat{\bfV}\in \mathbb{R}^{d\times K}$. \textcolor{black}{$\widehat{\bfV}$ is the estimation of matrix $\bfV$ and can be computed from the right singular vectors of $\widehat{\bfM}_{2}$.} 
\textcolor{black}{
The column vectors of ${\bfV}$ form a basis for the subspace spanned by $\{\bfw_{j}^* \}_{j=1}^{K}$, which indicates that $\bfV\bfV^T\bfw_j^* = \bfw_j^*$ for any $j\in[K]$. }
Then, from \eqref{eqn: M_3}, 
$\bfM_{3}(\widehat{\bfV}, \widehat{\bfV}, \widehat{\bfV})\in\mathbb{R}^{K\times K \times K}$  is defined as
 \begin{equation}\label{eqn: M_3_v1}
 \begin{split}
 	 &\bfM_{3}(\widehat{\bfV}, \widehat{\bfV}, \widehat{\bfV}) 
 	:= \mathbb{E}_{\bfX} \Big\{ y\big[(\bfa_n^T\bfX\widehat{\bfV})^{\otimes 3}
 	   - (\bfa_n^T\bfX\widehat{\bfV})\widetilde{\otimes} \bfI  \big]\Big\}.
\end{split}
\end{equation}
 
Similar to the case of $\widehat{\bfM}_{3}$, by applying the tensor decomposition method in $\widehat{\bfM}_{3}(\widehat{\bfV}, \widehat{\bfV}, \widehat{\bfV})$, one can obtain a series of normalized vectors, denoted as $\{\widehat{\bfu}_j \}_{j=1}^K \in \mathbb{R}^{K}$, which are the estimates of $\{\bfV^T\overline{\bfw}_j^*\}_{j=1}^K$.
Then, $\widehat{\bfV}\widehat{\bfu}_j$ is an estimate of $\overline{\bfw}_j^*$ since $\overline{\bfw}^*_j$ lies in the column space of $\bfV$ with $\bfV\bfV^T\overline{\bfw}_j^*=\overline{\bfw}_j^*$. 

From \cite{ZSJB17}, \eqref{eqn: M_1} can be written as 
\begin{equation}\label{eqn: M_1_2}
\bfM_1 = \sum_{j=1}^{K} \psi_1(\overline{\bfw}_{j}^*){\|\bfw_j^*\|_2{\bar{\bfw}}_j^*},
\end{equation}
where $\psi_1$ depends on the distribution of $\bfX$. 
Since the distribution of $\bfX$ is known, the values of $\psi(\widehat{\overline{\bfw}}^*_j)$ can be calculated exactly. Then, the magnitudes of $\bfw_j^*$'s are estimated through solving the following optimization problem:
\begin{equation}\label{eqn: int_op}
	\widehat{\boldsymbol{\alpha}}=\arg\min_{\boldsymbol{\alpha}\in\mathbb{R}^K}:\quad  \Big|\widehat{\bfM}_{1} - \sum_{j=1}^{K}\psi(\widehat{\overline{\bfw}}^*_j) \alpha_j \widehat{\overline{\bfw}}^*_{j}\Big|.
\end{equation}
 Thus,  $\bfW^{(0)}$ is given as 
 $\begin{bmatrix}
 \widehat{\alpha}_1\widehat{\overline{\bfw}}^*_{1}, &\cdots,& \widehat{\alpha}_K\widehat{\overline{\bfw}}^*_{K}
 \end{bmatrix}$.


	\section{Main Theoretical Results}
	
	
	Theorems 1 and 2 state our major results about the GNN model  for regression  and binary classification, respectively. Before formally presenting the results, we first summarize the key findings as follows. 
	
	\textbf{1. Zero generalization error of the learned model.} Algorithm ~\ref{Alg} can return $\bfW^*$ exactly for   regression (see \eqref{eqn: linear_convergence_lr})   and approximately for binary classification (see \eqref{eqn: linear_convergence2_cl}). 
	Specifically, since $\bfW^*$ is {often} not  a solution to \eqref{eqn: classification}, 
	Algorithm~\ref{Alg} returns a critical point $\widehat{\bfW}$ that is sufficiently close to $\bfW^*$, and the distance decreases with respect to the number of samples in the order of $\sqrt{1/|\Omega|}$. 
	Thus, with a sufficient number of samples, $\widehat{\bfW}$ {will be close to $\bfW^*$ and}  achieves a zero generalization error approximately for binary classification.  Algorithm~\ref{Alg} always returns $\bfW^*$ exactly for   regression, a zero \textcolor{black}{generalization} error is thus achieved.
	


\textbf{2. 	Fast linear convergence of  Algorithm~\ref{Alg}.}  Algorithm~\ref{Alg} is proved to converge linearly to $\bfW^*$ for regression and $\widehat{\bfW}$ for classification, as shown in \eqref{eqn: linear_convergence_lr} and \eqref{eqn:converge2}. That means the distance of the estimate during the iterations to $\bfW^*$ (or $\widehat{\bfW}$) decays exponentially. Moreover,  Algorithm~\ref{Alg} converges faster than the vanilla GD. The rate of convergence is $1-\Theta\big(\frac{1}{\sqrt{K}}\big)$
for regression
\footnote{$f(d)=O(g(d))$ means that if for some constant $C>0$,  $f(d)\leq Cg(d)$ holds when $d$ is sufficiently large.   $f(d)=\Theta(g(d))$ means that   for some constants $c>0$ and $C>0$,  $cg(d)\leq f(d)\leq Cg(d)$ holds when $d$ is sufficiently large.}
and $1-\Theta\big(\frac{1}{K}\big)$ for classification, where $K$ is the number of filters in the hidden layer. In comparison, the convergence rates of GD are $1-\Theta\big(\frac{1}{K}\big)$ and $1-\Theta\big(\frac{1}{K^2}\big)$, respectively. Note that a smaller value of the rate of convergence corresponds to faster convergence. 
 We remark that this is the first theoretical guarantee of  AGD methods for learning GNNs.

\textbf{3. Sample complexity analysis.} $\bfW^*$ can be estimated exactly for regression and approximately for classification, provided that the number of  samples is in the order of $(1+\delta^2) \textrm{poly}(\sigma_1(\bfA), K) d\log N\log(1/\varepsilon)$, as shown in \eqref{eqn: sample_complexity_lr} and \eqref{eqn: sample_complexity_cl}, where $\varepsilon$ is the desired estimation error tolerance. $\bfW^*$ has $Kd$ parameters,  where $K$ is the number of nodes in the hidden layer, and $d$ is the feature dimension.  Our sample complexity is order-wise optimal with respect to $d$  and  only logarithmic with respect to the total number of features $N$. 
{We further show that} the sample complexity is also positively   associated  with $\sigma_1(\bfA)$ and $\delta$.  That characterizes the relationship between the sample complexity and graph structural properties.  From Lemma \ref{Lemma: sigma_1}, we know that given $\delta$, $\sigma_1(\bfA)$ is positively correlated with the average node degree $\delta_{\text{ave}}$. Thus, the required number of samples increases when the maximum and average degrees of the graph increase. That coincides with the intuition that more edges in the graph corresponds to the stronger dependence of the labels on neighboring features, {thus requiring} more samples to learn these dependencies. Our sample complexity quantifies this intuition explicitly. 
 
Note that the graph structure affects this bound only through $\sigma_1(\bfA)$ and $\delta$. Different graph structures may require a similar number of samples to estimate $\bfW^*$, as long as they have similar $\sigma_1(\bfA)$ and $\delta$. We will verify this property on different graphs numerically in Figure~\ref{Figure: N_vs_graph}.
  
        \begin{lemma}\label{Lemma: sigma_1}
 	Give an un-directed graph $\mathcal{G}=\{\mathcal{V}, \mathcal{E} \}$ and the   normalized adjacency matrix $\bfA$ as defined in Section \ref{Section: pf_1}, 
 	the largest singular value $\sigma_1(\bfA)$ of $\bfA$ satisfies
		\begin{equation}
			\frac{ 1+\delta_{\textrm{ave}} }{1+ \delta_{\max} } \le \sigma_1(\bfA)\le 1, 
		\end{equation}
 		where $\delta_{\text{ave}}$ and $\delta$ are the average and maximum node degree, respectively. 
	\end{lemma}
 
 \subsection{Formal theoretical guarantees}\label{sec:theorem}
	To formally present the results, some parameters in the results are defined as follows. $\sigma_j(\bfW^*)$ ($j\in [N]$) is the $j$-th singular value of $\bfW^*$. $\kappa=\sigma_1(\bfW^*)/\sigma_K(\bfW^*)$ is the conditional number of $\bfW^*$. 
	$\gamma$ is defined as $\prod_{j=1}^K\sigma_j(\bfW^*)/\sigma_K(\bfW^*)$. For a fixed $\bfW^*$, both $\gamma$ and $\kappa$ can be viewed as constants and do not affect the order-wise analysis.

	\begin{theorem}\label{Thm: major_thm_lr}
		(Regression) Let $\{ \bfW^{(t)} \}_{t=1}^T$ be the sequence generated by  Algorithm \ref{Alg} to solve \eqref{eqn:linear_regression} with $\eta= {K}/{(8\sigma_1^2(\bfA))}$. 
		Suppose the number of samples satisfies 
		\begin{equation}\label{eqn: sample_complexity_lr}
			|\Omega|\ge C_1\varepsilon_0^{-2} \kappa^9\gamma^2 { (1+\delta^2)\sigma_1^4(\bfA)} K^8 d\log N \log(1/\varepsilon)
		\end{equation} 
		for some constants $C_1>0$ and $\varepsilon_0\in(0,\frac{1}{2})$. Then $\{ \bfW^{(t)} \}_{t=1}^{T}$ converges linearly to $\bfW^*$ with probability at least $\textcolor{black}{1-K^2T}\cdot N^{-10}$ as  
		\begin{equation}\label{eqn: linear_convergence_lr}
		\begin{gathered}
		\|	\W[t]-\bfW^*\|_2
		\le\nu(\beta)^t\|	\W[0]-\bfW^*\|_2, and\\
		\|	\W[T]-\bfW^*\|_2
		\le \varepsilon \|\bfW^* \|_2 
		\end{gathered}
		\end{equation}
		where $\nu(\beta)$ is the  rate of convergence that depends on $\beta$.
		Moreover, we have
		\begin{equation}\label{eqn: accelerated_rate_lr}
		\nu(\beta)< \nu(0) \text{\quad for some small nonzero } \beta.
		\end{equation} 
		Specifically, let $\beta^* = \Big(1-\sqrt{\frac{1-\varepsilon_0}{88\kappa^2 \gamma }}\Big)^2$, we have
		\begin{equation}\label{eqn:rate}
		\begin{gathered}
		 \nu(0)\ge 1-\frac{1-\varepsilon_0}{{88\kappa^2 \gamma K}},
		\nu(\beta^*)= 1-\frac{1-\varepsilon_0}{\sqrt{88\kappa^2 \gamma  K}}.	
		\end{gathered}
		\end{equation}
	\end{theorem}
	\begin{theorem}\label{Thm: major_thm_cl}
	(Classification)	Let $\{ \bfW^{(t)} \}_{t=1}^T$ be the sequence generated by Algorithm \ref{Alg} to solve \eqref{eqn: classification} with $\eta= {1}/{(2\sigma_1^2(\bfA))}$. 
		Suppose the number of samples satisfies 
		\begin{equation}\label{eqn: sample_complexity_cl}
		|\Omega|\ge C_2\varepsilon_0^{-2}(1+\delta^2)\kappa^8\gamma^2 \sigma_1^4(\bfA)K^8d\log N \log(1/\varepsilon)
		\end{equation} 
		for some positive constants $C_2$ and $\varepsilon_0\in(0,1)$. Then, let $\widehat{\bfW}$ be the nearest critical point of  \eqref{eqn: classification} to $\bfW^*$, we have that $\{ \bfW^{(t)} \}_{t=1}^{T}$ converges linearly to $\widehat{\bfW}$ with probability at least \textcolor{black}{$1-K^2T\cdot N^{-10}$} as  
		\begin{equation}\label{eqn:converge2}
		\begin{gathered}
		\|\bfW^{(t)} -\widehat{\bfW} \|_2 \le \nu(\beta)^t\|\bfW^{(0)} -\widehat{\bfW}\|_2, \textrm{ and }\\
		\|\bfW^{(T)} -\widehat{\bfW} \|_2 \le \varepsilon\|\bfW^{(0)} -\widehat{\bfW}\|_2.
		\end{gathered}
		\end{equation}
		The distance between $\widehat{\bfW}$ and $\bfW^*$ is bounded by
		\begin{equation} \label{eqn: linear_convergence2_cl}
		\| \widehat{\bfW} -\bfW^*  \|_2 \le C_3(1-\varepsilon_0)^{-1}\kappa^2 \gamma K \sqrt{\frac{(1+\delta^2)d\log N}{|\Omega|}},
		\end{equation}
		where $\nu(\beta)$ is the rate of convergence   that depends on $\beta$, and $C_3$ is some positive constant.
		Moreover, we have
		\begin{equation}\label{eqn: accelerated_rate_cl}
		\nu(\beta)< \nu(0) \text{\quad for some small nonzero } \beta,
		\end{equation} 
		Specifically, let $\beta^* = \Big(1-\sqrt{\frac{1-\varepsilon_0}{11\kappa^2\gamma K^2}}\Big)^2$, we have
		\begin{equation}\label{eqn: accelerated_rate1_cl}
		\begin{gathered}
		\nu(0)= 1-\frac{1-\varepsilon_0}{11\kappa^2\gamma K^2},
		\nu(\beta^*)= 1-\sqrt{\frac{1-\varepsilon_0}{11\kappa^2\gamma K^2}}.	
		\end{gathered}
		\end{equation}
	\end{theorem}

\subsection{Comparison with existing works}
 
Only \cite{VZ19,DHPSWX19} analyze the generalization error of one-hidden-layer GNNs in regression problems, while there is no existing work about the generalization error in classification problems. \cite{VZ19,DHPSWX19} show that the difference between the risks in the testing data and the training data
decreases in the order of $1/\sqrt{|\Omega|}$ as the sample size increases.  The GNN model in \cite{VZ19} only has one filter in the hidden layer, i.e., $K=1$, and the loss function is required to be a smooth function, excluding ReLU. Ref.~\cite{DHPSWX19} only considers infinitely wide GNNs.  In contrast, $\bfW^*$  returned by Algorithm\ref{Alg} can achieve zero risks for both training data and testing data in regression problems. Our results apply to an arbitrary number of filters and the ReLU activation function. 
 Moreover, this paper is the first work that characterizes the \textcolor{black}{generalization} error of GNNs for binary classification.

When $\delta$ is zero, our model reduces to one-hidden-layer NNs, and the corresponding sample complexity is $O\Big( \textrm{poly}(K) d\log N\log(1/\varepsilon)\Big)$. Our results are at least comparable to, if not better than, the state-of-art theoretical guarantees \textcolor{black}{that from the prespective of model estimation} for NNs. 
For example, \cite{ZSJB17} considers   one-hidden-layer NNs for regression  and proves the linear convergence of their algorithm to the ground-truth model parameters. The sample complexity of ~\cite{ZSJB17} is also linear in  $d$, but the activation function must be smooth. \cite{ZYWG18} considers one-hidden-layer 
NNs with the ReLU  activation function for regression, but the algorithm cannot converge to the ground-truth parameters exactly but up to a statistical error.  Our result in Theorem \ref{Thm: major_thm_lr} applies to the nonsmooth ReLU function and can recover $\bfW^*$ exactly.
   \cite{FCL19} considers   one-hidden-layer NNs for classification and proves linear convergence of their algorithm to a critical point sufficiently close to   $\bfW^*$ with the distance bounded by   $O(\sqrt{1/|\Omega|})$. The convergence rate in \cite{FCL19} is $1-\Theta({1}/{K^2})$, while Algorithm~\ref{Alg} has a faster convergence   rate of $1-\Theta({1}/{K})$.
\section{Numerical Results}
We verify our results on synthetic graph-structured data.
We consider four types of graph structures as shown in Figure ~\ref{Figure: graph structure}: (a) a connected-cycle graph {having} each node connecting to its $\delta$ closet neighbors;   (b) a two-dimensional grid {having} each node connecting to its nearest neighbors in axis-aligned directions;  (c) a random $\delta$-regular graph {having} each node connecting to $\delta$ other nodes randomly;   (d) a random graph with bounded degree   {having} each node degree selected from $0$ with probability $1-p$ and $\delta$ with probability $p$ for some $p\in[0,1]$. 
\begin{figure}[h]
    \centering
    \begin{subfigure}[h]{0.11\textwidth}
        \centering
        \includegraphics[width=1.\textwidth]{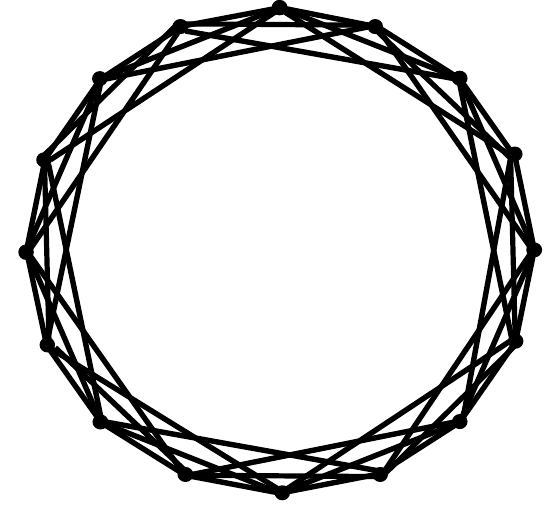}
        \caption{}
    \end{subfigure}%
    ~ 
    \begin{subfigure}[h]{0.11\textwidth}
        \centering
        \includegraphics[width=0.98\textwidth]{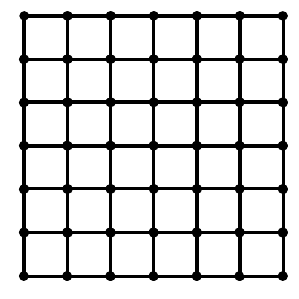}
        \caption{}
    \end{subfigure}
    ~
    \begin{subfigure}[h]{0.11\textwidth}
        \centering
        \includegraphics[width=1.\textwidth]{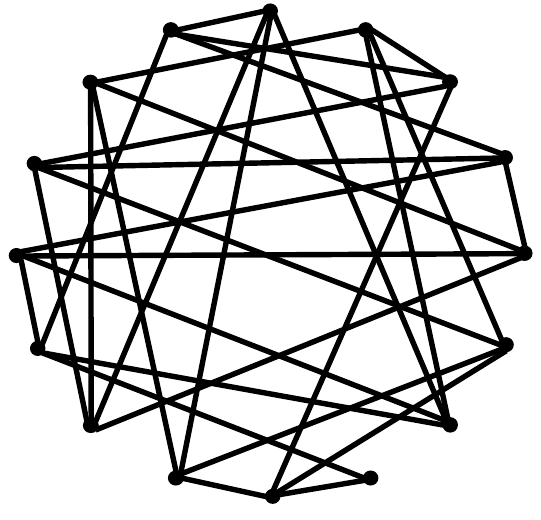}
        \caption{}
    \end{subfigure}%
    ~ 
    \begin{subfigure}[h]{0.11\textwidth}
        \centering
        \includegraphics[width=1.\textwidth]{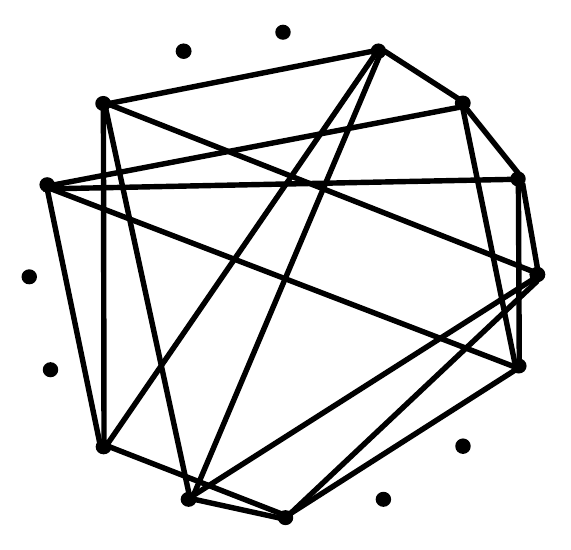}
        \caption{}
    \end{subfigure}
    \caption{Different graph structures: (a)  a connected-cycle graph, (b) a two-dimensional grid, (c) a random regular graph, (d) a random graph with a bounded degree.}
    \label{Figure: graph structure}
\end{figure}
The feature vectors $\{\bfx_n\}_{n=1}^{N}$ are randomly generated from the standard Gaussian distribution $\mathcal{N}(0 , \bfI_{d\times d})$. 
Each entry of  $\bfW^*$ is generated from   $\mathcal{N}(0, 5^2)$ independently. $\{z_n\}_{n=1}^N$ are computed based on \eqref{eqn: y_n}. 
The labels $\{y_n\}_{n=1}^N$ are generated by $y_n=z_n$ and  $\text{Prob}\{y_n=1\}=z_n$ for regression and classification problems, respectively.

During each iteration of Algorithm~\ref{Alg}, we use the whole training data  to calculate the gradient.
The initialization is randomly selected from $\big\{ \bfW^{(0)}\big| \| \bfW^{(0)} - \bfW^* \|_F/\|\bfW^* \|_F <0.5 \big\}$ to reduce the computation. As shown in \cite{FCL19,ZYWG18}, the
random initialization and the tensor initialization have very similar numerical performance. We consider the learning algorithm to be successful in estimation if the relative error,  defined as $\|\bfW^{(t)}-\bfW^*\|_F/\|\bfW^*\|_F$,  is less than $10^{-3}$, where $\bfW^{(t)}$ is the weight matrix returned by Algorithm ~\ref{Alg} when it terminates.


\subsection{Convergence rate }
We first verify the linear convergence of Algorithm~\ref{Alg}, 
 as shown in \eqref{eqn: linear_convergence_lr} and \eqref{eqn:converge2}.
\textcolor{black}{Figure \ref{Figure: convergence_k} (a) and (b) show the convergence rate of Algorithm \ref{Alg} when varying the number of nodes in the hidden layer $K$. The dimension $d$ of the feature vectors   is chosen as $10$, and the sample size $|\Omega|$ is chosen as $2000$. We consider the connected-cycle graph in Figure~\ref{Figure: graph structure} (a) with  $\delta = 4$. All cases   converge to $\bfW^*$ with the exponential decay. Moreover, from Figure \ref{Figure: convergence_k}, we can also see that the  rate of convergence is almost a linear function of  $1/\sqrt{K}$. That verifies our theoretical result of the convergence rate of $1-O(1/\sqrt{K})$ in \eqref{eqn:rate}.}
\begin{figure}[h]
    \centering
    \begin{subfigure}[h]{0.242\textwidth}
        \centering
        \includegraphics[width=1.0\textwidth]{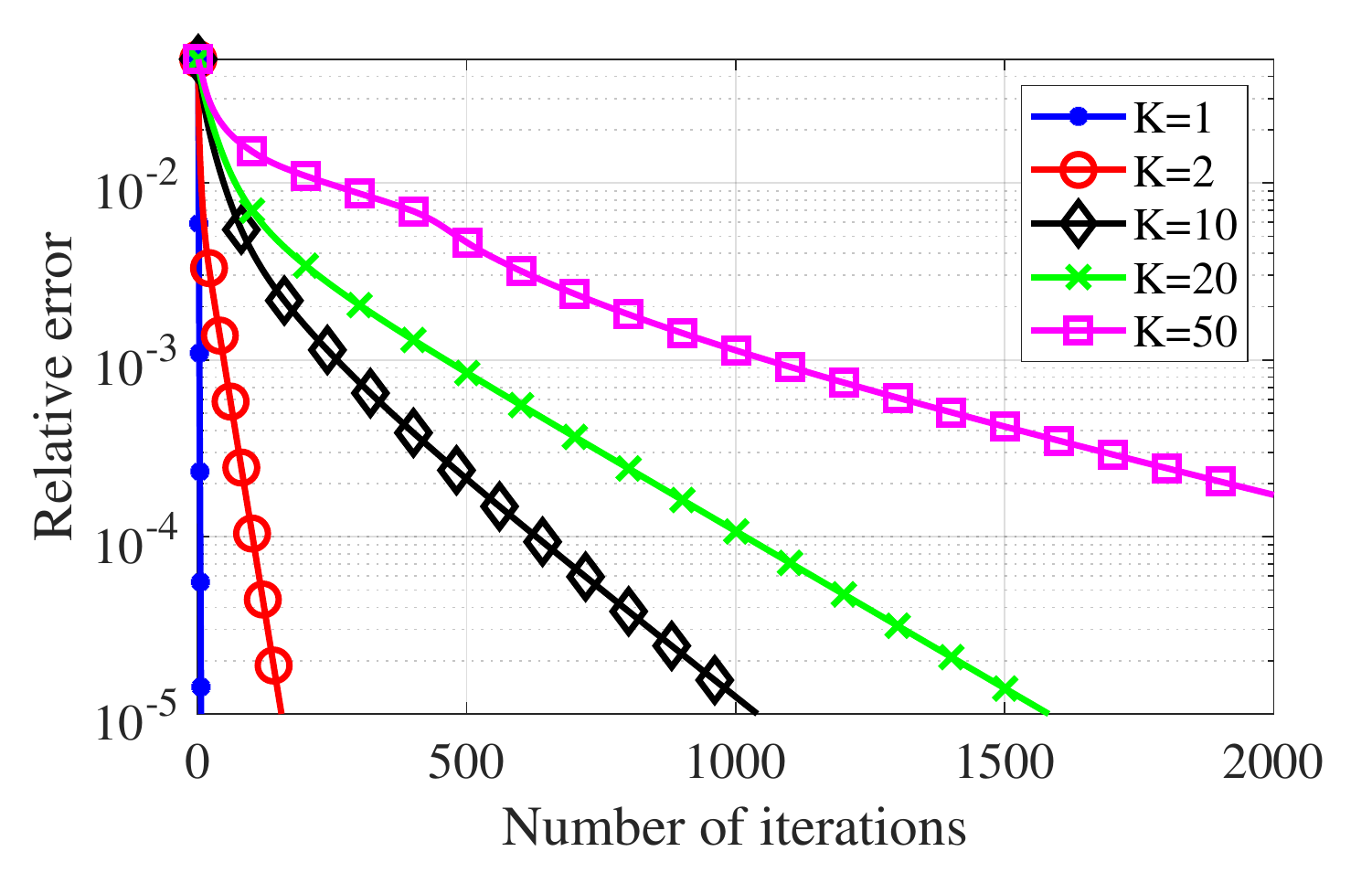}
        \caption{}
    \end{subfigure}%
    ~ 
    \begin{subfigure}[h]{0.242\textwidth}
        \centering
        \includegraphics[width=1.0\textwidth]{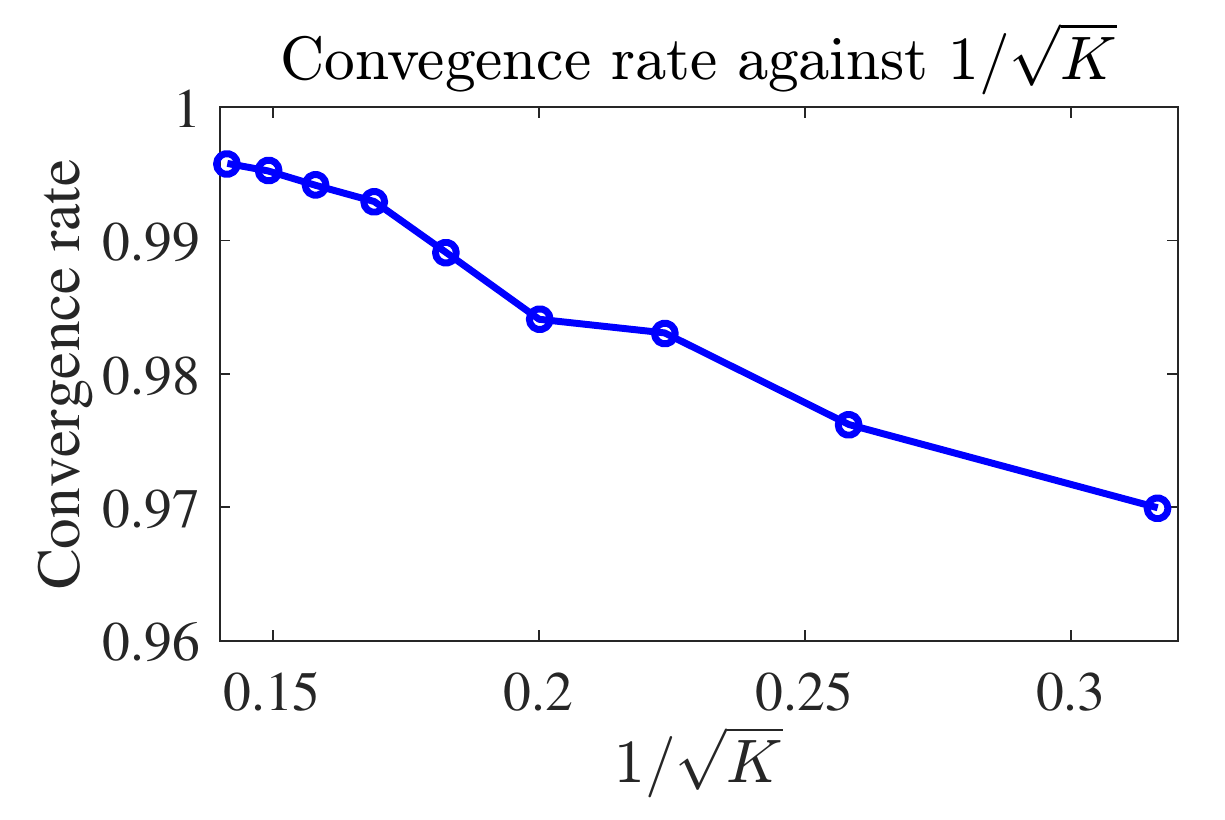}
        \caption{}
    \end{subfigure}
     	\caption{Convergence rate with different $K$ for a connected-circle graph}
 	\label{Figure: convergence_k}
\end{figure}

Figure~\ref{fig: N_vs_err} compares the  rates of convergence of AGD and GD in regression problems. We consider a connected-cycle graph with  $\delta=4$. The number of samples  $|\Omega|=500$,  $d=10$, and $K=5$.   Starting with the same initialization, we show the smallest number of iterations needed to reach a certain estimation error, and the results are averaged over $100$ independent trials. Both AGD and GD converge linearly. AGD requires a smaller number of the iterations than GD  to achieve the same relative error.

\begin{figure}[h]
	\centering
	\includegraphics[width=0.9\linewidth]{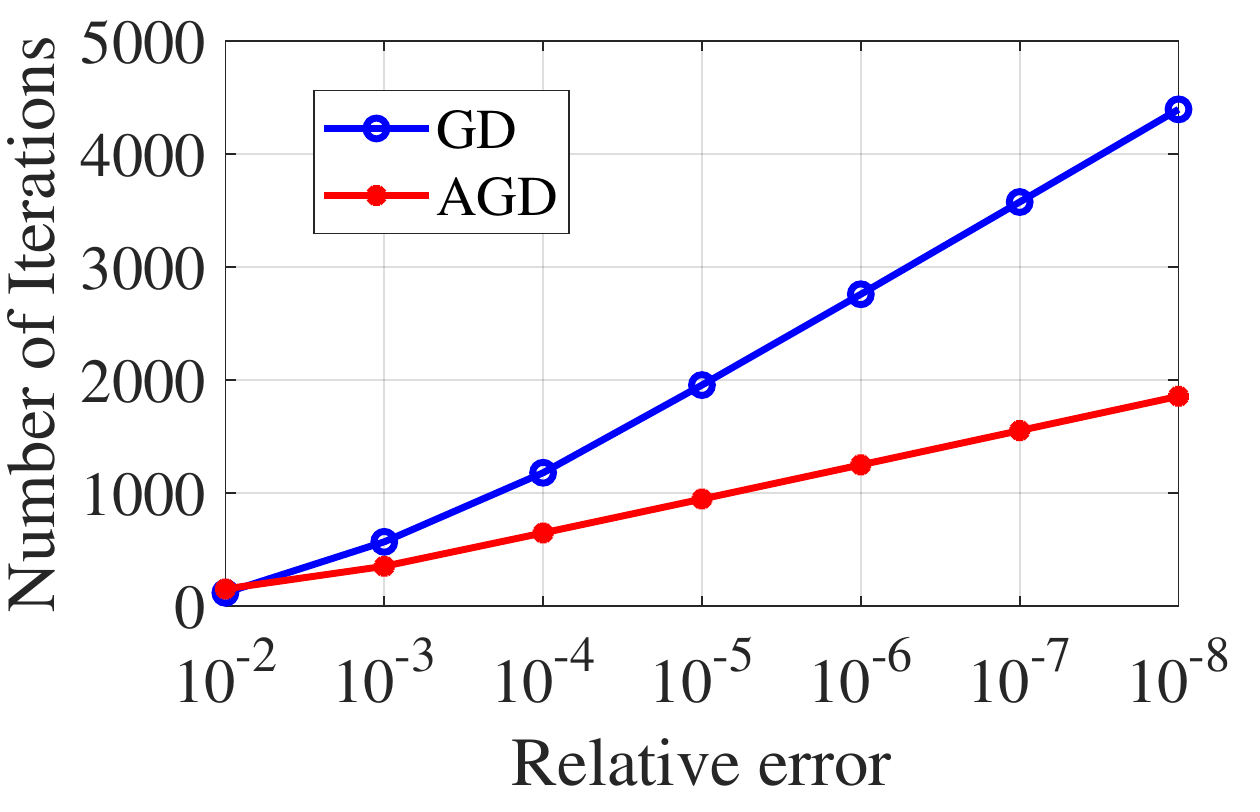}
	\caption{Convergence rates of AGD and GD}
	\label{fig: N_vs_err}
\end{figure}

\subsection{Sample complexity}
We next study the influence of   $d$, $\delta$, $\delta_{\text{ave}}$,  {and} different graph structures  on the  estimation performance 
of Algorithm \ref{Alg}. These relationships are summarized in the sample complexity analyses in \eqref{eqn: sample_complexity_lr} and \eqref{eqn: sample_complexity_cl} of section \ref{sec:theorem}. We have similar numerical results for both regression and classification,  and here we only present the regression case.


Figures \ref{Figure: Phrase transition of number of samples} (a) and (b)  show  the successful estimation rates  when the degree of graph  $\delta$ and the feature dimension  $d$ changes. We consider the connected-cycle graph in Figure~\ref{Figure: graph structure} (a), and  
$K$ is kept as 5. $d$ is  $40$ in Figure \ref{Figure: Phrase transition of number of samples} (a),  and  $\delta$ is   $4$ in Figure \ref{Figure: Phrase transition of number of samples} (b).
The results are averaged over $100$ independent trials. White block means all trials are successful while black block means all trials fail. We can see that the required number of samples for successful estimation increases as $d$ and $\delta$ increases. 
\begin{figure}[H]
    \centering
    \begin{subfigure}[h]{0.23\textwidth}
        \centering
        \includegraphics[width=0.96\textwidth]{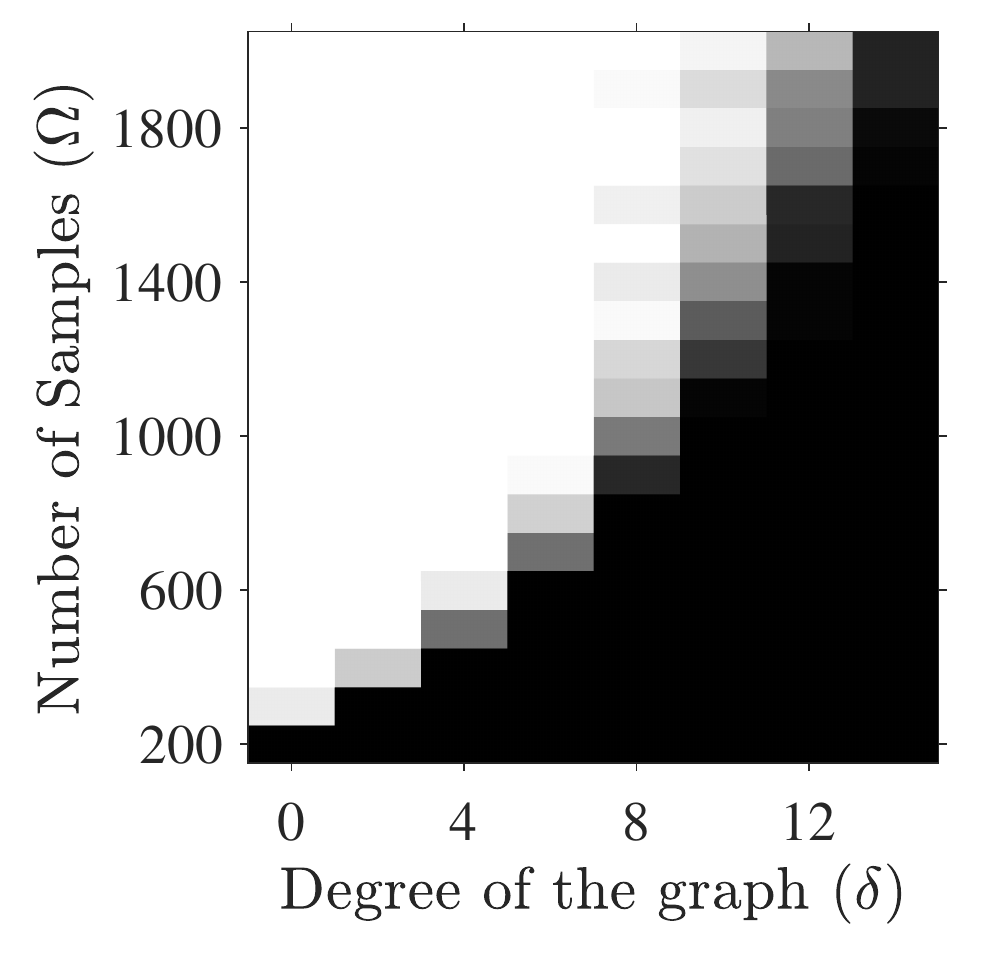}
        \caption{}
    \end{subfigure}%
    ~ 
    \begin{subfigure}[h]{0.23\textwidth}
        \centering
        \includegraphics[width=1.0\textwidth]{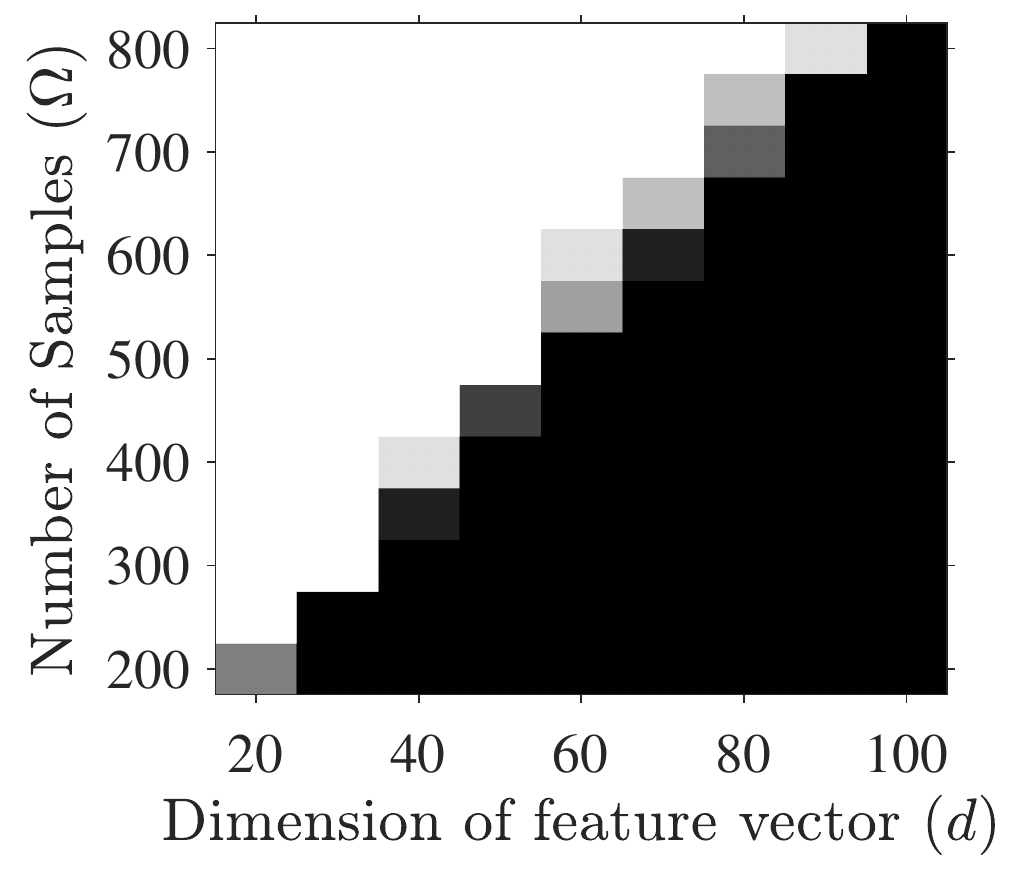}
        \caption{}
    \end{subfigure}
    \caption{Successful estimation rate for varying the required number of samples,  $\delta$, and $d$ in a connected-circle graphs }
    \label{Figure: Phrase transition of number of samples}
\end{figure}

Figure \ref{Figure: sigma} shows the success rate against the sample size $|\Omega|$ for the random graph in Figure~\ref{Figure: graph structure}(d) with different average node degrees.  We vary $p$ to change the average node degree  $\delta_{\text{ave}}$.  $K$ and $d$ are fixed as $5$ and $40$, respectively. The successful rate is calculated based on $100$ independent trials.    
We can see that more samples are needed for successful recovery for a larger $\delta_{\text{ave}}$ when the maximum degree $\delta$ is fixed. 
\begin{figure}[H]
	\centering
	\includegraphics[width=0.37\textwidth]{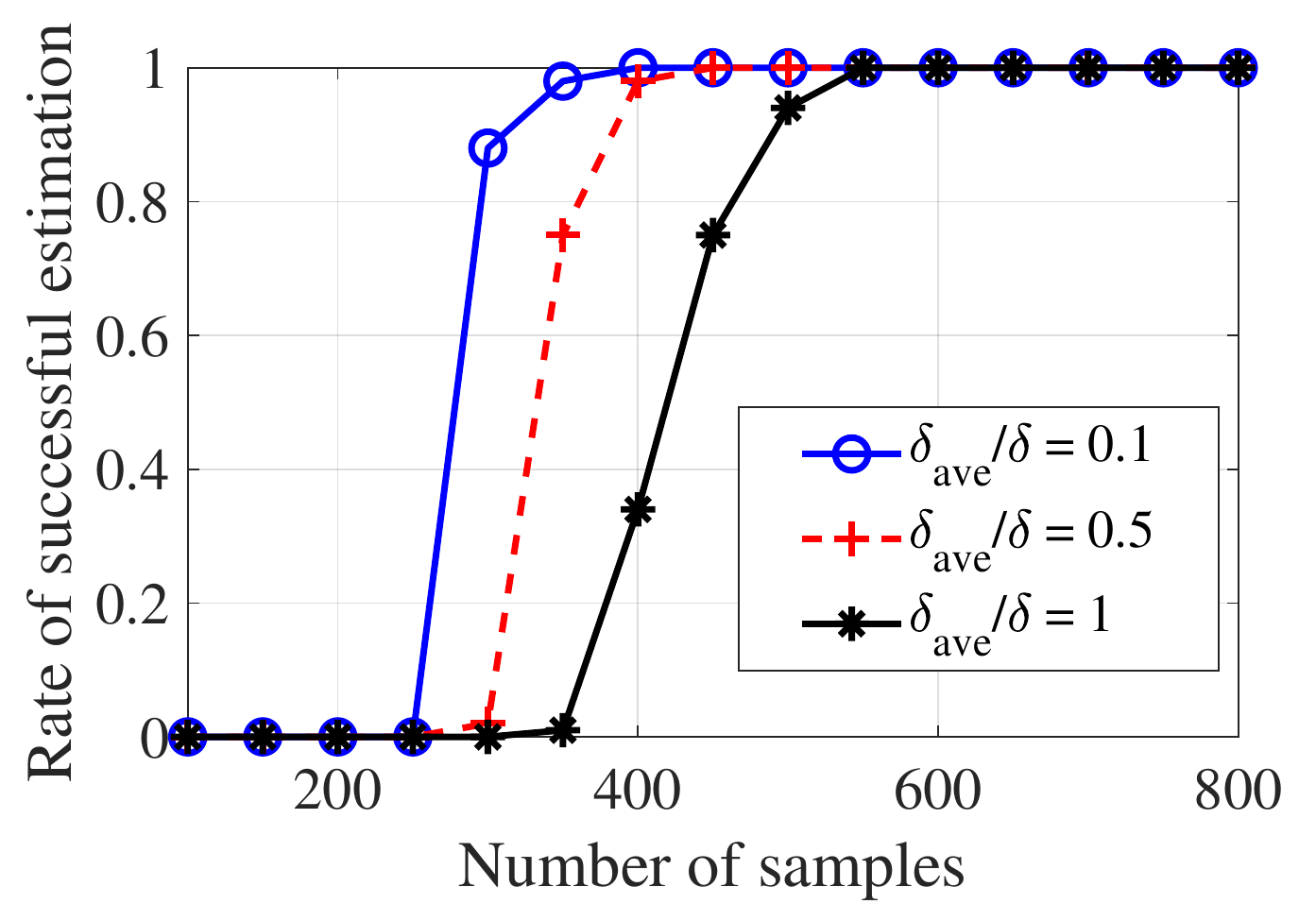}
	\caption{The success rate against the number of  samples for different $\delta_{\text{ave}}/ \delta$}
	\label{Figure: sigma}
\end{figure}

Figure \ref{Figure: N_vs_graph} shows the success rate against $|\Omega|$ for three different graph structures, including a connected cycle, a two-dimensional grid,  and  a random  regular graph  in Figure \ref{Figure: graph structure} (a), (b), and (c).  The maximum degrees of these graphs are all fixed with $\delta=4$. The average degrees of the connected-circle and the random $\delta$-regular graphs are also $\delta_{\text{ave}}=4$.  $\delta_{\text{ave}}$ is very close to $4$ for the two-dimensional grid when the graph size is large enough, because only the boundary nodes have smaller degrees, and the percentage of boundary nodes decays as the graph size increases. Then from Lemma \ref{Lemma: sigma_1}, we have $\sigma_1(\bfA)$ is $1$ for all these graphs. Although these graphs have different structures, the required numbers of samples to estimate $\bfW^*$ accurately are  the same, because both $\delta$  and $\sigma_1(\bfA)$ are the same. One can verify this property from Figure~\ref{Figure: N_vs_graph} where all three curves almost coincide.

\begin{figure}[H]
	\centering
	\includegraphics[width=0.37\textwidth]{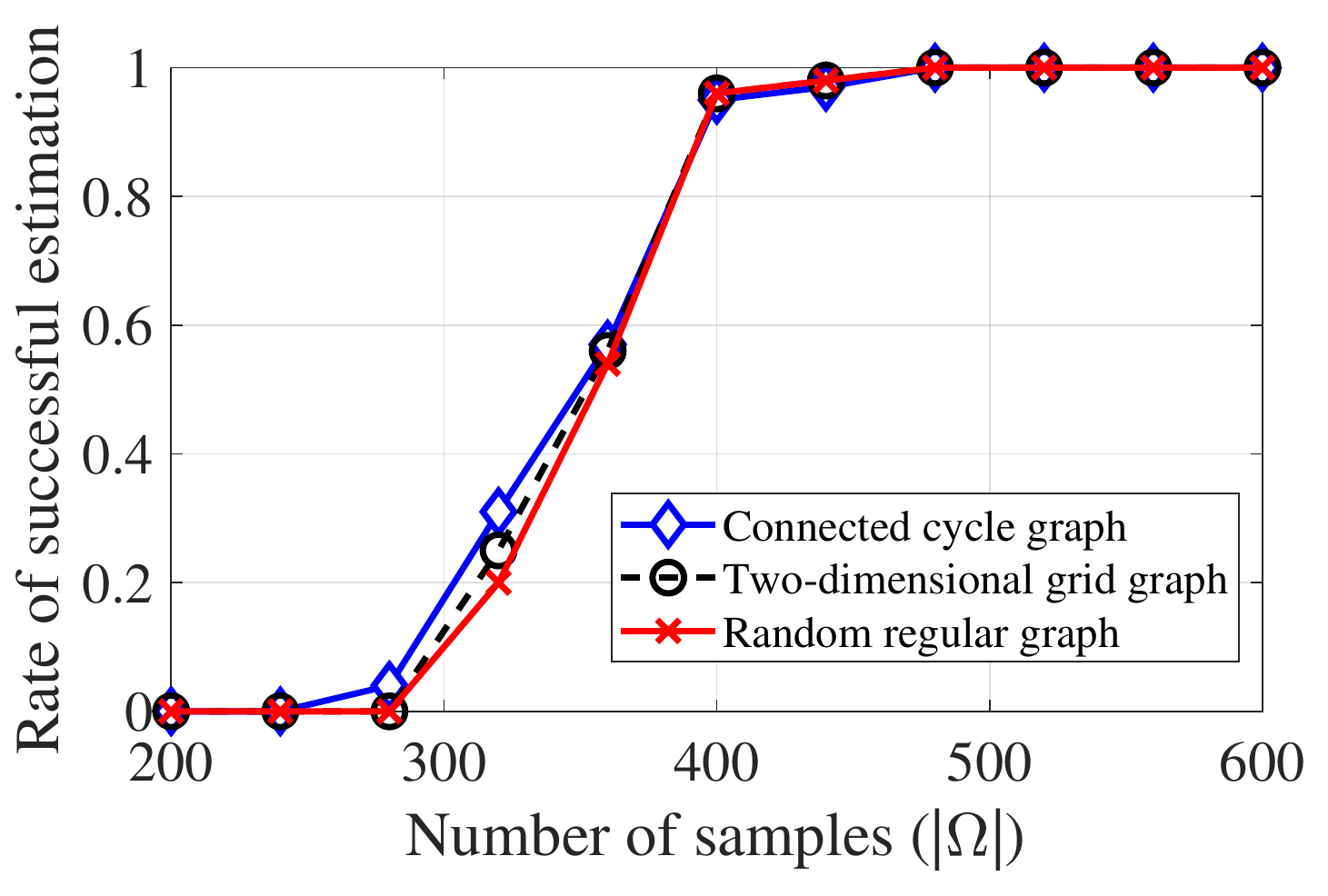}
	\caption{The success rate with respect to sample complexity for various graph structures} 
	\label{Figure: N_vs_graph}
\end{figure}

\subsection{Accuracy in learning $\bfW^*$}
We study the learning accuracy of $\bfW^*$, characterized in 
  \eqref{eqn: linear_convergence_lr} for regression and \eqref{eqn: linear_convergence2_cl}  for classification. 
For   regression problems, we simulate the general cases when the labels are noisy, i.e., $y_n = z_n+\xi_n$.
The noise $\{\xi_n \}_{n=1}^N$ are i.i.d. from $\mathcal{N}(0, \sigma^2)$,
and the noise level is measured by $\sigma/E_z$, where $E_z$ is the average energy of the noiseless labels $\{z_n\}_{n=1}^{N}$, calculated as $E_z = \sqrt{\frac{1}{N}\sum_{n=1}^{N}|z_n|^2}$.
The number of hidden nodes $K$ is   5, and the dimension of each feature   $d$ is  as $60$. We consider a connected-circle graph with $\delta=2$.  Figure \ref{Figure: noise_case} shows the performance of Algorithm \ref{Alg} in the noisy case. We can see that when the number of samples exceeds $Kd=300$, which is the degree of freedom of $\bfW^*$, the relative error decreases dramatically. Also, as $N$ increases, the relative error converges to the noise level. When there is no noise, the estimation of $\bfW^*$ is accurate.

\begin{figure}[h]
	\centering
	\includegraphics[width=0.37\textwidth]{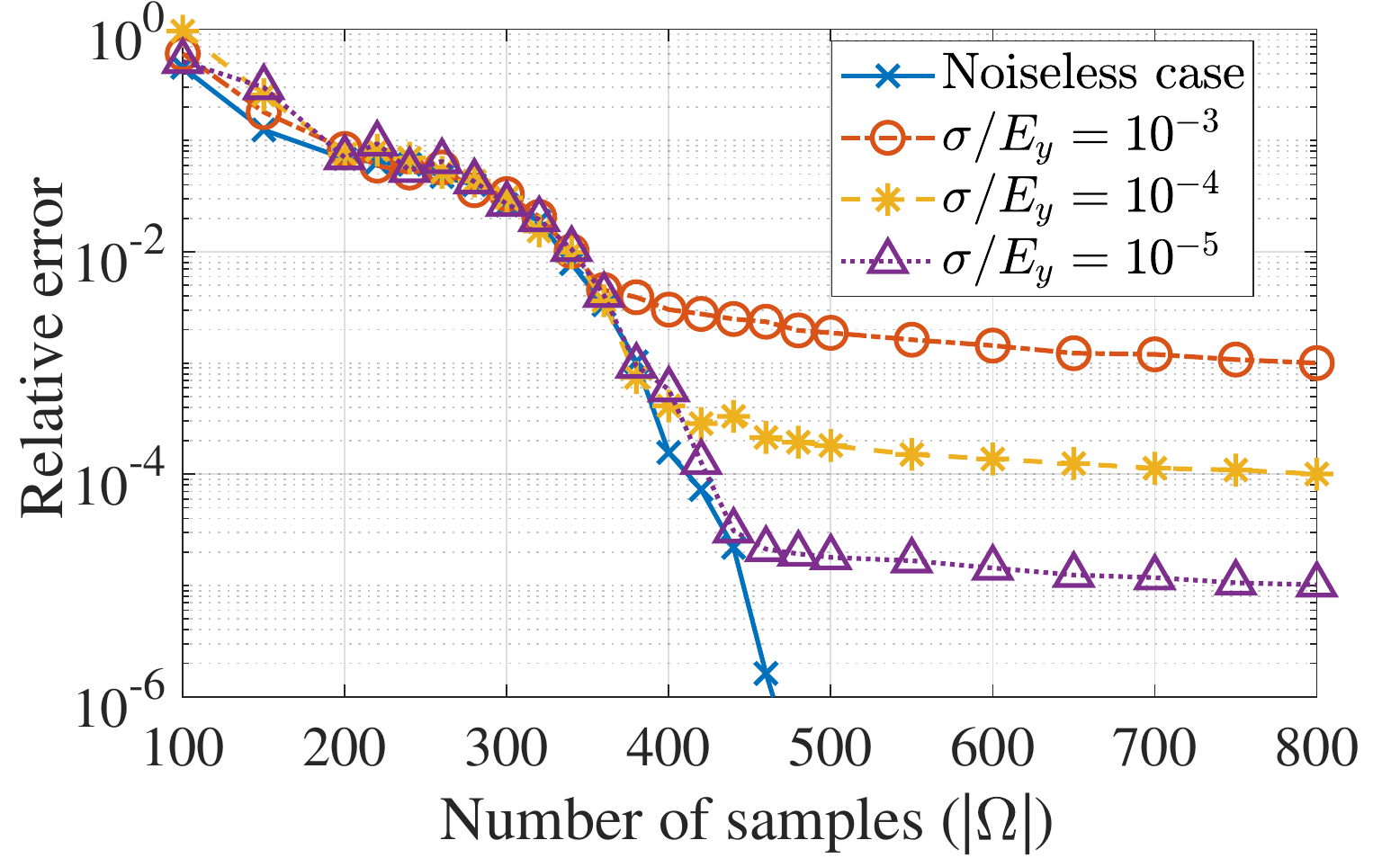}
	\caption{Learning accuracy of Algorithm~\ref{Alg} with noisy measurements for regression}
	\label{Figure: noise_case}
\end{figure}

For binary classification problems, Algorithm~\ref{Alg} returns  the nearest critical point $\widehat{\bfW}$ to $\bfW^*$. 
We show  the distance between the returned model and the ground-truth model $\bfW^*$   against the number of samples
in Figure \ref{Figure: N_error_cl}. 
We consider  a connected-cycle graph with the degree $\delta=2$. 
$K=3$ and $d=20$. 
The relative error $\|\widehat{\bfW}-\bfW^*\|_F/\|\bfW^*\|_F$ is averaged over  $100$ independent trials. 
We can see that the  distance between the returned model and the ground-truth model indeed decreases as the number of samples increases.

\begin{figure}[H]
	\centering
	\includegraphics[width=0.37\textwidth]{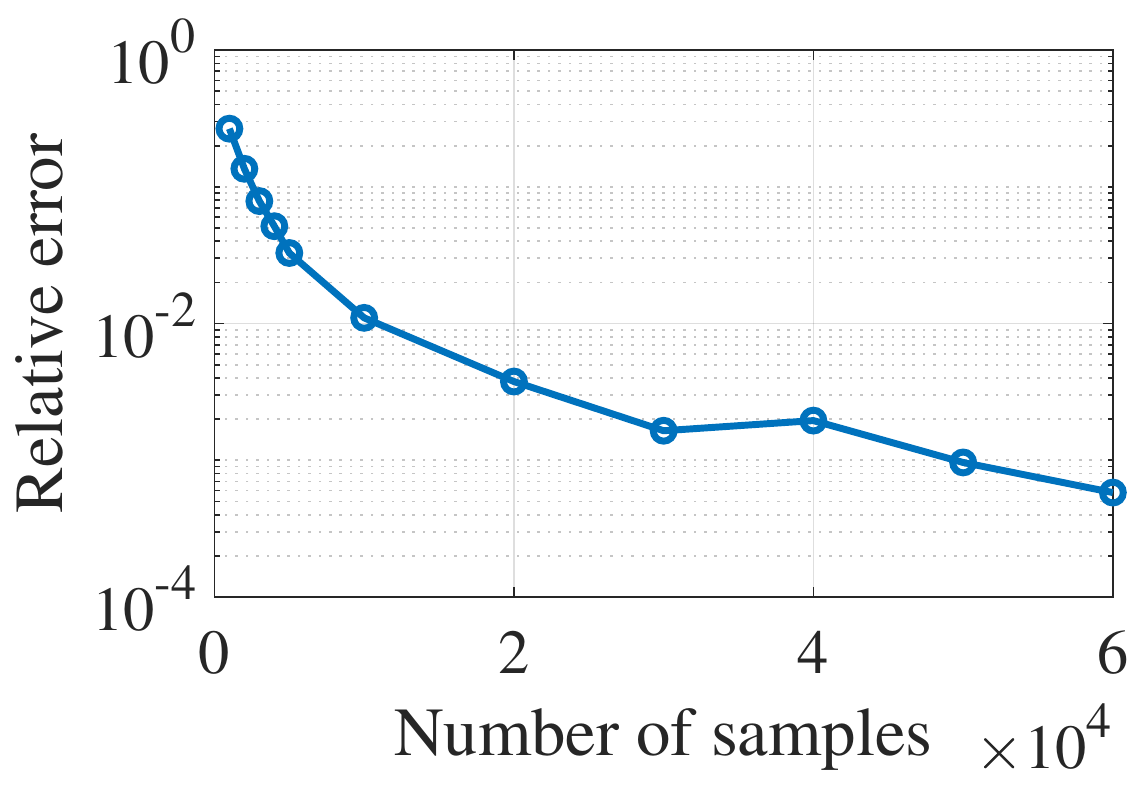}
	\caption{Distance between the returned model by Algorithm\ref{Alg} and the ground-truth model for binary  classification}
	\label{Figure: N_error_cl}
\end{figure}

\section{Conclusion}

Despite the practical success of graph neural networks in learning graph-structured data, the theoretical guarantee of the generalizability of graph neural networks is still elusive. Assuming the existence of a ground-truth model, this paper shows theoretically, for the first time,
learning a one-hidden-layer graph neural network  with a generation error that is zero for  regression or approximately zero for binary classification. With the tensor initialization, we prove that the accelerated gradient descent method converges to the ground-truth model  exactly for regression or approximately for binary classification at a linear rate. We also characterize the required number of training samples   as a function of the feature dimension, the model size, and the   graph structural properties. One future direction is to extend the analysis to multi-hidden-layer neural networks. 
 
 \section{Acknowledge }
 This work was supported by AFOSR FA9550-20-1-0122 and the Rensselaer-IBM AI Research Collaboration (http://airc.rpi.edu), part of the IBM AI Horizons Network (http://ibm.biz/AIHorizons).

\bibliography{ref, reference, MengWangPub}
\bibliographystyle{icml2020}

\newpage
\appendix
\onecolumn
\section{Proof of Theorem \ref{Thm: major_thm_lr}}
In this section, before presenting the proof of Theorem \ref{Thm: major_thm_lr}, we start with defining some useful notations.
Recall that in \eqref{eqn:linear_regression}, the empirical risk function for linear regression problem is defined as 
	\begin{equation}\label{eqn_app: linear_regression}
	\min_{\bfW}: \quad \hat{f}_{\Omega}(\bfW)=\frac{1}{2|\Omega|}\sum_{n\in \Omega}\Big| y_n - g(\bfW;\bfa_n^T\bfX) \Big|^2.
	\end{equation}
Population risk function, which is the expectation of the empirical risk function, is defined as
\begin{equation}\label{eqn_app: linear_regression_exp}
	\min_{\bfW}: \quad {f}_{\Omega}(\bfW)=\mathbb{E}_{\bfX}\frac{1}{2|\Omega|}\sum_{n\in \Omega}\Big| y_n - g(\bfW;\bfa_n^T\bfX) \Big|^2.
	\end{equation}
Then, the road-map of the proof can be summarized in the following three steps. 

First, we show the Hessian matrix of the population risk function $f_{\Omega_t}$ is positive-definite at ground-truth parameters $\bfW^*$ and then characterize the local convexity region of $f_{\Omega_t}$ near $\bfW^*$, which is summarized in Lemma \ref{Lemma: local_convexity}. 

Second, 
$\hat{f}_{\Omega_t}$ is non-smooth because of ReLU activation, but $f_{\Omega_t}$ is smooth.
Hence, we characterize the  gradient descent term as $\nabla \hat{f}_{\Omega_t}(\bfW^{(t)}) \simeq \langle \nabla^2f_{\Omega_t}(\widehat{\bfW}^{(t)}), \bfW^{(t)}-\bfW^* \rangle + \big(\hat{f}_{\Omega_t}({\bfW}^{(t)}) - f_{\Omega_t}({\bfW}^{(t)})\big)$. During this step, we need to apply concentration theorem to bound $\nabla\hat{f}_{\Omega_t}$  to its expectation $\nabla f_{\Omega_t}$ , which is summarized in Lemma \ref{Lemma: sampling_approximation_error}.

Third, we take the momentum term of $\beta(\bfW^{(t)}-\bfW^{(t-1)})$ into consideration and obtain the following recursive rule:
\begin{equation}\label{eqn: iteeee}
\begin{bmatrix}
\W[t+1]-\bfW^*\\
\W[t]-\bfW^*
\end{bmatrix}\\
=\bfL(\beta)
\begin{bmatrix}
\W[t]-\bfW^*\\
\W[t-1]-\bfW^*
\end{bmatrix}.
\end{equation}
Then, we know iterates $\bfW^{(t)}$ converge to the ground-truth with a linear rate which is the largest singlar value of matrix $\bfL(\beta)$. Recall that  AGD reduces to GD with $\beta=0$, so our analysis applies to GD method as well. We are able to show the convergence rate of AGD is faster than GD by proving the largest singluar value of $\bfL(\beta)$ is smaller than $\bfL(0)$ for some $\beta>0$.  Lemma \ref{Thm: initialization} provides the estimation error of $\bfW^{(0)}$ and sample complexity to guarantee $\|\bfL(\beta)\|_2$ is less than $1$ for $t=0$.

\begin{lemma}\label{Lemma: local_convexity}
	Let $f_{\Omega_t}$ be the population risk function in \eqref{eqn_app: linear_regression_exp} for regression problems, then for any $\bfW$ that satisfies
	\begin{equation}\label{eqn: initial_point_lr}
	\|\bfW^* -\bfW\|_2 \le \frac{\varepsilon_0 \sigma_K}{44\kappa^2\gamma K^2},
	\end{equation}
	the second-order derivative of $f_{\Omega_t}$ is bounded as 
	\begin{equation}
	\frac{(1-\varepsilon_0)\sigma_1^2(\bfA)}{11\kappa^2\gamma K^2}\bfI \preceq \nabla^2f_{\Omega_t}(\bfW)\preceq \frac{4\sigma_1^2(\bfA)}{K} \bfI.
	\end{equation}
\end{lemma}

\begin{lemma}\label{Lemma: sampling_approximation_error}
	Let $\hat{f}_{\Omega_t}$ and $f_{\Omega_t}$ be the  empirical and population risk functions in \eqref{eqn_app: linear_regression} and \eqref{eqn_app: linear_regression_exp} for regression problems, respectively. 
	Then, for any fixed point $\bfW$ satisfies \eqref{eqn: initial_point_lr}, we have \footnote{We use $f(d) \gtrsim(\lesssim) g(d)$ to denote there exists some positive constant $C$ such that $f(d)\ge(\le) C\cdot g(d)$ when $d$ is sufficiently large.} 
	\begin{equation}\label{eqn: lemma3}
	\left\|\nabla f_{\Omega_t}(\bfW)-\nabla \hat{f}_{\Omega_t}(\bfW)\right\|_2
	\lesssim \sigma_1^2(\bfA)\sqrt{\frac{(1+\delta^2)d\log N}{|\Omega_t|}}\|\bfW-\bfW^* \|_2,
	\end{equation}
	with probability at least $1-K^2\cdot N^{-10}$.
\end{lemma}
\begin{lemma}\label{Thm: initialization}
	Assume the number of samples $|\Omega_t| \gtrsim \kappa^3(1+\delta^2)\sigma_1^4(\bfA)Kd\log^4 N$, the tensor initialization method via Subroutine 1 outputs $\bfW^{(0)}$ such that
	\begin{equation}\label{eqn: ini}
	\|\bfW^{(0)} -\bfW^* \|_2\lesssim \kappa^6 \sigma_1^2(\bfA)\sqrt{\frac{K^4(1+\delta^2)d\log N}{|\Omega_t|}}\|\bfW^*\|_2
	\end{equation}
	with probability at least $1-N^{-10}$.
\end{lemma}
The proofs of Lemmas \ref{Lemma: local_convexity} and \ref{Lemma: sampling_approximation_error} are included in Appendix \ref{sec: Local_convexity} and \ref{sec: Lemma of sample approximation}, respectively,  while the proof of Lemma \ref{Thm: initialization} can be found in Appendix \ref{Sec: Proof of initialization}. With these three preliminary lemmas on hand, the proof of Theorem \ref{Thm: major_thm_lr} is formally summarized in the following contents. 
\begin{proof}[Proof of Theorem \ref{Thm: major_thm_lr}]
	The update rule of $\bfW^{(t)}$ is
	\begin{equation}
	\begin{split}
	\bfW^{(t+1)}
	=&\bfW^{(t)}-\eta\nabla \hat{f}_{\Omega_t}(\bfW^{(t)})+\beta(\W[t]-\W[t-1])\\
	=&\bfW^{(t)}-\eta\nabla {f}_{\Omega_t}(\bfW^{(t)})+\beta(\W[t]-\W[t-1])+\eta(\nabla{f}_{\Omega_t}(\bfW^{(t)})-\nabla\hat{f}_{\Omega_t}(\bfW^{(t)})).
	\end{split}
	\end{equation}
	Since $\nabla^2_{\Omega_t}$ is a smooth function, by the intermediate value theorem, we have
	\begin{equation}
	\begin{split}
	\bfW^{(t+1)}
	=\bfW^{(t)}&- \eta\nabla^2{f}_{\Omega_t}(\widehat{\bfW}^{(t)})(\bfW^{(t)}-\bfW^*)\\
	&+\beta(\W[t]-\W[t-1])\\
	&+\eta\big(\nabla{f}_{\Omega_t}(\bfW^{(t)})-\nabla\hat{f}_{\Omega_t}(\bfW^{(t)})\big),
	\end{split}
	\end{equation}	
	where $\widehat{\bfW}^{(t)}$ lies in the convex hull of $\bfW^{(t)}$ and $\bfW^*$.
	
	Next, we have
	\begin{equation}\label{eqn: major_thm_key_eqn}
	\begin{split}
	\begin{bmatrix}
	\W[t+1]-\bfW^*\\
	\W[t]-\bfW^*
	\end{bmatrix}
	=&\begin{bmatrix}
	\bfI-\eta\nabla^2{f}_{\Omega_t}(\widehat{\bfW}^{(t)})+\beta\bfI &\beta\bfI\\
	\bfI& 0\\
	\end{bmatrix}
	\begin{bmatrix}
	\W[t]-\bfW^*\\
	\W[t-1]-\bfW^*
	\end{bmatrix}
	+\eta
	\begin{bmatrix}
	\nabla{f}_{\Omega_t}(\bfW^{(t)})-\nabla\hat{f}_{\Omega_t}(\bfW^{(t)})\\
	0
	\end{bmatrix}.
	\end{split}
	\end{equation}
	Let $\bfL(\beta)=\begin{bmatrix}
	\bfI-\eta\nabla^2{f}_{\Omega_t}(\widehat{\bfW}^{(t)})+\beta\bfI &\beta\bfI\\
	\bfI& 0\\
	\end{bmatrix}$, so we have
	\begin{equation*}
	\begin{split}
	\left\|\begin{bmatrix}
	\W[t+1]-\bfW^*\\
	\W[t]-\bfW^*
	\end{bmatrix}
	\right\|_2
	=&
	\left\|\bfL(\beta)
	\right\|_2
	\left\|
	\begin{bmatrix}
	\W[t]-\bfW^*\\
	\W[t-1]-\bfW^*
	\end{bmatrix}
	\right\|_2
	+\eta
	\left\|
	\begin{bmatrix}
	\nabla{f}_{\Omega_t}(\bfW^{(t)})-\nabla\hat{f}_{\Omega_t}(\bfW^{(t)})\\
	0
	\end{bmatrix}
	\right\|_2.
	\end{split}
	\end{equation*}
	From Lemma \ref{Lemma: sampling_approximation_error}, we know that 
	\begin{equation}\label{eqn: convergence2}
	\begin{split}
	\eta\left\|\nabla{f}_{\Omega_t}(\bfW^{(t)})-\nabla\hat{f}_{\Omega_t}(\bfW^{(t)})\right\|_2
	\lesssim \eta \sigma_1^2(\bfA)\sqrt{\frac{(1+\delta^2)d\log N}{|\Omega_t|}}\|\bfW-\bfW^* \|_2.
	\end{split}
	\end{equation}
	Then, we have
	\begin{equation}\label{eqn: convergence}
	\begin{split}
	\|\W[t+1]-\bfW^* \|_2 
	\lesssim& \bigg(\| \bfL(\beta) \|_2+\eta\sigma_1^2(\bfA)\sqrt{\frac{(1+\delta^2)d\log N}{|\Omega_t|}}\bigg)\|\W[t]-\bfW^* \|_2\\
	:\eqsim&\nu(\beta)\|\W[t]-\bfW^* \|_2.
	\end{split}
	\end{equation}

	Let $\nabla^2f(\widehat{\bfW}^{(t)})=\bfS\mathbf{\Lambda}\bfS^{T}$ be the eigen-decomposition of $\nabla^2f(\widehat{\bfW}^{(t)})$. Then, we define
	\begin{equation}\label{eqn:Heavy_ball_eigen}
	\begin{split}
	\widetilde{\bfL}(\beta)
	: =&
	\begin{bmatrix}
	\bfS^T&\bf0\\
	\bf0&\bfS^T
	\end{bmatrix}
	\bfL(\beta)
	\begin{bmatrix}
	\bfS&\bf0\\
	\bf0&\bfS
	\end{bmatrix}
	=\begin{bmatrix}
	\bfI-\eta\bf\Lambda+\beta\bfI &\beta\bfI\\
	\bfI& 0\\
	\end{bmatrix}.
	\end{split}
	\end{equation}
	Since 
	$\begin{bmatrix}
	\bfS&\bf0\\
	\bf0&\bfS
	\end{bmatrix}\begin{bmatrix}
	\bfS^T&\bf0\\
	\bf0&\bfS^T
	\end{bmatrix}
	=\begin{bmatrix}
	\bfI&\bf0\\
	\bf0&\bfI
	\end{bmatrix}$, we know $\bfL(\beta)$ and $\widetilde{\bfL}(\beta)$
	share the same eigenvalues. 	Let $\lambda_i$ be the $i$-th eigenvalue of $\nabla^2f_{\Omega_t}(\widehat{\bfW}^{(t)})$, then the corresponding $i$-th eigenvalue of $\bfL(\beta)$, denoted by $\delta_i(\beta)$, satisfies 
	\begin{equation}\label{eqn:Heavy_ball_quaratic}
	\delta_i^2-(1-\eta \lambda_i+\beta)\delta_i+\beta=0.
	\end{equation}
	Then, we have 
	\begin{equation}\label{eqn: Heavy_ball_root}
	\delta_i(\beta)=\frac{(1-\eta \lambda_i+\beta)+\sqrt{(1-\eta \lambda_i+\beta)^2-4\beta}}{2},
	\end{equation}
	and
	\begin{equation}\label{eqn:heavy_ball_beta}
	\begin{split}
	|\delta_i(\beta)|
	=\begin{cases}
	\sqrt{\beta}, \qquad \text{if}\quad  \beta\ge \big(1-\sqrt{\eta\lambda_i}\big)^2,\\
	\frac{1}{2} \left| { (1-\eta \lambda_i+\beta)+\sqrt{(1-\eta \lambda_i+\beta)^2-4\beta}}\right| , \text{otherwise}.
	\end{cases}
	\end{split}
	\end{equation}
	Note that the other root of \eqref{eqn:Heavy_ball_quaratic} is abandoned because the root in \eqref{eqn: Heavy_ball_root} is always no less than the other root with $|1-\eta \lambda_i|<1$.   
	By simple calculations, we have 
	\begin{equation}\label{eqn: Heavy_ball_result}
	\delta_i(0)>\delta_i(\beta),\quad \text{for}\quad  \forall \beta\in\big(0, (1-{\eta \lambda_i})^2\big).
	\end{equation}
	Moreover, $\delta_i$ achieves the minimum $\delta_{i}^*=|1-\sqrt{\eta\lambda_i}|$ when $\beta= \big(1-\sqrt{\eta\lambda_i}\big)^2$.
	
	Let us first assume $\bfW^{(t)}$ satisfies \eqref{eqn: initial_point_lr}, then
	from Lemma \ref{Lemma: local_convexity}, we know that 
	$$0<\frac{(1-\varepsilon_0)\sigma_1^2(\bfA)}{11\kappa^2\gamma K^2}\le\lambda_i\le\frac{4\sigma_1^2(\bfA)}{K}.$$
	Let $\gamma_1=\frac{(1-\varepsilon_0)\sigma_1^2(\bfA)}{11\kappa^2\gamma K^2}$ and $\gamma_2=\frac{4\sigma_1^2(\bfA)}{K}$.
	If we choose $\beta$ such that
	\begin{equation}
	\beta^*=\max \big\{(1-\sqrt{\eta\gamma_1})^2, (1-\sqrt{\eta\gamma_2})^2 \big\},
	\end{equation} then we have $\beta\ge (1-\sqrt{\eta \lambda_i})^2$ and $\delta_i=\max\big\{|1-\sqrt{\eta\gamma_1}|, |1-\sqrt{\eta\gamma_2}|  \big\}$ for any $i$. 
	
	Let $\eta= \frac{1}{{2\gamma_2}}$, then $\beta^*$ equals to  $\Big(1-\sqrt{\frac{\gamma_1}{2\gamma_2}}\Big)^2$. 
	Then, for any $\varepsilon_0\in (0, {1}/{2})$, we have
	\begin{equation}\label{eqn: con1}
	\begin{split}
	\|\bfL(\beta^*) \|_2
	=\max_i \delta_{i}(\beta^*)
	= 1-\sqrt{\frac{\gamma_1}{2\gamma_2}}
	=1-\sqrt{\frac{1-\varepsilon_0}{88\kappa^2 \gamma K}}
	\le 1-\frac{1-(3/4)\cdot\varepsilon_0}{\sqrt{88\kappa^2 \gamma K}}.
	\end{split}
	\end{equation}
	Then, let 
	\begin{equation}\label{eqn: con2}
	\eta \sigma_1^2(\bfA)\sqrt{\frac{(1+\delta^2)d\log N}{|\Omega_t|}}\lesssim \frac{\varepsilon_0}{4\sqrt{88\kappa^2 \gamma K}},
	\end{equation}
	we need $|\Omega_t|\gtrsim\varepsilon_0^{-2}\kappa^2\gamma M(1+\delta^2)\sigma_1^2(\bfA)K^3d \log N$.		
	Combining \eqref{eqn: con1} and \eqref{eqn: con2}, we have 
	\begin{equation}
	\nu(\beta^*)\le 1-\frac{1-\varepsilon_0}{\sqrt{88\kappa^2 \gamma K}}.
	\end{equation}
	Let $\beta = 0$, we have 
	\begin{equation*}
	\begin{gathered}
	\nu(0) \ge \|\bfA(0) \|_2 = 1-\frac{1-\varepsilon_0}{{88\kappa^2 \gamma K}},	\\
	\nu(0) \lesssim \|\bfA(0) \|_2 + \eta \sigma_1^2(\bfA)\sqrt{\frac{(1+\delta^2)d\log N}{|\Omega_t|}} \le 1-\frac{1-2\varepsilon_0}{{88\kappa^2 \gamma K}}
	\end{gathered}
	\end{equation*}
	if $|\Omega_t|\gtrsim\varepsilon_0^{-2}\kappa^2\gamma M(1+\delta^2)\sigma_1^2(\bfA)K^3d \log N$.
	
	Hence, with $\eta =\frac{1}{2 \gamma_2}$ and $\beta = \big(1-\frac{\gamma_1}{2\gamma_2} \big)^2$, we have
	\begin{equation}\label{eqn: induction}
	\begin{split}
	\|	\W[t+1]-\bfW^*\|_2
	\le\Big(1-\frac{1-\varepsilon_0}{\sqrt{88\kappa^2 \gamma K}}\Big)\|	\W[t]-\bfW^*\|_2,
	\end{split}
	\end{equation}
	provided $\W[t]$ satisfies \eqref{eqn: initial_point_lr}, and 
	\begin{equation}\label{eqn: N_3}
	|\Omega_t|\gtrsim\varepsilon_0^{-2}\kappa^2\gamma (1+\delta^2)\sigma_1^4(\bfA)K^3d \log N.
	\end{equation}
	Then, we can start mathematical induction of \eqref{eqn: induction} over $t$.
	
	\textbf{Base case}: According to Lemma \ref{Thm: initialization}, we know that \eqref{eqn: initial_point_lr} holds for $\bfW^{(0)}$ if 
		\begin{equation}\label{eqn: N_1}
		|\Omega_t|\gtrsim \varepsilon_0^{-2}\kappa^{9}\gamma^2 (1+\delta^2)\sigma_1^4(\bfA) K^8d \log N.
		\end{equation}
		According to  Theorem \ref{Thm: major_thm_lr}, it is clear that the number of samples $|\Omega_t|$  satisfies \eqref{eqn: N_1}, then  \eqref{eqn: initial_point_lr} indeed holds for $t=0$. Since \eqref{eqn: initial_point_lr} holds for $t=0$ and $|\Omega_t|$ in Theorem \ref{Thm: major_thm_lr} satisfies \eqref{eqn: N_3} as well, we have \eqref{eqn: induction} holds for $t=0$.   
	
	\textbf{Induction step}:
		Assuming \eqref{eqn: induction} holds for $\W[t]$, we need to show that \eqref{eqn: induction} holds for $\W[t+1]$. That is to say, we need $|\Omega_t|$ satisfies \eqref{eqn: N_3}, which holds naturally from Theorem \ref{Thm: major_thm_lr}. 

		Therefore, when $|\Omega_t|\gtrsim \varepsilon_0^{-2}\kappa^{9}\gamma^2 (1+\delta^2)\sigma_1^4(\bfA)K^8 d \log N$, we know that \eqref{eqn: induction} holds for all $0\le t \le T-1$ with probability at least $1-K^2T\cdot N^{-10}$. By simple calculations, we can obtain
		\begin{equation}\label{eqn: Thm1_final}
		\begin{split}
		\|	\W[T]-\bfW^*\|_2
		\le&\Big(1-\frac{1-\varepsilon_0}{\sqrt{88\kappa^2 \gamma K}}\Big)^T\|	\W[0]-\bfW^*\|_2
		\end{split}
		\end{equation}
\end{proof}

\subsection{Proof of Lemma \ref{Lemma: local_convexity}}
\label{sec: Local_convexity}
In this section, we provide the proof of Lemma \ref{Lemma: local_convexity} which shows the local convexity of $f_{\Omega_t}$ in a small neighborhood of $\bfW^*$. The roadmap is to first bound the smallest eigenvalue of $\nabla^2 f_{\Omega_t}$ in the ground truth as shown in Lemma \ref{Lemma: lemma 5}, then show that the difference of $\nabla^2 f_{\Omega_t}$ between any fixed point $\bfW$ in this region and the ground  truth $\bfW^*$ is bounded in terms of $\|\bfW -\bfW^* \|_2$ by Lemma \ref{Lemma: lemma 6}.
\begin{lemma}\label{Lemma: lemma 5}
	The second-order derivative of $f_{\Omega_t}$ at the ground truth $\bfW^*$ satisfies
	\begin{equation}
	\frac{\sigma_1^2(\bfA)}{11 \kappa^2 \gamma K^2}   \preceq    \nabla^2 f_{\Omega_t}(\bfW^*)    \preceq \frac{3\sigma_1^2(\bfA)}{K}.
	\end{equation}
\end{lemma}

\begin{lemma}\label{Lemma: lemma 6}
	Suppose $\bfW$ satisfies \eqref{eqn: initial_point_lr}, we have
	\begin{equation}
	\left\|\nabla^2f_{\Omega_t}(\bfW) - \nabla^2f_{\Omega_t}(\bfW^*)  \right\|_2 \le  4\sigma_1^2(\bfA)\frac{\|\bfW^*-\bfW\|_2}{\sigma_K}.
	\end{equation}
\end{lemma}
The proofs of Lemmas \ref{Lemma: lemma 5} and \ref{Lemma: lemma 6} can be found in Sec. \ref{sec: sub_lr}. With these two preliminary lemmas on hand, the proof of Lemma \ref{Lemma: local_convexity} is formally summarized in the following contents. 
\begin{proof}[Proof of Lemma \ref{Lemma: local_convexity}]
	By the triangle inequality, we have 
	\begin{equation*}
	\Big| 
	\left\| \nabla^2 f_{\Omega_t}(\bfW)\right\|_2 -\left\| \nabla^2f_{\Omega_t}(\bfW^*)\right\|_2 
	\Big| \le \|\nabla^2f_{\Omega_t}(\bfW^*)-\nabla^2f_{\Omega_t}(\bfW)\|_2,
	\end{equation*}	
	and
	\begin{equation*}
	\begin{split}
	&\left\| \nabla^2f_{\Omega_t}(\bfW)\right\|_2 \le \left\| \nabla^2f_{\Omega_t}(\bfW^*)\right\|_2 +\|\nabla^2f_{\Omega_t}(\bfW^*)-\nabla^2f_{\Omega_t}(\bfW)\|_2,\\
	&\left\| \nabla^2f_{\Omega_t}(\bfW)\right\|_2 \ge \left\| \nabla^2f_{\Omega_t}(\bfW^*)\right\|_2 -\|\nabla^2f_{\Omega_t}(\bfW^*)-\nabla^2f_{\Omega_t}(\bfW)\|_2.
	\end{split}
	\end{equation*}	
	The error bound of $\|\nabla^2f_{\Omega_t}(\bfW^*)-\nabla^2f_{\Omega_t}(\bfW)\|_2$ can be derived from Lemma \ref{Lemma: lemma 6}, and the error bound of $\nabla^2f_{\Omega_t}(\bfW^*)$ is provided in Lemma \ref{Lemma: lemma 5}.
	
	Therefore, for any $\bfW$ satisfies \eqref{eqn: initial_point_lr}, we have 
	\begin{equation}\label{temp_L12}
	\frac{(1-\varepsilon_0)\sigma_1^2(\bfA)}{11\kappa^2\gamma K^2}\le\left\| \nabla^2f_{\Omega_t}(\bfW)\right\|_2 \le \frac{4\sigma_1^2(\bfA)}{K}.
	\end{equation}
\end{proof}

\subsection{Proof of Lemma \ref{Lemma: sampling_approximation_error}}\label{sec: Lemma of sample approximation}
The proof of Lemma \ref{Lemma: sampling_approximation_error} is mainly to bound the concentration error of random variables $\bfz_n(j,k)$ as shown in \eqref{eqn: sub-var}. We first show that $\bfz_n(j,k)$ is a  sub-exponential random variable, and the definitions of sub-Gaussian and sub-exponential random variables are provided in Definitions \ref{Def: sub-gaussian} and \ref{Def: sub-exponential}. Though Hoeffding's inequality provides the concentration error for sum of independent random variables,  random variables $\bfz_n(j,k)$ with different $j,k$ are not independent. Hence, we introduce Lemma \ref{Lemma: partial_independent_var} to provide the upper bound for the moment generation function of the sum of partly dependent random variables and then apply standard Chernoff inequality. Lemmas \ref{Lemma: covering_set} and \ref{Lemma: spectral_norm on net} are standard tools in analyzing spectral norms of high-demensional random matrices. 
\begin{definition}[Definition 5.7, \cite{V2010}]\label{Def: sub-gaussian}
	A random variable $X$ is called a sub-Gaussian random variable if it satisfies 
	\begin{equation}
	(\mathbb{E}|X|^{p})^{1/p}\le c_1 \sqrt{p}
	\end{equation} for all $p\ge 1$ and some constant $c_1>0$. In addition, we have 
	\begin{equation}
	\mathbb{E}e^{s(X-\mathbb{E}X)}\le e^{c_2\|X \|_{\psi_2}^2 s^2}
	\end{equation} 
	for all $s\in \mathbb{R}$ and some constant $c_2>0$, where $\|X \|_{\psi_2}$ is the sub-Gaussian norm of $X$ defined as $\|X \|_{\psi_2}=\sup_{p\ge 1}p^{-1/2}(\mathbb{E}|X|^{p})^{1/p}$.
	
	Moreover, a random vector $\bfX\in \mathbb{R}^d$ belongs to the sub-Gaussian distribution if one-dimensional marginal $\boldsymbol{\alpha}^T\bfX$ is sub-Gaussian for any $\boldsymbol{\alpha}\in \mathbb{R}^d$, and the sub-Gaussian norm of $\bfX$ is defined as $\|\bfX \|_{\psi_2}= \sup_{\|\boldsymbol{\alpha} \|_2=1}\|\boldsymbol{\alpha}^T\bfX \|_{\psi_2}$.
\end{definition}

\begin{definition}[Definition 5.13, \cite{V2010}]\label{Def: sub-exponential}
	
	A random variable $X$ is called a sub-exponential random variable if it satisfies 
	\begin{equation}
	(\mathbb{E}|X|^{p})^{1/p}\le c_3 {p}
	\end{equation} for all $p\ge 1$ and some constant $c_3>0$. In addition, we have 
	\begin{equation}
	\mathbb{E}e^{s(X-\mathbb{E}X)}\le e^{c_4\|X \|_{\psi_1}^2 s^2}
	\end{equation} 
	for $s\le 1/\|X \|_{\psi_1}$ and some constant $c_4>0$, where $\|X \|_{\psi_1}$ is the sub-exponential norm of $X$ defined as $\|X \|_{\psi_1}=\sup_{p\ge 1}p^{-1}(\mathbb{E}|X|^{p})^{1/p}$.
	
\end{definition}
\begin{lemma}\label{Lemma: partial_independent_var}
	Given a sampling set $\mathcal{X}=\{ x_n \}_{n=1}^N$ that contains $N$ partly dependent random variables, for each $n\in [N]$, suppose $x_n$ is dependent with at most $d_{\mathcal{X}}$ random variables  in $\mathcal{X}$ (including $x_n$ itself), and the moment generate function of $x_n$ satisfies $\mathbb{E}_{x_n} e^{sx_n}\le e^{Cs^2} $ for some constant $C$ that may depend on $x_n$. Then, the moment generation function of $\sum_{n=1}^N x_n$ is bounded as 
	\begin{equation}
	\mathbb{E}_{\mathcal{X}} e^{s\sum_{n=1}^N x_n} \le e^{Cd_{\mathcal{X}}Ns^2}.
	\end{equation}  
\end{lemma}

\begin{lemma}[Lemma 5.2, \cite{V2010}]\label{Lemma: covering_set}
	Let $\mathcal{B}(0, 1)\in\{ \boldsymbol{\alpha} \big| \|\boldsymbol{\alpha} \|_2=1, \boldsymbol{\alpha}\in \mathbb{R}^d  \}$ denote a unit ball in $\mathbb{R}^{d}$. Then, a subset $\mathcal{S}_\xi$ is called a $\xi$-net of $\mathcal{B}(0, 1)$ if every point $\bfz\in \mathcal{B}(0, 1)$ can be approximated to within $\xi$ by some point $\boldsymbol{\alpha}\in \mathcal{B}(0, 1)$, i.e. $\|\bfz -\boldsymbol{\alpha} \|_2\le \xi$. Then the minimal cardinality of a   $\xi$-net $\mathcal{S}_\xi$ satisfies
	\begin{equation}
	|\mathcal{S}_{\xi}|\le (1+2/\xi)^d.
	\end{equation}
\end{lemma}
\begin{lemma}[Lemma 5.3, \cite{V2010}]\label{Lemma: spectral_norm on net}
	Let $\bfA$ be an $N\times d$ matrix, and let $\mathcal{S}_{\xi}$ be a $\xi$-net of $\mathcal{B}(0, 1)$ in $\mathbb{R}^d$ for some $\xi\in (0, 1)$. Then
	\begin{equation}
	\|\bfA\|_2 \le (1-\xi)^{-1}\max_{\boldsymbol{\alpha}\in \mathcal{S}_{\xi}} |\boldsymbol{\alpha}^T\bfA\boldsymbol{\alpha}|.
	\end{equation} 
\end{lemma}

The proof of Lemma \ref{Lemma: partial_independent_var} can be found in Appendex \ref{sec: sub_lr}. With these preliminary Lemmas and definition on hand, the proof of Lemma \ref{Lemma: sampling_approximation_error} is formally summarized in the following contents.
\begin{proof}[Proof of Lemma \ref{Lemma: sampling_approximation_error} ]
	We have 
\begin{equation}
\begin{split}
\hat{f}_{\Omega_t}(\bfW)
=& \frac{1}{2{|\Omega_t|}} \sum_{n \in \Omega_t}\Big| y_n - g(\bfW;\bfa_n^T\bfX) \Big|^2\\
=& \frac{1}{2{|\Omega_t|}} \sum_{n \in \Omega_t}\Big| y_n - \sum_{j=1}^K \phi(\bfa_n^T \bfX\bfw_j) \Big|^2,
\end{split}
\end{equation}
and 
\begin{equation}
f_{\Omega_t}(\bfW)=\mathbb{E}_{\bfX}\hat{f}_{\Omega_t}(\bfW) = \frac{1}{2{|\Omega_t|}} \sum_{n \in \Omega_t}\mathbb{E}_{\bfx}\Big| y_n - \sum_{j=1}^K \phi(\bfa_n^T \bfX\bfw_j) \Big|^2.
\end{equation}
The gradients of $\hat{f}_{\Omega_t}$ are
\begin{equation}
\begin{split}
\frac{\partial\hat{f}_{\Omega_t}}{\partial \bfw_k}(\bfW)
=&\frac{1}{K^2|\Omega_t|} \sum_{n \in \Omega_t} \Big(y_n - \sum_{j=1}^{K}\phi(\bfa_n^T\bfX\bfw_j)  \Big)\bfX^T\bfa_n\phi^{\prime}(\bfa_n^T\bfX\bfw_k)\\
=&\frac{1}{K^2|\Omega_t|} \sum_{n \in \Omega_t} \Big( \sum_{j=1}^{K}\phi(\bfa_n^T\bfX\bfw_j^*) - \sum_{j=1}^{K}\phi(\bfa_n^T\bfX\bfw_j)  \Big)\bfX^T\bfa_n\phi^{\prime}(\bfa_n^T\bfX\bfw_k)\\
=&\sum_{j=1}^{K}\frac{1}{K^2|\Omega_t|} \sum_{n \in \Omega_t} 
\big( \phi(\bfa_n^T\bfX\bfw_j^*) - \phi(\bfa_n^T\bfX\bfw_j)  \big)\bfX^T\bfa_n\phi^{\prime}(\bfa_n^T\bfX\bfw_k).\\
\end{split}
\end{equation}
Let us define
\begin{equation}\label{eqn: sub-var}
\bfz_{n}(k, j) = \bfX^T\bfa_n\phi^{\prime}(\bfa_n^T\bfX\bfw_k)\big( \phi(\bfa_n^T\bfX\bfw_j^*) - \phi(\bfa_n^T\bfX\bfw_j)  \big), 
\end{equation}
 then for any normalized $\boldsymbol{\alpha}\in \mathbb{R}^d$, we have 
\begin{equation}
\begin{split}
&p^{-1}\Big(\mathbb{E}_{\bfX} \big|\boldsymbol{\alpha}^T\bfX^T\bfa_n\phi^{\prime}(\bfa_n^T\bfX\bfw_k)\big( \phi(\bfa_n^T\bfX\bfw_j^*) - \phi(\bfa_n^T\bfX\bfw_j)  \big)\big|^p  \Big)^{1/p}\\
\le&p^{-1}\Big(\mathbb{E}_{\bfX} \big|\boldsymbol{\alpha}^T\bfX^T\bfa_n\big|^{2p}\cdot \mathbb{E}_{\bfX}\big| \phi^{\prime}(\bfa_n^T\bfX\bfw_k)\big( \phi(\bfa_n^T\bfX\bfw_j^*) - \phi(\bfa_n^T\bfX\bfw_j)  \big)\big|^{2p}  \Big)^{1/2p}\\
\le & p^{-1}\Big(\mathbb{E}_{\bfX} \big|\boldsymbol{\alpha}^T\bfX^T\bfa_n\big|^{2p} \Big)^{1/2p}
\cdot \Big(\mathbb{E}_{\bfX}\big| \bfa_n^T\bfX(\bfw_j^*-\bfw_j) \big|^{2p}  \Big)^{1/2p}\\
\end{split}
\end{equation} 
where the first inequality comes from the Cauchy-Schwarz inequality. Furthermore, $\bfa_n^T\bfX$ belongs to the Gaussian distribution and thus is a sub-Gaussian random vector as well. Then, from Definition \ref{Def: sub-gaussian}, we have 
\begin{equation}
\begin{gathered}
    \Big(\mathbb{E}_{\bfX} \big|\boldsymbol{\alpha}^T\bfX^T\bfa_n\big|^{2p} \Big)^{1/2p}  \le (2p)^{1/2}\|\bfX^T\bfa_n\|_{\psi_2} 
    \le (2p)^{1/2}\|\bfa_n \|_2,\\
    \text{and}\quad 
    \Big(\mathbb{E}_{\bfX}\big| \bfa_n^T\bfX(\bfw_j^*-\bfw_j) \big|^{2p}  \Big)^{1/2p}
    \le (2p)^{1/2}\|\bfa_n \|_2\cdot \|\bfw_j^*-\bfw_j \|_2.
    \end{gathered}
\end{equation}
Then, we have 
\begin{equation}
\begin{split}
&p^{-1}\Big(\mathbb{E}_{\bfX} \big|\boldsymbol{\alpha}^T\bfX^T\bfa_n\phi^{\prime}(\bfa_n^T\bfX\bfw_k)\big( \phi(\bfa_n^T\bfX\bfw_j^*) - \phi(\bfa_n^T\bfX\bfw_j)  \big)\big|^p  \Big)^{1/p}\\
    \le& p^{-1}\cdot 2p\|\bfa_n\|_2^2\cdot \|\bfw_j^*-\bfw_j \|_2\\
\le& 2\sigma_1^2(\bfA)\cdot \|\bfw_j^*-\bfw_j \|_2.
\end{split}
\end{equation}

Therefore, from Definition \ref{Def: sub-exponential}, $\bfz_n(k,j)$ belongs to the sub-exponential distribution with
\begin{equation}
    \|\bfz_n \|_{\phi_1} \le 2\sigma^2_1(\bfA)\cdot \|\bfw_j^*-\bfw_j\|_2.
\end{equation}

Recall that each node is connected with at most $\delta$ other nodes. Hence, for any fixed $\bfz_n$, there are at most $(1+\delta^2)$ (including $\bfz_n$ itself) elements in $\big\{ \bfz_l\big| l\in\Omega_t \big\}$ are dependant with $\bfz_n$. From Lemma \ref{Lemma: partial_independent_var},   the moment generation function of $\sum_{n\in\Omega_t} (\bfz_n-\mathbb{E}_{\bfX}\bfz_n)$ satisfies 
\begin{equation}
    \mathbb{E}_{\bfX} e^{s\sum_{n\in\Omega_t} (\bfz_n-\mathbb{E}_{\bfX}\bfz_n)}\le e^{C(1+\delta^2)|\Omega_t|s^2}.
\end{equation}

By Chernoff inequality, we have 
\begin{equation}
    \text{Prob}
    \bigg\{  \Big\|\frac{1}{|\Omega_t|} \sum_{n \in \Omega_t} \big(\bfz_n(k,j) - \mathbb{E}_{\bfX}\bfz_n(k,j)\big) \Big\|_2>t   
    \bigg
    \}\le \frac{e^{C(1+\delta^2)|\Omega_t|s^2}}{e^{|\Omega_t|ts}}
\end{equation}
for any $s>0$. 

Let $s = t/\big(C(1+\delta^2)\|\bfz_n \|_{\phi_1}^2\big)$ and $t = \sqrt{ \frac{(1+\delta^2)d\log N}{|\Omega_t|}} \|\bfz_n \|_{\phi_1}$, we have 
\begin{equation}
\begin{split}
\Big\|\frac{1}{|\Omega_t|} \sum_{n \in \Omega_t} \big(\bfz_n(k,j) - \mathbb{E}_{\bfX}\bfz_n(k,j)\big) \Big\|_2
\le& C\sqrt{\frac{(1+\delta^2)d\log N}{|\Omega_t|}}\sigma_1^2(\bfA)\cdot \|\bfw_{j}^*-\bfw_{j}\|_2\\
\le&C\sigma_1^2(\bfA)\sqrt{\frac{(1+\delta^2)d\log N}{|\Omega_t|}}\cdot \|\bfW^*-\bfW\|_2
\end{split}
\end{equation}
with probability at least $1- N^{-d}$.

In conclusion, by selecting $\xi=\frac{1}{2}$ in Lemmas \ref{Lemma: covering_set} and \ref{Lemma: spectral_norm on net}, we have
\begin{equation}
\begin{split}
\Big\|\frac{\partial\hat{f}_{\Omega_t}}{\partial \bfw_k}(\bfW) - \frac{\partial{f}_{\Omega_t}}{\partial \bfw_k}(\bfW)\Big\|_2
\le& \sum_{k=1}^{K}\sum_{j=1}^{K}\frac{1}{K^2}\Big\|\frac{1}{|\Omega_t|} \sum_{n \in \Omega_t}\bfz_n(k,j) - \mathbb{E}_{\bfX}\bfz_n(k,j)\Big\|_2\\				
\le& C\sigma_1^2(\bfA)\sqrt{\frac{(1+\delta^2)d\log N}{|\Omega_t|}}\cdot \|\bfW^*-\bfW \|_2
\end{split}
\end{equation}
with probability at least $1-\big(\frac{5}{N} \big)^d$.
\end{proof}

\subsection{Proof of auxiliary lemmas for regression problems}\label{sec: sub_lr}
\subsubsection{Proof of Lemma \ref{Lemma: lemma 5}}\label{Sec: lemma5}
\begin{proof}[Proof of Lemma \ref{Lemma: lemma 5}  ] 
	For any normalized $\boldsymbol{\alpha}\in \mathbb{R}^{Kd}$, the lower bound of $\nabla^2f_{\Omega_t}(\bfW^*)$ is derived from 
	\begin{equation}
	\begin{split}
	\boldsymbol{\alpha}^T {\nabla^2 f(\bfW^*)} \boldsymbol{\alpha}
	=&\frac{1}{K^2|\Omega_t|} \sum_{n \in \Omega_t} \mathbb{E}_{\bfX}\bigg[ \Big( \sum_{j=1}^{K} \boldsymbol{\alpha}_j^T\bfX^T \bfa_n \phi^{\prime}(\bfa_n^T\bfX
	\bfw_j^*)  \Big)^2  \bigg]\\
	\ge & \frac{1}{K^2|\Omega_t|} \sum_{n \in \Omega_t}\frac{\|\bfa_n\|_2^2}{11 \kappa^2 \gamma}\|\boldsymbol{\alpha}\|_2^2
	=  \frac{\sigma_1^2(\bfA)}{11 \kappa^2 \gamma K^2},
	\end{split}
	\end{equation}
	where the last inequality comes from Lemma D.6 in \cite{ZSJB17_v2}.
	
	Next,  the upper bound of $\nabla^2f_{\Omega_t}(\bfW^*)$ is derived from
	\begin{equation}
	\begin{split}
	&\boldsymbol{\alpha}^T {\nabla^2 f(\bfW^*)} \boldsymbol{\alpha}\\
	=&\frac{1}{K^2|\Omega_t|} \sum_{n \in \Omega_t} \mathbb{E}_{\bfX}\bigg[ \Big( \sum_{j=1}^{K} \boldsymbol{\alpha}_j^T\bfX^T \bfa_n \phi^{\prime}(\bfa_n^T\bfX\bfw_j^*)  \Big)^2  \bigg]\\
	=&\frac{1}{K^2|\Omega_t|} \sum_{n \in \Omega_t} \sum_{j_1=1}^{K} \sum_{j_2=1}^{K} \mathbb{E}_{\bfX}\bigg[  \boldsymbol{\alpha}_{j_1}^T\bfX^T\bfa_n \phi^{\prime}(\bfa_n^T\bfX\bfw_{j_1}^*) 
	\boldsymbol{\alpha}_{j_2}^T\bfX^T\bfa_n \phi^{\prime}(\bfa_n^T\bfX\bfw_{j_2}^*)  \bigg]\\
	\le&\frac{1}{K^2|\Omega_t|} \sum_{n \in \Omega_t} \sum_{j_1=1}^{K} \sum_{j_2=1}^{K} \Big[\mathbb{E}_{\bfX}|\boldsymbol{\alpha}_{j_1}^T\bfX^T\bfa_n|^4 \cdot\mathbb{E}_{\bfX}|\phi^{\prime}(\bfa_n^T\bfX\bfw_{j_1}^*)|^4 
	\cdot\mathbb{E}_{\bfX}|\boldsymbol{\alpha}_{j_2}^T\bfX^T\bfa_n|^4
	\cdot\mathbb{E}_{\bfX}|\phi^{\prime}(\bfa_n^T\bfX\bfw_{j_2}^*)|^4\Big]^{\frac{1}{4}}\\
	\le& \frac{1}{K^2|\Omega_t|} \sum_{n \in \Omega_t} \sum_{j_1=1}^{K} \sum_{j_2=1}^{K} 3\sigma_1^2(\bfA)\|\boldsymbol{\alpha}_{j_1} \|_2\|\boldsymbol{\alpha}_{j_2} \|_2\\
	\le& 3\sigma_1^2(\bfA) \frac{\|\boldsymbol{\alpha}\|_2^2}{K},
	\end{split}
	\end{equation}
	which completes the proof.
\end{proof}
\subsubsection{Proof of Lemma \ref{Lemma: lemma 6}}\label{Sec: lemma6}
\begin{proof}[Proof of Lemma \ref{Lemma: lemma 6} ]
	The second-order derivative of $f_{\Omega_t}$ is written as
	\begin{equation}
	\begin{split}
	&\frac{\partial^2 f_{\Omega_t}}{\partial \bfw_{j_1}\partial \bfw_{j_2}}(\bfW) - \frac{\partial^2 f_{\Omega_t}}{\partial \bfw_{j_1}\partial \bfw_{j_2}}(\bfW^*)\\
	=&\frac{1}{K^2|\Omega_t|} \sum_{n \in \Omega_t} \mathbb{E}_{\bfX}(\bfX^T \bfa_n)(\bfX^T\bfa_n)^T \Big[\phi^{\prime}(\bfa_n^T\bfX\bfw_{j_1})\phi^{\prime}(\bfa_n^T\bfX\bfw_{j_2})
	-\phi^{\prime}(\bfa_n^T\bfX\bfw^*_{j_1})\phi^{\prime}(\bfa_n^T\bfX\bfw^*_{j_2})\Big] \\
	=&\frac{1}{K^2|\Omega_t|} \sum_{n \in \Omega_t}  \mathbb{E}_{\bfX}(\bfX^T \bfa_n)(\bfX^T\bfa_n)^T \big(\phi^{\prime}(\bfa_n^T\bfX\bfw_{j_1})-\phi^{\prime}(\bfa_n^T\bfX\bfw_{j_1}^*)\big)\phi^{\prime}(\bfa_n^T\bfX\bfw_{j_2})\\
	&-\frac{1}{K^2|\Omega_t|} \sum_{n \in \Omega_t} \mathbb{E}_{\bfX}(\bfX^T \bfa_n)(\bfX^T\bfa_n)^T\phi^{\prime}(\bfa_n^T\bfX\bfw^*_{j_1})\big(\phi^{\prime}(\bfa_n^T\bfX\bfw^*_{j_2})-\phi^{\prime}(\bfa_n^T\bfX\bfw_{j_2})\big).
	\end{split}
	\end{equation}
	For any normalized $\boldsymbol{\alpha}\in \mathbb{R}^d$, we have 
	\begin{equation}
	\begin{split}
	&\Big|\boldsymbol{\alpha}^T\Big[\frac{\partial^2 f_{\Omega_t}}{\partial \bfw_{j_1}\partial \bfw_{j_2}}(\bfW) - \frac{\partial^2 f_{\Omega_t}}{\partial \bfw_{j_1}\partial \bfw_{j_2}}(\bfW^*)\Big]\boldsymbol{\alpha}\Big|\\
	\le&\Big|\frac{1}{K^2|\Omega_t|} \sum_{n \in \Omega_t}  \mathbb{E}_{\bfX}(\boldsymbol{\alpha}^T\bfX^T\bfa_n)^2 \big(\phi^{\prime}(\bfa_n^T\bfX\bfw_{j_1})-\phi^{\prime}(\bfa_n^T\bfX\bfw_{j_1}^*)\big)\phi^{\prime}(\bfa_n^T\bfX\bfw_{j_2})\Big|\\
	&+\Big|\frac{1}{K^2|\Omega_t|} \sum_{n \in \Omega_t} \mathbb{E}_{\bfX}(\boldsymbol{\alpha}^T\bfX^T\bfa_n)^2\phi^{\prime}(\bfa_n^T\bfX\bfw^*_{j_1})\big(\phi^{\prime}(\bfa_n^T\bfX\bfw^*_{j_2})-\phi^{\prime}(\bfa_n^T\bfX\bfw_{j_2})\big)\Big| \\
	\le&\frac{1}{K^2|\Omega_t|} \sum_{n \in \Omega_t}  \mathbb{E}_{\bfX}\big|\boldsymbol{\alpha}^T\bfX^T\bfa_n\big|^2\cdot \Big|\phi^{\prime}(\bfa_n^T\bfX\bfw_{j_1})-\phi^{\prime}(\bfa_n^T\bfX\bfw_{j_1}^*)\Big|\\
	&+\frac{1}{K^2|\Omega_t|} \sum_{n \in \Omega_t} \mathbb{E}_{\bfX}\big|\boldsymbol{\alpha}^T\bfX^T\bfa_n\big|^2\cdot \Big|\phi^{\prime}(\bfa_n^T\bfX\bfw^*_{j_2})-\phi^{\prime}(\bfa_n^T\bfX\bfw_{j_2})\Big|.
	\end{split}
	\end{equation}
	
	{ It is easy to verify there exists a basis such that $\mathcal{B}=\{\boldsymbol{\alpha},\boldsymbol{\beta}, \boldsymbol{\gamma}, \boldsymbol{\alpha}_4^{\perp}, \cdots, \boldsymbol{\alpha}_d^{\perp}\}$ with $\{\boldsymbol{\alpha}, \boldsymbol{\beta}, \boldsymbol{\gamma} \}$ spanning a subspace that contains $\boldsymbol{\alpha}, \bfw_{j_1}$ and $\bfw_{j_1}^*$. Then, for any $\bfX^T\bfa_n\in \mathbb{R}^d$, we have a unique $\bfz=\begin{bmatrix}
		z_1&z_2&\cdots & z_d
		\end{bmatrix}^T$ such that 
		\begin{equation*}
		\bfX^T\bfa_n = z_1\boldsymbol{\alpha} + z_2 \boldsymbol{\beta} +z_3\boldsymbol{\gamma} + \cdots + z_d\boldsymbol{\alpha}^{\perp}_d.
		\end{equation*}
		Also, since $\bfX^T\bfa_n \sim \mathcal{N}(\boldsymbol{0}, \|\bfa_n \|_2^2\bfI_d)$, we have $\bfz \sim \mathcal{N}(\boldsymbol{0}, \|\bfa_n \|_2^2\bfI_d)$. Then, we have 
		\begin{equation*}
		\begin{split}
		& \mathbb{E}_{\bfX}\big|\boldsymbol{\alpha}^T\bfX^T \bfa_n\big|^2\cdot \Big|\phi^{\prime}(\bfa_n^T\bfX\bfw_{j_1})-\phi^{\prime}(\bfa_n^T\bfX\bfw_{j_1}^*)\Big| \\
		=& \mathbb{E}_{z_1, z_2, z_3}|\phi^{\prime}\big(\bfw_{j_1}^T\widetilde{\bfx}\big)-\phi^{\prime}\big({\bfw_{j_1}^*}^T\widetilde{\bfx}\big)|\cdot 
		|\bfa^T\widetilde{\bfx}|^2\\
		=&\int|\phi^{\prime}\big(\bfw_{j_1}^T\widetilde{\bfx}\big)-\phi^{\prime}\big({\bfw_{j_1}^*}^T\widetilde{\bfx}\big)|\cdot 
		|\bfa^T\widetilde{\bfx}|^2\cdot f_Z(z_1, z_2,z_3)dz_1 dz_2 dz_3,
		\end{split}
		\end{equation*}
		where $\widetilde{\bfx}= z_1\boldsymbol{\alpha} + z_2\boldsymbol{\beta} + z_3\boldsymbol{\gamma}$ and $f_Z(z_1, z_2,z_3)$ is the probability density function of $(z_1, z_2, z_3)$. Next, we consider spherical coordinates with $z_1= r\cos\phi_1 , z_2 = r\sin\phi_1\sin\phi_2, z_3=z_2 = r\sin\phi_1\cos\phi_2$. Hence,
		\begin{equation}
		\begin{split}
		& \mathbb{E}_{\bfX}\big|\boldsymbol{\alpha}^T\bfX^T\bfa_n\big|^2\cdot \Big|\phi^{\prime}(\bfa_n^T\bfX\bfw_{j_1})-\phi^{\prime}(\bfa_n^T\bfX\bfw_{j_1}^*)\Big| \\
		=&\int\int\int|\phi^{\prime}\big(\bfw_{j_1}^T\widetilde{\bfx}\big)-\phi^{\prime}\big({\bfw_{j_1}^*}^T\widetilde{\bfx}\big)|\cdot |r\cos\phi_1|^2
		\cdot f_Z(r, \phi_1, \phi_2)r^2\sin \phi_1 drd\phi_1d\phi_2.
		\end{split}
		\end{equation}
		It is easy to verify that $\phi^{\prime}\big(\bfw_{j_1}^T\widetilde{\bfx}\big)$ only depends on the direction of $\widetilde{\bfx}$ and 
		\begin{equation*}
		f_Z(r, \phi_1, \phi_2) = \frac{1}{(2\pi\|\bfa_n \|_2^2)^{\frac{3}{2}}}e^{-\frac{x_1^2+x_2^2+x_3^2}{2\|\bfa_n \|_2^2}}=\frac{1}{(2\pi\|\bfa_n \|_2^2)^{\frac{3}{2}}}e^{-\frac{r^2}{2\|\bfa_n \|_2^2}}
		\end{equation*}
		only depends on $r$. Then, we have 
		\begin{equation}\label{eqn: ean111}
		\begin{split}
		&\mathbb{E}_{\bfX}\big|\boldsymbol{\alpha}^T\bfX^T\bfa_n\big|^2\cdot \Big|\phi^{\prime}(\bfa_n^T\bfX\bfw_{j_1})-\phi^{\prime}(\bfa_n^T\bfX\bfw_{j_1}^*)\Big| \\
		=&\int \int\int|\phi^{\prime}\big(\bfw_{j_1}^T(\widetilde{\bfx}/r)\big)-\phi^{\prime}\big({\bfw_{j_1}^*}^T(\widetilde{\bfx}/r)\big)|\cdot |r\cos\phi_1|^2
		\cdot f_Z(r)r^2\sin \phi_1 drd\phi_1d\phi_2\\
		=& \int_0^{\infty} r^4f_z(r)dr\int_{0}^{\pi}\int_{0}^{2\pi}|\cos \phi_1|^2\cdot \sin\phi_1\cdot|\phi^{\prime}\big(\bfw_{j_2}^T(\widetilde{\bfx}/r)\big)-\phi^{\prime}\big({\bfw_{j_2}^*}^T(\widetilde{\bfx}/r)\big)|d\phi_1d\phi_2\\
		\le& 3\|\bfa_n\|^2_2 \cdot \int_0^{\infty} r^2f_z(r)dr\int_{0}^{\pi}\int_{0}^{2\pi} \sin\phi_1\cdot|\phi^{\prime}\big(\bfw_{j_2}^T(\widetilde{\bfx}/r)\big)-\phi^{\prime}\big({\bfw_{j_2}^*}^T(\widetilde{\bfx}/r)\big)|d\phi_1d\phi_2\\
		=&3\|\bfa_n\|^2_2\cdot\mathbb{E}_{z_1, z_2, z_3}\big|\phi^{\prime}\big(\bfw_{j_1}^T\widetilde{\bfx}\big)-\phi^{\prime}\big({\bfw_{j_1}^*}^T\widetilde{\bfx}\big)|\\
		=&3\|\bfa_n\|^2_2\cdot\mathbb{E}_{\bfX}\big|\phi^{\prime}\big(\bfa_n^T\bfX\bfw_{j_1}\big)-\phi^{\prime}\big(\bfa_n^T\bfX{\bfw_{j_1}^*}\big)\big|\\
		\end{split}
		\end{equation}

				Define a set $\mathcal{A}_1=\{\bfx| ({\bfw_{j_1}^*}^T\bfx)({\bfw_{j_1}}^T\bfx)<0 \}$. If $\bfx\in\mathcal{A}_1$, then ${\bfw_{j_1}^*}^T\bfx$ and ${\bfw_{j_1}}^T\bfx$ have different signs, which means the value of $\phi^{\prime}(\bfw_{j_1}^T\bfx)$ and $\phi^{\prime}({\bfw_{j_1}^*}^T\bfx)$ are different. This is equivalent to say that 
	\begin{equation}\label{eqn:I_1_sub1}
	|\phi^{\prime}(\bfw_{j_1}^T\bfx)-\phi^{\prime}({\bfw_{j_1}^*}^T\bfx)|=
	\begin{cases}
	&1, \text{ if $\bfx\in\mathcal{A}_1$}\\
	&0, \text{ if $\bfx\in\mathcal{A}_1^c$}
	\end{cases}.
	\end{equation}
	Moreover, if $\bfx\in\mathcal{A}_1$, then we have 
	\begin{equation}
	\begin{split}
	|{\bfw_{j_1}^*}^T\bfx|
	\le&|{\bfw_{j_1}^*}^T\bfx-{\bfw_{j_1}}^T\bfx|
	\le\|{\bfw_{j_1}^*}-{\bfw_{j_1}}\|\cdot\|\bfx\|.
	\end{split}
	\end{equation}
			Define a set $\mathcal{A}_2$ such that
	\begin{equation}
	\begin{split}
	\mathcal{A}_2
	=&\Big\{\bfx\Big|\frac{|{\bfw_{j_1}^*}^T\bfx|}{\|\bfw_{j_1}^*\|\|\bfx\|}\le\frac{\|{\bfw_{j_1}^*}-{\bfw_{j_1}}\|}{\|\bfw_{j_1}^*\|}   \Big\}
	=\Big\{\theta_{\bfx,\bfw_{j_1}^*}\Big||\cos\theta_{\bfx,\bfw_{j_1}^*}|\le\frac{\|{\bfw_{j_1}^*}-{\bfw_{j_1}}\|}{\|\bfw_{j_1}^*\|}   \Big\}.
	\end{split}
	\end{equation}	
	Hence, we have that
	\begin{equation}\label{eqn:I_1_sub2}
	\begin{split}
	& \mathbb{E}_{\bfx}
	|\phi^{\prime}(\bfw_{j_1}^T\bfx)-\phi^{\prime}({\bfw_{j_1}^*}^T\bfx_{i_2})|
	= \text{Prob}(\bfx\in\mathcal{A}_1)
	\le\text{Prob}(\bfx\in\mathcal{A}_2).
	\end{split}
	\end{equation}
	Since $\bfx\sim \mathcal{N}({\bf0}, \bfI)$, $\theta_{\bfx,\bfw_{j_1}^*}$ belongs to the uniform distribution on $[-\pi, \pi]$, we have
	\begin{equation}\label{eqn:I_1_sub3}
	\begin{split}
	\text{Prob}(\bfx\in\mathcal{A}_2)
	=&\frac{\pi- \arccos\frac{\|{\bfw_{j_1}^*}-{\bfw_{j_1}}\|}{\|\bfw_{j_1}^*\|} }{\pi}\\
	\le&\frac{1}{\pi}\tan(\pi- \arccos\frac{\|{\bfw_{j_1}^*}-{\bfw_{j_1}}\|}{\|\bfw_{j_1}^*\|})\\
	=&\frac{1}{\pi}\cot(\arccos\frac{\|{\bfw_{j_1}^*}-{\bfw_{j_1}}\|}{\|\bfw_{j_1}^*\|})\\
	\le&\frac{2}{\pi}\frac{\|{\bfw_{j_1}^*}-{\bfw_{j_1}}\|}{\|\bfw_{j_1}^*\|}.
	\end{split}
	\end{equation}
	Hence,  \eqref{eqn: ean111} and \eqref{eqn:I_1_sub3} suggest that
	\begin{equation}\label{eqn:I_1_bound}
	\mathbb{E}_{\bfX}\big|\phi^{\prime}\big(\bfa_n^T\bfX\bfw_{j_1}\big)-\phi^{\prime}\big(\bfa_n^T\bfX{\bfw_{j_1}^*}\big)\big|\le\frac{6}{\pi}\frac{\|{\bfw_{j_1}^*}-{\bfw_{j_1}}\|}{\|\bfw_{j_1}^*\|}.
	\end{equation}
	}
	
	Then, we have
		\begin{equation}\label{eqn: ean111}
		\begin{split}
		&\mathbb{E}_{\bfX}\big|\boldsymbol{\alpha}^T\bfX^T\bfa_n\big|^2\cdot \Big|\phi^{\prime}(\bfa_n^T\bfX\bfw_{j_1})-\phi^{\prime}(\bfa_n^T\bfX\bfw_{j_1}^*)\Big| \\
		=&3\|\bfa_n\|^2_2\cdot\mathbb{E}_{\bfX}\big|\phi^{\prime}\big(\bfa_n^T\bfX\bfw_{j_1}\big)-\phi^{\prime}\big(\bfa_n^T\bfX{\bfw_{j_1}^*}\big)\big|\\
		\le & \frac{6\|\bfa_n \|_2^2}{\pi}\cdot \frac{\|\bfw_{j_1} - \bfw_{j_1}^* \|_2}{\|\bfw_{j_1}^*\|_2},
		\end{split}
		\end{equation}
	
	All in all, we have
	\begin{equation}
	\begin{split}
	\left\|\nabla^2f_{\Omega_t}(\bfW) - \nabla^2f_{\Omega_t}(\bfW^*)  \right\|_2 
	\le&\sum_{j_1}^{K}\sum_{j_2}^{K}\Big\|\frac{\partial^2 f_{\Omega_t}}{\partial \bfw_{j_1}\partial \bfw_{j_2}}(\bfW) - \frac{\partial^2 f_{\Omega_t}}{\partial \bfw_{j_1}\partial \bfw_{j_2}}(\bfW^*)\Big\|_2\\
	\le& K^2\max_{j_1, j_2}\Big\|\frac{\partial^2 f_{\Omega_t}}{\partial \bfw_{j_1}\partial \bfw_{j_2}}(\bfW) - \frac{\partial^2 f_{\Omega_t}}{\partial \bfw_{j_1}\partial \bfw_{j_2}}(\bfW^*)\Big\|_2\\
	\le& K^2\cdot \frac{12\|\bfa_n \|_2^2}{\pi}\max_{j}\frac{\|\bfw_{j} - \bfw_{j}^* \|_2}{\|\bfw_{j}^*\|_2}\\
	\le & 4\sigma_1^2(\bfA)\frac{\|\bfW^*-\bfW\|_2}{\sigma_K}.
	\end{split}
	\end{equation}
\end{proof}

\subsubsection{Proof of Lemma \ref{Lemma: partial_independent_var}}
\begin{proof}[Proof of Lemma \ref{Lemma: partial_independent_var} ]
	According to the Definitions in \cite{J04}, there exists a family of $\{ (\mathcal{X}_j, w_j) \}_j$, where $\mathcal{X}_j \subseteq \mathcal{X} $ and $w_j \in [0,1]$, such that $\sum_{j}w_j \sum_{x_{n_j}\in\mathcal{X}_j}x_{n_j}=\sum_{n=1}^{N}x_n$, and $\sum_{j}w_j \le d_{\mathcal{X}}$ by equations (2.1) and (2.2) in \cite{J04}. Then, let $p_j$ be any positive numbers with $\sum_{j}p_j=1$. By Jensen's inequality, for any $s\in\mathbb{R}$, we have 
	\begin{equation}
	e^{s\sum_{n=1}^Nx_n}=e^{\sum_{j}p_j\frac{s w_j}{p_j}X_j}\le \sum_{j}p_je^{\frac{sw_j}{p_j}X_j},
	\end{equation}
	where $X_j = \sum_{x_{n_j}\in\mathcal{X}_j}x_{n_j}$.
	
	Then, we have
	\begin{equation}
	\begin{split}
	\mathbb{E}_{\mathcal{X}}e^{s\sum_{n=1}^Nx_n}
	\le& \mathbb{E}_{\mathcal{X}}\sum_{j}p_je^{\frac{sw_j}{p_j}X_j}
	= \sum_{j} p_j \prod_{\mathcal{X}_j}\mathbb{E}_{\mathcal{X}}e^{\frac{sw_j}{p_j}x_{n_j}}\\
	\le& \sum_{j} p_j \prod_{\mathcal{X}_j}e^{\frac{Cw_j^2}{p_j^2}s^2}\\
	\le & \sum_{j} p_j e^{\frac{C|\mathcal{X}_j|w_j^2}{p_j^2}s^2}.\\
	\end{split}
	\end{equation}
	
	Let $p_j = \frac{w_j|\mathcal{X}_j|^{1/2}}{\sum_{j}w_j|\mathcal{X}_j|^{1/2}}$, then we have 
	\begin{equation}
	\mathbb{E}_{\mathcal{X}}e^{s\sum_{n=1}^Nx_n}			
	\le \sum_{j} p_j e^{C\big(\sum_{j}w_j|\mathcal{X}_j|^{1/2}\big)^2s^2}
	=e^{C\big(\sum_{j}w_j|\mathcal{X}_j|^{1/2}\big)^2s^2}.
	\end{equation}
	
	By Cauchy-Schwarz inequality, we have 
	\begin{equation}
	\big(\sum_{j}w_j|\mathcal{X}_j|^{1/2}\big)^2\le \sum_{j}w_j \sum_{j}w_j |\mathcal{X}_j|\le d_{\mathcal{X}}N.
	\end{equation}
	
	Hence, we have 
	\begin{equation}
	\mathbb{E}_{\mathcal{X}}e^{s\sum_{n=1}^Nx_n}\le 	e^{Cd_{\mathcal{X}}Ns^2}.		
	\end{equation}	
\end{proof}

\section{Proof of Theorem \ref{Thm: major_thm_cl}}
Recall that the empirical risk function in \eqref{eqn: classification} is defined as
	\begin{equation}\label{eqn_app: cl}
		\begin{split}
		\min_{\bfW}: \quad \hat{f}_{\Omega}(\bfW)
		=&\frac{1}{|\Omega|}\sum_{n\in \Omega} -y_n \log \big(g(\bfW;\bfa_n^T\bfX)\big)
		-(1-y_n) \log \big(1-g(\bfW;\bfa_n^T\bfX)\big).
		\end{split}
	\end{equation}
	The population risk function is defined as
\begin{equation}\label{eqn_app: cl_exp}
	\begin{split}
	f_{\Omega}(\bfW):=& \mathbb{E}_{\bfX,y_n}\hat{f}_{\Omega}(\bfW)\\
	=&\mathbb{E}_{\bfX}\mathbb{E}_{y_n|\bfX}\Big[\frac{1}{|\Omega|}\sum_{n\in \Omega}
	- y_n \log \big(g(\bfW;\bfa_n^T\bfX)\big)
	-(1-y_n) \log \big(1-g(\bfW;\bfa_n^T\bfX)\big)\Big]\\
	=&\mathbb{E}_{\bfX}\frac{1}{|\Omega|}\sum_{n\in \Omega} - g(\bfW^*;\bfa_n^T\bfX) \log \big(g(\bfW;\bfa_n^T\bfX)\big)
	-(1-g(\bfW^*;\bfa_n^T\bfX)) \log \big(1-g(\bfW;\bfa_n^T\bfX)\big).
	\end{split}
\end{equation}
The road-map of proof for Theorem \ref{Thm: major_thm_cl} follows the similar three steps as those for Theorem \ref{Thm: major_thm_lr}. The major differences lie  in three aspects: (i) in the second step,
the objective function $\hat{f}_{\Omega_t}$ is smooth since the activation function $\phi(\cdot)$ is sigmoid. Hence, we can directly apply the mean value theorem as $\nabla \hat{f}_{\Omega_t}(\bfW^{(t)}) \simeq \langle \nabla^2\hat{f}_{\Omega_t}(\widehat{\bfW}^{(t)}), \bfW^{(t)}-\bfW^* \rangle$ to characterize the effects of the gradient descent term in each iteration, and the error bound of $\nabla^2\hat{f}_{\Omega_t}$ is provided in Lemma \ref{lemma: local_convex_cl}; 
(ii) the objective function is the sum of cross-entry loss functions, which have more complex structure of derivatives than those of square loss functions; 
(iii) as the convergent point may not be the critical point of empirical loss function,
we need to provide  the distance from the convergent point to the ground-truth parameters additionally, where Lemma \ref{lemma: first_order_cl} is used.

Lemmas \ref{lemma: local_convex_cl} and \ref{lemma: first_order_cl} are summarized in the following contents. Also, the notations $\lesssim$ and $\gtrsim$ follow the same definitions as in \eqref{eqn: lemma3}.
The proofs of Lemmas \ref{lemma: local_convex_cl} and \ref{lemma: first_order_cl} can be found in Appendix \ref{sec: local_cl} and \ref{sec: sub_cl}, respectively.
\begin{lemma}\label{lemma: local_convex_cl}
	For any $\bfW$ that satisfies
	\begin{equation}\label{eqn: initial_point_cl}
	\|\bfW-\bfW^*\| \le \frac{2\sigma_1^2(\bfA)}{11\kappa^2\gamma K^2}
	\end{equation}
	then the second-order derivative of the empirical risk function in \eqref{eqn_app: cl} for binary classification problems is bounded as 
	\begin{equation}
			\frac{2(1-\varepsilon_0)}{11\kappa^2 \gamma K^2}\sigma_1^2(\bfA) 
			\preceq \nabla^2\hat{f}_{\Omega_t}(\bfW)
			\preceq \sigma_1^2(\bfA).
	\end{equation}	
	provided the number of samples satisfies	
	\begin{equation}\label{eqn: sample_cl}
	|\Omega_t|\gtrsim \varepsilon_0^{-2}(1+\delta^2)\kappa^2\gamma \sigma_1^4(\bfA)K^6d\log N.
	\end{equation}
\end{lemma}

\begin{lemma}\label{lemma: first_order_cl}
	Let $\hat{f}_{\Omega_t}$ and ${f}_{\Omega_t}$ be the empirical and population risk function in \eqref{eqn_app: cl} and \eqref{eqn_app: cl_exp} for binary classification problems, respectively, then the first-order derivative of $\hat{f}_{\Omega_t}$ is close to its expectation ${f}_{\Omega_t}$ with an upper bound as
	\begin{equation}
	\begin{split}
	\|\nabla f_{\Omega_t}(\bfW) - \nabla \hat{f}_{\Omega_t}(\bfW) \|_2
	\lesssim K^2\sigma_1^2(\bfA)\sqrt{\frac{(1+\delta^2)d\log d}{|\Omega_t|}}
	\end{split}
	\end{equation}
	with probability at least $1-K^2N^{-10}$.
\end{lemma}
 With these preliminary lemmas, the proof of Theorem \ref{Thm: major_thm_cl} is formally summarized in the following contents.
\begin{proof}[Proof of Theorem \ref{Thm: major_thm_cl}]
	The update rule of $\bfW^{(t)}$ is
	\begin{equation}
	\begin{split}
	\bfW^{(t+1)}
	=&\bfW^{(t)}-\eta\nabla \hat{f}_{\Omega_t}(\bfW^{(t)})+\beta(\W[t]-\W[t-1])\\
	\end{split}
	\end{equation}
	Since $\widehat{\bfW}$ is a critical point, then we have  $\nabla \hat{f}_{\Omega_t}(\widehat{\bfW})=0$. By the intermediate value theorem, we have
	\begin{equation}
	\begin{split}
	\bfW^{(t+1)}
	=\bfW^{(t)}&- \eta\nabla^2\hat{f}_{\Omega_t}(\widehat{\bfW}^{(t)})(\bfW^{(t)}-\widehat{\bfW})\\
	&+\beta(\W[t]-\W[t-1])\\
	\end{split}
	\end{equation}	
	where $\widehat{\bfW}^{(t)}$ lies in the convex hull of $\bfW^{(t)}$ and $\widehat{\bfW}$.
	
	Next, we have
	\begin{equation}\label{eqn: major_thm_key_eqn_cl}
	\begin{split}
	\begin{bmatrix}
	\W[t+1]-\bfW^*\\
	\W[t]-\bfW^*
	\end{bmatrix}
	=&\begin{bmatrix}
	\bfI-\eta\nabla^2{\hat{f}}_{\Omega_t}(\widehat{\bfW}^{(t)})+\beta\bfI &\beta\bfI\\
	\bfI& 0\\
	\end{bmatrix}
	\begin{bmatrix}
	\W[t]-\bfW^*\\
	\W[t-1]-\bfW^*
	\end{bmatrix}
	.
	\end{split}
	\end{equation}
	Let $\bfP(\beta)=\begin{bmatrix}
	\bfI-\eta\nabla^2{\hat{f}}_{\Omega_t}(\widehat{\bfW}^{(t)})+\beta\bfI &\beta\bfI\\
	\bfI& 0\\
	\end{bmatrix}$, so we have
	\begin{equation*}
	\begin{split}
	\left\|\begin{bmatrix}
	\W[t+1]-\bfW^*\\
	\W[t]-\bfW^*
	\end{bmatrix}
	\right\|_2
	=&
	\left\|\bfP(\beta)
	\right\|_2
	\left\|
	\begin{bmatrix}
	\W[t]-\bfW^*\\
	\W[t-1]-\bfW^*
	\end{bmatrix}
	\right\|_2.
	\end{split}
	\end{equation*}
	Then, we have
	\begin{equation}\label{eqn: convergence_cl}
	\begin{split}
	\|\W[t+1]-\bfW^* \|_2 
	\lesssim& \| \bfP(\beta) \|_2 \|\W[t]-\bfW^* \|_2\\
	\end{split}
	\end{equation}	
Let $\lambda_i$ be the $i$-th eigenvalue of $\nabla^2\hat{f}_{\Omega_t}(\widehat{\bfW}^{(t)})$, and $\delta_{i}$ be the $i$-th eigenvalue of matrix $\bfP(\beta)$. Following the similar analysis in proof of Theorem \ref{Thm: major_thm_lr}, we have
	\begin{equation}\label{eqn: Heavy_ball_result_cl}
	\delta_i(0)>\delta_i(\beta),\quad \text{for}\quad  \forall \beta\in\big(0, (1-{\eta \lambda_i})^2\big).
	\end{equation}
	Moreover, $\delta_i$ achieves the minimum $\delta_{i}^*=|1-\sqrt{\eta\lambda_i}|$ when $\beta= \big(1-\sqrt{\eta\lambda_i}\big)^2$.
	
	Let us first assume $\bfW^{(t)}$ satisfies \eqref{eqn: initial_point_cl} and the number of samples satisfies \eqref{eqn: sample_cl}, then
	from Lemma \ref{lemma: local_convex_cl}, we know that 
	$$0<\frac{2(1-\varepsilon_0)\sigma_1^2(\bfA)}{11\kappa^2\gamma K^2}\le\lambda_i\le {\sigma_1^2(\bfA)}.$$
	We define $\gamma_1=\frac{2(1-\varepsilon_0)\sigma_1^2(\bfA)}{11\kappa^2\gamma K^2}$ and $\gamma_2={\sigma_1^2(\bfA)}$.
	Also, for any $\varepsilon_0\in (0, 1)$, we have
	\begin{equation}\label{eqn: con1_cl}
	\begin{split}
	\nu(\beta^*)= \|\bfP(\beta^*) \|_2
	= 1-\sqrt{\frac{\gamma_1}{2\gamma_2}}
	=1-\sqrt{\frac{1-\varepsilon_0}{11\kappa^2 \gamma K}}
	\end{split}
	\end{equation}
	Let $\beta = 0$, we have 
	\begin{equation*}
	\begin{gathered}
	\nu(0) = \|\bfA(0) \|_2 = 1-\frac{1-\varepsilon_0}{{11\kappa^2 \gamma K}}.	\\
	\end{gathered}
	\end{equation*}
	
	Hence, with probability at least $1- K^2\cdot N^{-10}$, we have
	\begin{equation}\label{eqn: induction_cl}
	\begin{split}
	\|	\W[t+1]-\bfW^*\|_2
	\le\Big(1-\sqrt{\frac{1-\varepsilon_0}{11\kappa^2 \gamma K}}\Big)\|	\W[t]-\bfW^*\|_2,
	\end{split}
	\end{equation}
	provided that $\W[t]$ satisfies \eqref{eqn: initial_point_lr}, and 
	\begin{equation}\label{eqn: N_3_cl}
	|\Omega_t|\gtrsim\varepsilon_0^{-2}\kappa^2\gamma (1+\delta^2)\sigma_1^4(\bfA)K^6d \log N.
	\end{equation}
	According to Lemma \ref{Thm: initialization}, we know that \eqref{eqn: initial_point_cl} holds for $\bfW^{(0)}$ if 
	\begin{equation}\label{eqn: N_1_cl}
	|\Omega_t|\gtrsim \varepsilon_0^{-2}\kappa^{8}\gamma^2 (1+\delta^2) K^8d \log N.
	\end{equation}
	Combining \eqref{eqn: N_3_cl} and \eqref{eqn: N_1_cl}, we need $|\Omega_t|\gtrsim \varepsilon_0^{-2}\kappa^{8}\gamma^2 (1+\delta^2)\sigma_1^4(\bfA) K^8d \log N$.	

	Finally,  by the mean value theorem, we have 
	\begin{equation}
	\hat{f}_{\Omega_t}(\widehat{\bfW}) \le \hat{f}_{\Omega_t}(\bfW^*) + {\nabla\hat{f}_{\Omega_t}(\bfW^*)}^T(\widehat{\bfW}-\bfW^*) +
	\frac{1}{2} (\widehat{\bfW}-\bfW^*)^T\nabla^2\hat{f}_{\Omega_t}(\widetilde{\bfW})(\widehat{\bfW}-\bfW^*)
	\end{equation}
	for some $\widetilde{\bfW}$ between $\widehat{\bfW}$ and $\bfW^*$.
	Since $\widehat{\bfW}$ is the local minima, we have $\hat{f}_{\Omega_t}(\widehat{\bfW})\le \hat{f}_{\Omega_t}(\bfW^*)$. That is to say 
	\begin{equation}
	{\nabla\hat{f}_{\Omega_t}(\bfW^*)}^T(\widehat{\bfW}-\bfW^*) +
	\frac{1}{2} (\widehat{\bfW}-\bfW^*)^T\nabla^2\hat{f}_{\Omega_t}(\widetilde{\bfW})(\widehat{\bfW}-\bfW^*)\le 0
	\end{equation}
	which implies
	\begin{equation}\label{eqn: cccccc3}
	\frac{1}{2} \|\nabla^2\hat{f}_{\Omega_t}(\widetilde{\bfW})\|_2\|\widehat{\bfW}-\bfW^*\|_2^2\le \|{\nabla\hat{f}_{\Omega_t}(\bfW^*)}\|_2\|\widehat{\bfW}-\bfW^*\|_2.
	\end{equation}
	From Lemma \ref{lemma: local_convex_cl}, we know that 
	\begin{equation}\label{eqn: cccccc1}
	\|\nabla^2\hat{f}_{\Omega_t}(\widetilde{\bfW})\|_2\ge \frac{2(1-\varepsilon_0)}{11\kappa^2\gamma K^2}\sigma^2(\bfA).
	\end{equation}
	From Lemma \ref{lemma: first_order_cl}, we know that
	\begin{equation}\label{eqn: cccccc2}
	\|{\nabla\hat{f}_{\Omega_t}(\bfW^*)} \|_2 = \| {\nabla\hat{f}_{\Omega_t}(\bfW^*)} - {\nabla{f}_{\Omega_t}(\bfW^*)}\|_2 \lesssim K^2 \sigma_1^2(\bfA)\sqrt{\frac{(1+\delta^2)d\log N}{|\Omega_t|}}.
	\end{equation}
	Plugging inequalities \eqref{eqn: cccccc1} and \eqref{eqn: cccccc2} back into \eqref{eqn: cccccc3}, we have
	\begin{equation}
	\|\widehat{\bfW} -\bfW^* \|_2 \lesssim (1-\varepsilon_0)^{-1}\kappa^2 \gamma K^4\sqrt{\frac{(1+\delta^2)d\log d}{|\Omega_t|}}.
	\end{equation}
	
\end{proof}

\subsection{Proof of Lemma \ref{lemma: local_convex_cl}}\label{sec: local_cl}
The roadmap of proof for Lemma \ref{lemma: local_convex_cl} follows the similar steps as those  of Lemma \ref{Lemma: local_convexity} for regression problems. Lemmas \ref{Lemma: CL_ground_truth}, \ref{Lemma: CL_W} and \ref{Lemma: CL_sample} are the preliminary lemmas, and
their proofs can be found in Appendix \ref{sec: sub_cl}. The proof of Lemma \ref{lemma: local_convex_cl} is summarized after these preliminary lemmas. 
\begin{lemma}\label{Lemma: CL_ground_truth}
The second-order derivative of $f_{\Omega_t}$ at the ground truth $\bfW^*$ satisfies
	\begin{equation}
	\frac{4\sigma_1^2(\bfA)}{11\kappa^2 \gamma K^2}\bfI 
	\preceq \nabla^2f_{\Omega_t}(\bfW^*)
	\preceq \frac{\sigma_1^2(\bfA)}{4}\bfI.
	\end{equation}
\end{lemma}

\begin{lemma}\label{Lemma: CL_W}
	Suppose $f_{\Omega_t}$ is the population loss function with respect to binary classification problems, then we have
	\begin{equation}
	\|\nabla^2f_{\Omega_t}(\bfW) - \nabla^2f_{\Omega_t}(\bfW^*) \|_2 \lesssim \|\bfW -\bfW^* \|_2.
	\end{equation}
\end{lemma}
\begin{lemma}\label{Lemma: CL_sample}
	Suppose $\hat{f}_{\Omega_t}$ is the empirical loss function with respect to binary classification problems, then the second-order derivative of $\widehat{f}_{\Omega_t}$ is close to its expectation with an upper bound as
	\begin{equation}
	\begin{split}
	\|\nabla^2f_{\Omega_t}(\bfW) - \nabla^2\hat{f}_{\Omega_t}(\bfW) \|_2
	\lesssim K^2\sigma_1^2(\bfA)\sqrt{\frac{(1+\delta^2)d\log d}{|\Omega_t|}}
	\end{split}
	\end{equation}
	with probability at least $1-K^2N^{-10}$.
\end{lemma}

\begin{proof}[Proof of Lemma \ref{lemma: local_convex_cl} ]
For any $\bfW$, we have 
\begin{equation}
\Big|\| \nabla^2 f_{\Omega_t}(\bfW)\|_2 -\|\nabla^2 f_{\Omega_t}(\bfW^*)  \|_2\Big|
\le\| \nabla^2 f_{\Omega_t}(\bfW) -\nabla^2 f_{\Omega_t}(\bfW^*)  \|_2 .
\end{equation}
That is 
\begin{equation}
\begin{split}
&\| \nabla^2 f_{\Omega_t}(\bfW)\|_2 \le \|\nabla^2 f_{\Omega_t}(\bfW^*)  \|_2 + \| \nabla^2 f_{\Omega_t}(\bfW) -\nabla^2 f_{\Omega_t}(\bfW^*)  \|_2 \\
\text{and \quad} &\| \nabla^2 f_{\Omega_t}(\bfW)\|_2 \ge \|\nabla^2 f_{\Omega_t}(\bfW^*)  \|_2 - \| \nabla^2 f_{\Omega_t}(\bfW) -\nabla^2 f_{\Omega_t}(\bfW^*)  \|_2 \\
\end{split}
\end{equation}
Then, for any $\bfW$ that satisfies $\|\bfW-\bfW^*\| \le \frac{2\sigma_1^2(\bfA)}{11\kappa^2\gamma K^2}$, from Lemmas \ref{Lemma: CL_ground_truth} and \ref{Lemma: CL_W}, we have 
\begin{equation}\label{eqn: CL_theorem_1}
\frac{2}{11\kappa^2 \gamma K^2}\sigma_1^2(\bfA) 
\preceq \nabla^2f_{\Omega_t}(\bfW)
\preceq \frac{1}{2}\sigma_1^2(\bfA).
\end{equation}
Next, we have 
\begin{equation}
\begin{split}
&\| \nabla^2 \hat{f}_{\Omega_t}(\bfW)\|_2 \le \|\nabla^2 f_{\Omega_t}(\bfW)\|_2+ \|\nabla^2\hat{f}_{\Omega_t}(\bfW) -\nabla^2 f_{\Omega_t}(\bfW) \|_2  \\
\text{and \quad} &\| \nabla^2 \hat{f}_{\Omega_t}(\bfW)\|_2 \ge \| \nabla^2 f_{\Omega_t}(\bfW)\|_2- \|\nabla^2 \hat{f}_{\Omega_t}(\bfW) -\nabla^2 f_{\Omega_t}(\bfW) \|_2  \\
\end{split}
\end{equation}
Then, from \eqref{eqn: CL_theorem_1} and Lemma \ref{Lemma: CL_sample}, we have 
\begin{equation}
\frac{2(1-\varepsilon_0)}{11\kappa^2 \gamma K^2}\sigma_1^2(\bfA) 
\preceq \nabla^2\hat{f}_{\Omega_t}(\bfW)
\preceq \sigma_1^2(\bfA)
\end{equation}
provided that  the sample size $|\Omega_t|\gtrsim \varepsilon_0^{-2}(1+\delta^2)\kappa^2\gamma \sigma_1^4(\bfA)K^6d\log N$.
\end{proof}

\subsection{Proof of auxiliary lemmas for binary classification problems}\label{sec: sub_cl}
\subsubsection{Proof of Lemma \ref{Lemma: CL_ground_truth}}
\begin{proof}[Proof of Lemma \ref{Lemma: CL_ground_truth} ]
	Since $\mathbb{E}_{\bfX}y_n = g_n(\bfW^*;\bfa_n)$, then  we have 
	\begin{equation}
	\begin{split}
	\frac{\partial^2 f_{{\Omega_t}}(\bfW^*)}{\partial \bfw_j^* \partial \bfw_k^*}
	=&\mathbb{E}_{\bfX} \frac{\partial^2 \hat{f}_{\Omega_t}(\bfW^*)}{\partial \bfw_j^* \partial \bfw_k^*}\\
	=&\mathbb{E}_{\bfX}\frac{1}{K^2|\Omega_t|}\sum_{n\in\Omega_t}\frac{1}{g(\bfW;\bfa_n)\big(1-g(\bfW;\bfa_n)\big)}\phi^{\prime}(\bfw_j^T\bfX^T\bfa_n)\phi^{\prime}(\bfw_k^T\bfX^T\bfa_n)(\bfX^T\bfa_n)(\bfX^T\bfa_n)^T,
	\end{split}
	\end{equation}
	for any $j, k\in [K]$.
	
	Then, for any $\boldsymbol{\alpha} = \begin{bmatrix}
	\boldsymbol{\alpha}_1^T ,& \boldsymbol{\alpha}_2^T, &\cdots,&\boldsymbol{\alpha}_K^T
	\end{bmatrix}^T\in\mathbb{R}^{dk}$ with $\boldsymbol{\alpha}_j \in \mathbb{R}^d$, the lower bound can be obtained from
	\begin{equation}
	\begin{split}
	\boldsymbol{\alpha}^T \nabla^2 f_{\Omega_t}(\bfW^*) \boldsymbol{\alpha}
	=&\mathbb{E}_{\bfX}\frac{1}{K^2|\Omega_t|}\sum_{n\in\Omega_t}\frac{\Big(\sum_{j=1}^{K}\boldsymbol{\alpha}_j^T\bfX^T\bfa_n\phi^{\prime}({\bfw_j^*}^T\bfX^T\bfa_n)\Big)^2}{g(\bfW^*;\bfa_n)\big(1-g(\bfW^*;\bfa_n)\big)}\\
	\ge&\mathbb{E}_{\bfX}\frac{4}{K^2|\Omega_t|}\sum_{n\in\Omega_t}\Big(\sum_{j=1}^{K}\boldsymbol{\alpha}_j^T\bfX^T\bfa_n\phi^{\prime}({\bfw_j^*}^T\bfX^T\bfa_n)\Big)^2\\
	\ge& \frac{4\sigma_1^2(\bfA)}{11\kappa^2\gamma K^2}.
	\end{split}
	\end{equation}
	
	Also, for the upper bound, we have
	\begin{equation}
	\begin{split}
	\boldsymbol{\alpha}^T \nabla^2 f_{\Omega_t}(\bfW^*) \boldsymbol{\alpha}
	=&\mathbb{E}_{\bfX}\frac{1}{K^2|\Omega_t|}\sum_{n\in\Omega_t}\frac{\Big(\sum_{j=1}^{K}\boldsymbol{\alpha}_j^T\bfX^T\bfa_n\phi^{\prime}({\bfw_j^*}^T\bfX^T\bfa_n)\Big)^2}{g(\bfW^*;\bfa_n)\big(1-g(\bfW^*;\bfa_n)\big)}\\
	=&\mathbb{E}_{\bfX}\frac{1}{|\Omega_t|}\sum_{n\in\Omega_t}\frac{\Big(\sum_{j=1}^{K}\boldsymbol{\alpha}_j^T\bfX^T\bfa_n\phi^{\prime}({\bfw_j^*}^T\bfX^T\bfa_n)\Big)^2}{\sum_{j_1=1}^{K}\phi({\bfw_{j_1}^*}^T\bfX^T\bfa_n)\sum_{j_2=1}^{K}\big(1-\phi({\bfw_{j_2}^*}^T\bfX^T\bfa_n)\big)}\\
	\le& \mathbb{E}_{\bfX}\frac{1}{|\Omega_t|}\sum_{n\in\Omega_t}\frac{\sum_{j=1}^{K}\big(\boldsymbol{\alpha}_j^T\bfX^T\bfa_n\big)^2 \sum_{j=1}^{K}\big( \phi^{\prime}({\bfw_j^*}^T\bfX^T\bfa_n)\big)^2}{\sum_{j_1=1}^{K}\phi({\bfw_{j_1}^*}^T\bfX^T\bfa_n)\sum_{j_2=1}^{K}\big(1-\phi({\bfw_{j_2}^*}^T\bfX^T\bfa_n)\big)}.\\
	\end{split}
	\end{equation}
	For the denominator item, we have 
	\begin{equation}
	\begin{split}
	\sum_{j_1=1}^{K}\phi({\bfw_{j_1}^*}^T\bfX^T\bfa_n)\sum_{j_2=1}^{K}\big(1-\phi({\bfw_{j_2}^*}^T\bfX^T\bfa_n)\big)
	\ge&\sum_{j=1}^{K}\phi({\bfw_{j}^*}^T\bfX^T\bfa_n)\big(1-\phi({\bfw_{j}^*}^T\bfX^T\bfa_n)\big)\\
	=&\sum_{j=1}^{K}\phi^{\prime}({\bfw_{j}^*}^T\bfX^T\bfa_n)\\
	\ge&4 \sum_{j=1}^{K}\phi^{\prime}({\bfw_{j}^*}^T\bfX^T\bfa_n)^2.\\
	\end{split}
	\end{equation}
	Hence, we have
	\begin{equation}
	\begin{split}
	\boldsymbol{\alpha}^T \nabla^2 f_{\Omega_t}(\bfW^*) \boldsymbol{\alpha}
	\le \mathbb{E}_{\bfX}\frac{1}{4|\Omega_t|}\sum_{n\in \Omega_t}\sum_{j=1}^{K}(\boldsymbol{\alpha}_j^T \bfX^T\bfa_n)^2
	\le \frac{1}{4}\sigma_1^2(\bfA).
	\end{split}
	\end{equation}
\end{proof}
\subsubsection{Proof of Lemma \ref{Lemma: CL_W}}
\begin{proof}[Proof of Lemma \ref{Lemma: CL_W}  ]
	Recall that
	\begin{equation}
	\begin{split}
	&\frac{\partial^2 {f}_{\Omega_t}(\bfW)}{\partial \bfw_j \partial \bfw_k} \\
	=&\mathbb{E}_{\bfX} \frac{1}{K^2|\Omega_t|}\sum_{n\in\Omega_t}\bigg(\frac{g(\bfW^*;\bfa_n)}{g^2(\bfW;\bfa_n)} + \frac{1 - g(\bfW^*;\bfa_n)}{(1-g(\bfW;\bfa_n))^2} \bigg)\phi^{\prime}(\bfw_j^T\bfX^T\bfa_n)\phi^{\prime}(\bfw_k^T\bfX^T\bfa_n)(\bfX^T\bfa_n)(\bfX^T\bfa_n)^T,
	\end{split}
	\end{equation}
	and
	\begin{equation}
	\begin{split}
	\frac{\partial^2 {f}_{\Omega_t}(\bfW)}{\partial \bfw_j^2 } 
	=&\mathbb{E}_{\bfX} \frac{1}{K^2|\Omega_t|}\sum_{n\in\Omega_t}\bigg(\frac{g(\bfW^*;\bfa_n)}{g^2(\bfW;\bfa_n)} + \frac{1 - g(\bfW^*;\bfa_n)}{(1-g(\bfW;\bfa_n))^2} \bigg)\phi^{\prime}(\bfw_j^T\bfX^T\bfa_n)^2(\bfX^T\bfa_n)(\bfX^T\bfa_n)^T\\
	&-\mathbb{E}_{\bfX}\frac{1}{K|\Omega_t|}\sum_{n\in\Omega_t}\bigg(-\frac{g(\bfW^*;\bfa_n)}{g(\bfW;\bfa_n)} + \frac{1 - g(\bfW^*;\bfa_n)}{1-g(\bfW;\bfa_n)} \bigg)\phi^{\prime\prime}(\bfw_j^T\bfX^T\bfa_n)(\bfX^T\bfa_n)(\bfX^T\bfa_n)^T.
	\end{split}
	\end{equation}
	Let us denote $A_{j,k}(\bfW;\bfa_n)$ as 
	\begin{equation}
	A_{j,k}(\bfW;\bfa_n) = \begin{cases}
	\frac{1}{K^2}\big(\frac{g(\bfW^*;\bfa_n)}{g^2(\bfW;\bfa_n)} + \frac{1 - g(\bfW^*;\bfa_n)}{(1-g(\bfW;\bfa_n))^2} \big)\phi^{\prime}(\bfw_j^T\bfX^T\bfa_n)\phi^{\prime}(\bfw_k^T\bfX^T\bfa_n)\\
	~~~~~~~~~~~~~-\frac{1}{K}\big(-\frac{g(\bfW^*;\bfa_n)}{g(\bfW;\bfa_n)} + \frac{1 - g(\bfW^*;\bfa_n)}{1-g(\bfW;\bfa_n)} \big)\phi^{\prime\prime}(\bfw_j^T\bfX^T\bfa_n),\text{\quad when \quad} j=k;\\
	~\\
	\frac{1}{K^2}\big(\frac{g(\bfW^*;\bfa_n)}{g^2(\bfW;\bfa_n)} + \frac{1 - g(\bfW^*;\bfa_n)}{(1-g(\bfW;\bfa_n))^2} \big)\phi^{\prime}(\bfw_j^T\bfX^T\bfa_n)\phi^{\prime}(\bfw_k^T\bfX^T\bfa_n), \text{\quad  when \quad} j\neq k.		 
	\end{cases}
	\end{equation} 
	Further, let us define $M(\bfW;\bfa_n) = \max \Big\{\frac{2}{K^3}\frac{1}{g^3(\bfW;\bfa_n)}, \frac{2}{K^3}\frac{1}{(1-g(\bfW;\bfa_n))^3}, \frac{1}{K^2}\frac{1}{g^2(\bfW;\bfa_n)},\frac{1}{K^2} \frac{1}{ ((1-g(\bfW;\bfa_n))^2 } \Big\}$.
	
	Then, by the mean value theorem, we have 
	\begin{equation}
	A_{j,k}(\bfW;\bfa_n) - A_{j,k}(\bfW;\bfa_n) = \sum_{l=1}^{K}\langle\frac{\partial A_{j,k}}{\partial \bfw_{l}}(\widetilde{\bfW};\bfa_n), \bfw_l-\bfw_l^* \rangle.
	\end{equation}
	For $\frac{\partial A_{j,k}}{\partial \bfw_{l}}$, we have 
	\begin{equation}
	\frac{\partial A_{j,k}}{\partial \bfw_{l}}(\widetilde{\bfW};\bfa_n) = B_{j,k,l}(\widetilde{\bfW};\bfa_n) \bfX^T\bfa_n
	\end{equation}
	with 
	\begin{equation}
	\begin{split}
	|B_{j,k,l}(\widetilde{\bfW};\bfa_n)| 
	\le& \frac{2}{K^3}\frac{1}{g^3(\widetilde{\bfW};\bfa_n)} +  \frac{2}{K^3}\frac{1}{(1-g(\widetilde{\bfW};\bfa_n))^3} + \frac{1}{K^2}\frac{1}{g(\widetilde{\bfW};\bfa_n)} + \frac{1}{K^2} \frac{1}{ (1-g(\widetilde{\bfW};\bfa_n))^2}\\
	\le & 4 M(\widetilde{\bfW};\bfa_n).
	\end{split}
	\end{equation}
	for all $j\in[K],k\in[K],l\in[K]$.
	
	Therefore, for any $\boldsymbol{\alpha}\in \mathbb{R}^{Kd}$, we have 
	\begin{equation}
	\begin{split}
	&\boldsymbol{\alpha}^T \nabla^2f_{\Omega_t}(\bfW)\boldsymbol{\alpha}\\
	\le& \frac{1}{|\Omega_t|}\sum_{n\in |\Omega_t|} \sum_{j=1}^{K}\sum_{k=1}^{K}\mathbb{E}_{\bfX}\Big|\boldsymbol{\alpha}_j^T \frac{\partial f_{\Omega_t}}{\partial \bfw_j \partial \bfw_k}(\bfW) \boldsymbol{\alpha}_k \Big|\\
	=& \frac{1}{|\Omega_t|}\sum_{n\in |\Omega_t|} \sum_{j=1}^{K}\sum_{k=1}^{K}\mathbb{E}_{\bfX}\Big|\sum_{l=1}^{K}|B_{j,k,l}(\widetilde{\bfW};\bfa_n)| \langle \bfw_l-\bfw_l^*, \bfX^T\bfa_n \rangle \langle \boldsymbol{\alpha}_j, \bfX^T\bfa_n \rangle\langle \boldsymbol{\alpha}_k, \bfX^T\bfa_n \rangle \Big|\\
	=& \frac{1}{|\Omega_t|}\sum_{n\in |\Omega_t|} \sum_{j=1}^{K}\sum_{k=1}^{K}\bigg(\sum_{l=1}^{K}\mathbb{E}_{\bfX}|B_{j,k,l}(\widetilde{\bfW};\bfa_n)|^2\bigg)^{\frac{1}{2}}
	\bigg(\sum_{l=1}^{K}\mathbb{E}_{\bfX}\big|\langle \bfw_l-\bfw_l^*, \bfX^T\bfa_n \rangle \langle \boldsymbol{\alpha}_j, \bfX^T\bfa_n \rangle\langle \boldsymbol{\alpha}_k, \bfX^T\bfa_n \rangle\big|^2\bigg)^{\frac{1}{2}}\\
	\le&\frac{1}{|\Omega_t|}\sum_{n\in |\Omega_t|} \sum_{j=1}^{K}\sum_{k=1}^{K} 36K^{\frac{1}{2}} \Big(\mathbb{E}_\bfX M^2(\widehat{\bfW};\bfa_n) \Big)^\frac{1}{2} \cdot 
	\Big(\sum_{l=1}^{K}\|\bfw_l -\bfw_l^*\|_2^2\Big)^{\frac{1}{2}} \|\bfa_j\|_2 \|\bfa_k\|_2\\
	\le &\frac{1}{|\Omega_t|}\sum_{n\in |\Omega_t|}  36K^{3} \Big(\mathbb{E}_\bfX M^2(\widehat{\bfW};\bfa_n) \Big)^\frac{1}{2} 
	\|\bfW -\bfW^*\|_2\\
	\stackrel{(a)}{\lesssim} & e^{\sigma_1^2(\bfA)} \|\bfW -\bfW^*\|_2\\  
	\lesssim & \|\bfW -\bfW^*\|_2,
	\end{split}
	\end{equation}
	where ($a$) comes from Lemma 5 in \cite{FCL19}.
\end{proof}

\subsubsection{Proof of Lemma \ref{Lemma: CL_sample}}
\begin{proof}[Proof of Lemma \ref{Lemma: CL_sample} ]
	Recall that
	\begin{equation}
	\begin{split}
	&\frac{\partial^2 \hat{f}_{\Omega_t}(\bfW)}{\partial \bfw_j \partial \bfw_k} \\
	=&\frac{1}{K^2|\Omega_t|}\sum_{n\in\Omega_t}\bigg(\frac{y_n}{g^2(\bfW;\bfa_n)} + \frac{1 - y_n}{(1-g(\bfW;\bfa_n))^2} \bigg)\phi^{\prime}(\bfw_j^T\bfX^T\bfa_n)\phi^{\prime}(\bfw_k^T\bfX^T\bfa_n)(\bfX^T\bfa_n)(\bfX^T\bfa_n)^T,
	\end{split}
	\end{equation}
	and
	\begin{equation}
	\begin{split}
	\frac{\partial^2 \hat{f}_{\Omega_t}(\bfW)}{\partial \bfw_j^2 } 
	=&\frac{1}{K^2|\Omega_t|}\sum_{n\in\Omega_t}\bigg(\frac{y_n}{g^2(\bfW;\bfa_n)} + \frac{1 - y_n}{(1-g(\bfW;\bfa_n))^2} \bigg)\phi^{\prime}(\bfw_j^T\bfX^T\bfa_n)^2(\bfX^T\bfa_n)(\bfX^T\bfa_n)^T\\
	&-\frac{1}{K|\Omega_t|}\sum_{n\in\Omega_t}\bigg(-\frac{y_n}{g(\bfW;\bfa_n)} + \frac{1 - y_n}{1-g(\bfW;\bfa_n)} \bigg)\phi^{\prime\prime}(\bfw_j^T\bfX^T\bfa_n)(\bfX^T\bfa_n)(\bfX^T\bfa_n)^T.
	\end{split}
	\end{equation}
	
	When $y_n=1$ and $j\neq k$, we have 
	\begin{equation}
	\begin{split}
	\frac{\partial^2 \hat{f}_{\Omega_t}(\bfW)}{\partial \bfw_j \partial \bfw_k} 
	=\frac{1}{K^2|\Omega_t|}\sum_{n\in\Omega_t}\frac{\phi^{\prime}(\bfw_j^T\bfX^T\bfa_n)\phi^{\prime}(\bfw_k^T\bfX^T\bfa_n)}{g^2(\bfW;\bfa_n)}(\bfX^T\bfa_n)(\bfX^T\bfa_n)^T,
	\end{split}
	\end{equation}
	and 
	\begin{equation}\label{eqn: sths_1}
	\begin{split}
	\frac{\phi^{\prime}(\bfw_j^T\bfX^T\bfa_n)\phi^{\prime}(\bfw_k^T\bfX^T\bfa_n)}{g^2(\bfW;\bfa_n)}
	=& \frac{\phi(\bfw_j^T\bfX^T\bfa_n)(1-\phi(\bfw_j^T\bfX^T\bfa_n))\phi(\bfw_k^T\bfX^T\bfa_n)(1-\phi(\bfw_k^T\bfX^T\bfa_n))}{\big(\frac{1}{K}\sum_{l=1}^{K}\phi(\bfw_l^T\bfX^T\bfa_n)\big)^2}\\
	\le&K^2\frac{\phi(\bfw_j^T\bfX^T\bfa_n)(1-\phi(\bfw_j^T\bfX^T\bfa_n))\phi(\bfw_k^T\bfX^T\bfa_n)(1-\phi(\bfw_k^T\bfX^T\bfa_n))}{\phi(\bfw_j^T\bfX^T\bfa_n)\phi(\bfw_k^T\bfX^T\bfa_n)}\\
	=&K^2(1-\phi(\bfw_j^T\bfX^T\bfa_n))(1-\phi(\bfw_k^T\bfX^T\bfa_n))\le K^2.
	\end{split}
	\end{equation}
	When $y_n=1$ and $j= k$, we have
	\begin{equation}
	\begin{split}
	\frac{\partial^2 \hat{f}_{\Omega_t}(\bfW)}{\partial \bfw_j \partial \bfw_k} 
	=\frac{1}{|\Omega_t|}\sum_{n\in\Omega_t}\Big[\frac{1}{K^2}\frac{\phi^{\prime}(\bfw_j^T\bfX^T\bfa_n)\phi^{\prime}(\bfw_k^T\bfX^T\bfa_n)}{g^2(\bfW;\bfa_n)}
	+\frac{1}{K}\frac{\phi^{\prime\prime}(\bfw_j^T\bfX^T\bfa_n)}{g(\bfW;\bfa_n)} \Big](\bfX^T\bfa_n)(\bfX^T\bfa_n)^T,
	\end{split}
	\end{equation}
	and 
	\begin{equation}\label{eqn:sths_2}
	\begin{split}
	\Big|\frac{\phi^{\prime\prime}(\bfw_j^T\bfX^T\bfa_n)}{g(\bfW;\bfa_n)}\Big|
	= \frac{\phi(\bfw_k^T\bfX^T\bfa_n)(1-\phi(\bfw_k^T\bfX^T\bfa_n))\cdot \big|1-2\phi(\bfw_k^T\bfX^T\bfa_n)\big|}{\frac{1}{K}\sum_{l=1}^{K}\phi(\bfw_l^T\bfX^T\bfa_n)}
	\le K.
	\end{split}
	\end{equation}
	Similar to \eqref{eqn: sths_1} and \eqref{eqn:sths_2}, we can obtain the following inequality for $y_n=0$.
	\begin{equation}
	\begin{gathered}
	\frac{\phi^{\prime}(\bfw_j^T\bfX^T\bfa_n)\phi^{\prime}(\bfw_k^T\bfX^T\bfa_n)}{\big(1-g(\bfW;\bfa_n)\big)^2}
	\le K^2, \text{\quad and \quad}
	\Big|\frac{\phi^{\prime\prime}(\bfw_j^T\bfX^T\bfa_n)}{1-g(\bfW;\bfa_n)}\Big|
	\le K.
	\end{gathered}
	\end{equation}
	
	Then, for any $\boldsymbol{\alpha}\in \mathbb{R}^d$, we have 
	\begin{equation}
	\begin{split}
	\boldsymbol{\alpha}^T \frac{\partial^2 \hat{f}_{\Omega_t}(\bfW)}{\partial \bfw_j \partial \bfw_k} \boldsymbol{\alpha}
	=& \frac{1}{|\Omega_t|}\sum_{n\in\Omega_t}\bigg[\frac{1}{K^2}\bigg(\frac{y_n}{g^2(\bfW;\bfa_n)} + \frac{1 - y_n}{(1-g(\bfW;\bfa_n))^2} \bigg)\phi^{\prime}(\bfw_j^T\bfX^T\bfa_n)\phi^{\prime}(\bfw_k^T\bfX^T\bfa_n)\\
	&-\frac{\mathds{1}_{\{j=k\}}}{K}\bigg(-\frac{y_n}{g(\bfW;\bfa_n)} + \frac{1 - y_n}{1-g(\bfW;\bfa_n)} \bigg)\phi^{\prime\prime}(\bfw_j^T\bfX^T\bfa_n) \bigg](\boldsymbol{\alpha}^T\bfX^T\bfa_n)^2\\
	:=&\frac{1}{|\Omega_t|}\sum_{n\in\Omega_t} H_{j,k}(\bfa_n)\cdot(\boldsymbol{\alpha}^T\bfX^T\bfa_n)^2.
	\end{split}
	\end{equation}
	
	Next, we show that $H_{j,k}(\bfa_n)\cdot(\boldsymbol{\alpha}^T\bfX^T\bfa_n)^2$ belongs to the sub-exponential distribution. 
	For any $p\in\mathbb{N}^+$, we have
	\begin{equation}\label{eqn: second_h_cl}
	\begin{split}
	\Big(\mathbb{E}_{\bfX,\bfy_n} \Big[ \big|H_{j,k}(\bfa_n)\cdot(\boldsymbol{\alpha}^T\bfX^T\bfa_n)^2\big|^p \Big]\Big)^{1/p}
	\le&\Big(\mathbb{E}_{\bfX} \Big[ \big|4(\boldsymbol{\alpha}^T\bfX^T\bfa_n)^2\big|^p \Big]\Big)^{1/p}\\
	\le& 8\|\bfa_n \|_2^2 p\le  8\sigma_1^2(\bfA) p
	\end{split} 
	\end{equation}
	Hence, $H_{j,k}(\bfa_n)\cdot(\boldsymbol{\alpha}^T\bfX^T\bfa_n)^2$ belongs to the sub-exponential distribution with $\| H_{j,k}(\bfa_n)(\boldsymbol{\alpha}^T\bfX^T\bfa_n)^2 \|_{\psi_1} = 8\sigma_1^2(\bfA)$. Then, the moment generation function of $H_{j,k}(\bfa_n)\cdot(\boldsymbol{\alpha}^T\bfX^T\bfa_n)^2$ can be bounded as 
	\begin{equation}
	\mathbb{E}e^{sH_{j,k}(\bfa_n)\cdot(\boldsymbol{\alpha}^T\bfX^T\bfa_n)^2}\le e^{C\sigma_1^2(\bfA)s^2}
	\end{equation}
	for some positive constant $C$ and any $s\in\mathbb{R}$. From Lemma \ref{Lemma: partial_independent_var} and Chernoff bound, we have 
	\begin{equation}
	\boldsymbol{\alpha}^T \Big(\frac{\partial^2 \hat{f}_{\Omega_t}(\bfW)}{\partial \bfw_j \partial \bfw_k} - \frac{\partial^2{f}_{\Omega_t}(\bfW)}{\partial \bfw_j \partial \bfw_k}\Big)\boldsymbol{\alpha} \le C\sigma_1^2(\bfA) \sqrt{\frac{(1+\delta^2)d\log N}{|\Omega_t|}}
	\end{equation}
	with probability at least $1-N^{-d}$. By selecting $\xi=\frac{1}{2}$ in Lemmas \ref{Lemma: covering_set} and \ref{Lemma: spectral_norm on net}, we have 
	\begin{equation}
	\left\| \frac{\partial^2 \hat{f}_{\Omega_t}(\bfW)}{\partial \bfw_j \partial \bfw_k} - \frac{\partial^2{f}_{\Omega_t}(\bfW)}{\partial \bfw_j \partial \bfw_k} \right\|_2
	\le  C\sigma_1^2(\bfA) \sqrt{\frac{(1+\delta^2)d\log N}{|\Omega_t|}}
	\end{equation}
	with probability at least $1-\big(\frac{5}{N}\big)^{d}$.
	
	In conclusion, we have 
	\begin{equation}
	\begin{split}
	\|\nabla^2f_{\Omega_t}(\bfW) - \nabla^2\hat{f}_{\Omega_t}(\bfW) \|_2
	\le& \sum_{j=1}^{K} \sum_{k=1}^{K}\left\| \frac{\partial^2 \hat{f}_{\Omega_t}(\bfW)}{\partial \bfw_j \partial \bfw_k} 
	- \frac{\partial^2{f}_{\Omega_t}(\bfW)}{\partial \bfw_j \partial \bfw_k} \right\|_2\\
	\le& CK^2\sigma_1^2(\bfA)\sqrt{\frac{(1+\delta^2)d\log d}{|\Omega_t|}}
	\end{split}
	\end{equation}
	with probability at least $1-\big(\frac{5}{d}\big)^{d}$.
\end{proof}

\subsubsection{Proof of Lemma \ref{lemma: first_order_cl}}
\begin{proof}[Proof of Lemma \ref{lemma: first_order_cl}]
	Recall that the first-order derivative of $\hat{f}_{\Omega_t}(\bfW)$ is calculated from
	\begin{equation}
	\frac{\partial \hat{f}_{\Omega_t}(\bfW)}{\partial \bfw_j} = -\frac{1}{K|\Omega_t|}\sum_{n\in\Omega}\frac{y_n-g(\bfW;\bfa_n)}{g(\bfW;\bfa_n)\big(1-g(\bfW;\bfa_n)\big)}\phi^{\prime}(\bfw_j^T\bfX^T\bfa_n)\bfX^T\bfa_n.
	\end{equation}
	Similar to  \eqref{eqn:sths_2}, we have
	\begin{equation}
	\begin{split}
	\Big|\frac{\phi^{\prime}(\bfw_j^T\bfX^T\bfa_n)}{g(\bfW;\bfa_n)}\Big|
	= \frac{\phi(\bfw_k^T\bfX^T\bfa_n)(1-\phi(\bfw_k^T\bfX^T\bfa_n))}{\frac{1}{K}\sum_{l=1}^{K}\phi(\bfw_l^T\bfX^T\bfa_n)}
	\le K.
	\end{split}
	\end{equation}
	Similar to \eqref{eqn: second_h_cl}, for any fixed $\boldsymbol{\alpha}\in \mathbb{R}^{dK}$, we can show that random variable $\boldsymbol{\alpha}^T\frac{\partial \hat{f}_{\Omega_t}(\bfW)}{\partial \bfw_j}$ belongs to sub-exponential distribution with the same bounded norm up to a constant.
	 Hence, by applying Lemma \ref{Lemma: partial_independent_var} and the Chernoff bound, we have
	 \begin{equation}
	 \left\| \nabla f_{\Omega_t}(\bfW) - \nabla\hat{f}_{\Omega_t}(\bfW) \right\|_2
	 \lesssim K^2\sigma_1^2(\bfA) \sqrt{\frac{(1+\delta^2)d\log N}{|\Omega_t|}}
	 \end{equation}
	 with probability at least $1-\big(\frac{5}{N}\big)^{d}$.
\end{proof}
\section{Proof of Lemma \ref{Lemma: sigma_1}}
\begin{proof}[Proof of Lemma \ref{Lemma: sigma_1}]
Let $\widetilde{\bfA}$  denote the adjacency matrix, then we have
    \begin{equation}
        \sigma_1(\widetilde{\bfA}) = \max_{\bfz} \frac{\bfz^T \widetilde{\bfA} \bfz}{\bfz^T \bfz} \ge \frac{\boldsymbol{1}^T\widetilde{\bfA}\boldsymbol{1}}{\boldsymbol{1}^T\boldsymbol{1}}= 1 + \frac{\sum_{n=1}^N \delta_{n} }{N},
    \end{equation}
    where $\delta_n$ denotes the degree of node $v_n$.
    Let $\bfz$ be the eigenvetor of the maximum eigenvalue $\sigma_1(\bfA)$. Since $\sigma_1(\bfA)= \bfD^{-1/2}\widetilde{\bfA}\bfD^{-1/2}$ and $\bfD$ is diagonal matrix, then $\bfz$ is the eigenvector to $\sigma_1(\widetilde{\bfA})$ as well. Then, let $n\in [N]$ be the index of the largest value of vector $\bfz_n$ as $z_n  = \|\bfz\|_{\infty}$, we have
    \begin{equation}
        \sigma_1({\widetilde{\bfA}}) = \frac{(\widetilde{\bfA} \bfz)_n}{z_n}
        = \frac{\widetilde{\bfa}_n^T \bfz}{z_n}
        \le \frac{\|\bfa_n\|_1 \|\bfz\|_{\infty}}{z_n}
        = 1 + \delta.
    \end{equation}
    where $\widetilde{\bfa}_n$ is the $n$-th row of $\widetilde{\bfA}$.
    
    Since $\bfD$ is a diagonal matrix with $\|\bfD \|_2 \le 1+\delta$, then we can conclude the inequality in this lemma.
\end{proof}

\section{Proof of Lemma \ref{Thm: initialization}}\label{Sec: Proof of initialization}
{The proof of Lemma \ref{Thm: initialization} is divided into three major  parts to bound $I_1$, $I_2$ and $I_3$ in \eqref{eqn: temppppp}. Lemmas \ref{Lemma: M_2}, \ref{Lemma: M_3} and \ref{Lemma: M_1} provide the error bounds for  $I_1$, $I_2$ and $I_3$, respectively. The proofs of these preliminary lemmas are similar to those of Theorem 5.6 in \cite{ZSJB17}, the difference is to apply Lemma \ref{Lemma: partial_independent_var} plus Chernoff inequality  instead of standard Hoeffding inequality, and we skip the details of the proofs of Lemmas \ref{Lemma: M_2}, \ref{Lemma: M_3} and \ref{Lemma: M_1} here.
\begin{lemma}\label{Lemma: M_2}
	Suppose $\bfM_{2}$ is defined as in \eqref{eqn: M_2} and $\widehat{\bfM}_{2}$ is the estimation of $\bfM_{2}$ by samples. Then, with probability 
	$1-N^{-10}$, we have 
	\begin{equation}
	\|\widehat{\bfM}_{2}-\bfM_{2}\| \lesssim \sigma_1^2(\bfA)\sqrt{\frac{(1+\delta^2)d\log N}{|\Omega|}}, 
	\end{equation}
	provided that $|\Omega|\gtrsim(1+\delta^2) d\log^4 N$.
\end{lemma}

\begin{lemma}\label{Lemma: M_3}
	Let $\widehat{\bfV}$ be generated by step 4 in Subroutine 1. Suppose $\bfM_{3}(\widehat{\bfV},\widehat{\bfV},\widehat{\bfV})$ is defined as in \eqref{eqn: M_3_v1} and $\widehat{\bfM}_{3}(\widehat{\bfV},\widehat{\bfV},\widehat{\bfV})$ is the estimation of $\bfM_{3}(\widehat{\bfV},\widehat{\bfV},\widehat{\bfV})$ by samples. 
	Further, we assume $\bfV\in\mathbb{R}^{d \times K}$ is an orthogonal basis of $\bfW^*$ and satisfies $\| \bfV\bfV^T - \widehat{\bfV}\widehat{\bfV}^T \|\le 1/4$.
	Then, provided that $N \gtrsim K^5 \log^6d$, with probability at least
	$1-N^{-10}$,  we have 
	\begin{equation}
	\begin{split}
	\|\widehat{\bfM}_{3}(\widehat{\bfV},\widehat{\bfV},\widehat{\bfV})-\bfM_{3}(\widehat{\bfV},\widehat{\bfV},\widehat{\bfV})\| 
	\lesssim \sigma_1^2(\bfA)\sqrt{\frac{(1+\delta^2)K^3\log N}{|\Omega|}}.
	\end{split}
	\end{equation}
\end{lemma}
\begin{lemma}\label{Lemma: M_1}
	Suppose $\bfM_{1}$ is defined as in \eqref{eqn: M_1} and $\widehat{\bfM}_{1}$ is the estimation of $\bfM_{1}$ by samples. Then, with probability 
	$1-N^{-10}$, we have 
	\begin{equation}
	\|\widehat{\bfM}_{1}-\bfM_{1}\| \lesssim  \sigma_{1}^2(\bfA) \sqrt{\frac{(1+\delta^2)d\log N}{|\Omega|}}
	\end{equation}
	provided that $|\Omega| \gtrsim (1+\delta^2) d \log^4 N$.
\end{lemma}

\begin{lemma}[\cite{ZSJB17}, Lemma E.6]\label{Lemma: subspace_error}
	Let $\bfV\in\mathbb{R}^{d \times K}$ be an orthogonal basis of $\bfW^*$ and $\widehat{\bfV}$ be generated by step 4 in Subroutine 1. Assume $\|\widehat{\bfM}_{2}-\bfM_{2} \|_2 \le \sigma_{K}(\bfM_{2})/10$. Then, for some small $\varepsilon_0$, we have 
	\begin{equation}
	\|\bfV\bfV^T -\widehat{\bfV}\widehat{\bfV}^T \|_2 \le \frac{\|\bfM_{2} -\widehat{\bfM}_{2} \|}{\sigma_{K}(\bfM_{2})}.
	\end{equation}
\end{lemma}
\begin{lemma}[\cite{ZSJB17}, Lemma E.13]\label{Lemma: first_order_solution}
	Let $\bfV\in\mathbb{R}^{d \times K}$ be an orthogonal basis of $\bfW^*$ and $\widehat{\bfV}$ be generated by step 4 in Subroutine 1. Assume $\bfM_{1}$ can be written in the form of \eqref{eqn: M_1} with some homogeneous function $\phi_1$, and let $\widehat{\bfM}_{1}$ be the estimation of $\bfM_{1}$ by samples. Let $\widehat{\bf \alpha}$ be the optimal solution of \eqref{eqn: int_op} with $\widehat{\overline{\bfw}}_j =\widehat{\bfV}\widehat{\bfu}_{j}$. Then, for each $j\in[K]$, if
	\begin{equation}\label{eqn: condition_Lemma_first_order}
	\begin{gathered}
	T_1:=\|\bfV\bfV^T -\widehat{\bfV}\widehat{\bfV}^T \|_2\le \frac{1}{\kappa^2\sqrt{K}},\\
	T_2:=\|\widehat{\bfu}_j -\widehat{\bfV}^T\overline{\bfw}_j \|_2 \le \frac{1}{\kappa^2\sqrt{K}},\\
	T_3:=\|\widehat{\bfM}_{1}-\bfM_{1} \|_2\le \frac{1}{4}\|\bfM_{1} \|_2,
	\end{gathered}
	\end{equation} 
	then we have 
	\begin{equation}
	\Big| \|\bfw_j\|_2 - \widehat{\alpha}_j \Big|\le \Big(\kappa^4K^{\frac{3}{2}}\big(T_1+T_2\big)+\kappa^2  K^{\frac{1}{2}}T_3\Big)\|\bfW^* \|_2.
	\end{equation}
\end{lemma}
\begin{proof}[Proof of Lemma \ref{Thm: initialization}] 
	we have
	\begin{equation}\label{eqn: temppppp}
	\begin{split}
	\|\bfw_j^* -\widehat{\alpha}_j\widehat{\bfV}\widehat{\bfu}_j \|_2
	\le& \Big\| \bfw_j^* -\|\bfw_j \|_2\widehat{\bfV}\widehat{\bfu}_j +\|\bfw_j \|_2\widehat{\bfV}\widehat{\bfu}_j - \widehat{\alpha}_j\widehat{\bfV}\widehat{\bfu}_j\Big\|_2\\
	\le&\Big\| \bfw_j^* -\|\bfw_j \|_2\widehat{\bfV}\widehat{\bfu}_j\|_2 \Big\|_2 \Big\|  \|\bfw_j \|_2\widehat{\bfV}\widehat{\bfu}_j - \widehat{\alpha}_j\widehat{\bfV}\widehat{\bfu}_j\Big\|_2\\
	\le& \|\bfw_j^* \|_2 \| \overline{\bfw}_j^* - \widehat{\bfV}\widehat{\bfu}_j \|_2 + 
	\Big| \|\bfw_j\|_2 - \widehat{\alpha}_j \Big\|_2  \|\widehat{\bfV}\widehat{\bfu}_j \|_2\\
	\le &\sigma_{1} \big(\|\overline{\bfw}^*_j - \widehat{\bfV}\widehat{\bfV}^T\overline{\bfw}^*_j \|_2 +  \|\widehat{\bfV}^T\overline{\bfw}^*_j-\widehat{\bfu}_j \|_2 \big)
	+\Big| \|\bfw_j\|_2 - \widehat{\alpha}_j \Big|\\
	:=&\sigma_{1} \big(I_1 +  I_2 \big)
	+I_3.
	\end{split} 
	\end{equation}
	From Lemma \ref{Lemma: subspace_error}, we have 
	\begin{equation}
	\begin{split}
	I_1
	= \|\overline{\bfw}^*_j - \widehat{\bfV}\widehat{\bfV}^T\overline{\bfw}^*_j \|_2
	\le \|\bfV\bfV^T - \widehat{\bfV}\widehat{\bfV}^T \|_2
	\le \frac{\|\widehat{\bfM}_{2} -\bfM_{2}  \|_2}{\sigma_{K}(\bfM_{2})} ,
	\end{split}
	\end{equation}
	where the last inequality comes from Lemma \ref{Lemma: M_2}.
	Then, from \eqref{eqn: M_2}, we know that 
	\begin{equation}
	\sigma_{K}(\bfM_{2})\lesssim  \min_{1\le j \le K} \|\bfw^*_{j} \|_2\lesssim \sigma_{K}(\bfW^*).
	\end{equation}
	From Theorem 3 in \cite{KCL15}, we have
	\begin{equation}
	\begin{split}
	I_2
	=\|\widehat{\bfV}^T\overline{\bfw}^*_j-\widehat{\bfu}_j \|_2
	\lesssim \frac{\kappa}{\sigma_{K}(\bfW^*)}\| \widehat{\bfM}_{3}(\widehat{\bfV}, \widehat{\bfV} , \widehat{\bfV}) - {\bfM}_{3}(\widehat{\bfV}, \widehat{\bfV} , \widehat{\bfV}) \|_2.
	\end{split}
	\end{equation}
	To guarantee the condition \eqref{eqn: condition_Lemma_first_order} in Lemma \ref{Lemma: first_order_solution} hold, according to Lemmas \ref{Lemma: M_2} and \ref{Lemma: M_3}, we need $|\Omega|\gtrsim \kappa^3(1+\delta^2)Kd\log N$.
	Then, from Lemma \ref{Lemma: first_order_solution}, we have
	\begin{equation}
	I_3 = \Big(\kappa^4 K^{3/2}(I_1 + I_2) + \kappa^2 K^{1/2}\|\widehat{\bfM}_{1}-\bfM_{1} \|\Big)\|\bfW^* \|_2.
	\end{equation}
	Since $d\gg K$, according to Lemmas \ref{Lemma: M_2}, \ref{Lemma: M_3} and \ref{Lemma: M_1}, we have 
	\begin{equation}
	\|\bfw_j^* -\widehat{\alpha}_j\widehat{\bfV}\widehat{\bfu}_j \|_2\lesssim  \kappa^6 \sigma_{1}^2(\bfA)\sqrt{\frac{K^3(1+\delta^2)d\log N}{|\Omega|}}\|\bfW^* \|_2
	\end{equation}
	provided $|\Omega|\gtrsim (1+\delta^2)d\log^4 N $.
\end{proof}
\end{document}